%% file: article.tex
\documentclass[
 reprint,
showpacs,
 amsmath,amssymb,amsfonts,amsthm
 aps,
prx,
floatfix,
]{revtex4-2} 

\usepackage{graphicx}
\usepackage{dcolumn}
\usepackage{bm}
\usepackage{hyperref}

\usepackage{amsthm}

\newtheorem{thm}{Theorem}

\usepackage{tabularx}
\usepackage{xspace} 
\usepackage{afterpage} 
\usepackage{color}
\usepackage{etoolbox}

\begin{document}

\input{macros}

\newtoggle{arxiv}
\togglefalse{arxiv}

\title{The optimality of syntactic dependency distances}

\author{Ramon Ferrer-i-Cancho$^1$}
 \email{rferrericancho@cs.upc.edu}
 \homepage{http://www.cs.upc.edu/~rferrericancho/} 

\author{Carlos G\'{o}mez-Rodr\'{i}guez$^2$}
  \email{carlos.gomez@udc.es}
  \homepage{http://www.grupolys.org/~cgomezr/}

\author{Juan Luis Esteban$^3$}
  \email{esteban@cs.upc.edu}
  \homepage{https://www.cs.upc.edu/~esteban/}

\author{Llu\'is Alemany-Puig$^1$}
  \email{lluis.alemany.puig@upc.edu }
  \homepage{https://cqllab.upc.edu/people/lalemany/}

 \affiliation{$^1$Complexity and Quantitative Linguistics Lab \\
  LARCA Research Group \\
  Departament de Ci\`encies de la Computaci\'o \\
  Universitat Polit\`ecnica de Catalunya \\
  Campus Nord, Edifici Omega\\
  Jordi Girona Salgado 1-3 \\
  08034 Barcelona, Catalonia, Spain
 }
 \affiliation{$^2$Universidade da Coru\~na, CITIC\\
  FASTPARSE Lab, LyS Research Group\\
  Departamento de Ciencias de la Computaci\'on y Tecnolog\'ias de la Informaci\'on \\
  Facultade de Inform\'atica, Elvi\~{n}a, 
  15071, A Coru\~na, Spain
 }
 \affiliation{$^3$Departament de Ci\`encies de la Computaci\'o \\
  Universitat Polit\`ecnica de Catalunya (UPC) \\
  Campus Nord, Edifici Omega\\
  Jordi Girona Salgado 1-3 \\
  08034 Barcelona, Catalonia, Spain.
 }


\date{\today}

\begin{abstract}
It is often stated that human languages, as other biological systems, are shaped by cost-cutting pressures but, to what extent? Attempts to quantify the degree of optimality of languages by means of an optimality score have been scarce and focused mostly on English. Here we recast the problem of the optimality of the word order of a sentence as an optimization problem on a spatial network where the vertices are words, arcs indicate syntactic dependencies and the space is defined by the linear order of the words in the sentence. We introduce a new score to quantify the cognitive pressure to reduce the distance between linked words in a sentence. 
The analysis of sentences from \Languages languages representing \Families linguistic families reveals that half of languages are optimized to a $70\%$ or more. The score indicates that distances are not significantly reduced in a few languages and confirms two theoretical predictions, i.e. that longer sentences are more optimized and that distances are more likely to be longer than expected by chance in short sentences. We present a new hierarchical ranking of languages by their degree of optimization. The new score has implications for various fields of language research (dependency linguistics, typology, historical linguistics, clinical linguistics and cognitive science). Finally, the principles behind the design of the score have implications for network science. 
\end{abstract}

\maketitle

\iftoggle{arxiv}{\tableofcontents}{}

\section{Introduction}

\input{introduction}

\section{Optimality scores}

\label{optimality_metric_section}

\input{optimality_metric}

\section{Results}

\label{results_section}

\input{results}

\section{Discussion}

\input{discussion}

\section{Materials and methods}

\input{matmethods}

\begin{acknowledgments}
This article is dedicated to the memory of G. Altmann (1931-2020) \cite{Koehler2021a}. We are grateful to C. Bentz, N. Catal\`a, M. H. Christiansen, M. Gustison and A. Hern\'andez-Fern\'andez for helpful comments. 
RFC and LAP are supported by the grant TIN2017-89244-R from MINECO (Ministerio de Econom\'ia, Industria y Competitividad). RFC is also supported by the recognition 2017SGR-856 (MACDA) from AGAUR (Generalitat de Catalunya). LAP is also supported by Secretaria d'Universitats i Recerca de la Generalitat de Catalunya and the Social European Fund.
CGR is supported by the European Research Council (ERC), under the European Union's Horizon 2020 research and innovation
programme (FASTPARSE, grant agreement No 714150), the ANSWER-ASAP project (TIN2017-85160-C2-1-R) from ERDF/MICINN-AEI, 
Xunta de Galicia (ED431C 2020/11 and an Oportunius program grant to complement ERC grants); the CITIC research center is funded by ERDF and Xunta de Galicia (ERDF - Galicia 2014-2020 program, grant ED431G 2019/01).
JLE is funded by the grants TIN2016-76573-C2-1-P and PID2019-109137GB-C22 from MINECO.
\end{acknowledgments}

R.F.C. designed research; R.F.C and C.G.R. wrote the paper; C.G.R. collected and preprocessed data; C.G.R. and R.F.C. analyzed the data; R.F.C, C.G.R and J.L.E developed the mathematical arguments; C.G.R., J.L.E. and L.A.P. wrote the core code; all authors verified the mathematical arguments.

\appendix

\section{The mathematics of optimality scores}

\label{maths_appendix}

\input{SI_maths}

\section{Materials}

\label{materials_appendix}

\input{SI_materials}

\section{Methods}

\label{methods_appendix}

\input{SI_methods}

\section{Results}

\label{results_appendix}

\input{SI_results}

\bibliography{../../biblio_dt/main,../../biblio_dt/twoplanaracl,../../biblio_dt/Ramon,../../biblio_dt/optimization_in_biology,../../biblio_dt/twoplanaracl_ours,../../biblio_dt/Ramon_ours}

\end{document}

%% file: macros.tex
\newcommand{\nMaxOmegaMin}{24\xspace}
\newcommand{\cUpperBoundOmegaMin}{1.89\xspace} 
\newcommand{\Languages}{93\xspace} 
\newcommand{\Families}{19\xspace}
\newcommand{\IndoeuropeanLanguages}{51\xspace} 
\newcommand{\IncreaseFutrellEtAlPNAS}{56\xspace} 
\newcommand{\IncreaseFutrellEtAlLanguage}{40\xspace} 
\newcommand{\IncreaseGulordavaMerloOptimality}{71\xspace} 

\newcommand{\SentencePUD}{1000\xspace}
\newcommand{\SentencePUDAfterReparallelization}{995\xspace}

\newcommand{\TreebanksUD}{163\xspace}
\newcommand{\LanguagesUD}{92\xspace} 
\newcommand{\IndoeuropeanLanguagesUD}{51\xspace} 
\newcommand{\UralicLanguagesUD}{11\xspace}
\newcommand{\DravidianLanguagesUD}{2\xspace}
\newcommand{\AustronesianLanguagesUD}{2\xspace}

\newcommand{\TreebanksHamleDT}{30\xspace}
\newcommand{\LanguagesHamleDT}{30\xspace}
\newcommand{\IndoeuropeanLanguagesHamleDT}{21\xspace}

\newcommand{\LanguagesPUD}{20\xspace} 

\newcommand{\E}[0]{\mathbb{E}}
\newcommand{\V}[0]{\mathbb{V}} 

\newcommand{\meanD}[0]{\bar{d}}

\newcommand{\optimalityL}[0]{\left< \Omega \right>}
\newcommand{\optimalityLrla}[0]{\left< \Omega \right>_{rla}}
\newcommand{\optimalityn}[0]{\left< \Omega \right>(n)}

\newcommand{\optimalityLD}[0]{\left< D \right>} 
\newcommand{\optimalitynD}[0]{\left< D \right>(n)} 
\newcommand{\optimalityLmeanD}[0]{\left< \meanD \right>}
\newcommand{\optimalitynmeanD}[0]{\left< \meanD \right>(n)}
\newcommand{\optimalityLDelta}[0]{\left< \Delta \right>}
\newcommand{\optimalitynDelta}[0]{\left< \Delta \right>(n)}
\newcommand{\optimalityLGamma}[0]{\left< \Gamma \right>}
\newcommand{\optimalitynGamma}[0]{\left< \Gamma \right>(n)}
\newcommand{\optimalityLDz}[0]{\left< D_z \right>}
\newcommand{\optimalitynDz}[0]{\left< D_z \right>(n)}

%% file: introduction.tex

Optimization has shed light on many features of biological systems \cite{Alexander1996a, Parker1990a, Perez-Escudero2009a}. Some examples are the genetic code \cite{Itzkovitz2007a,Kumar2016a}, neural structures \cite{Klyachko2003, Chen2006a,Perez-Escudero2007a}, metabolic networks \cite{Segre2002a, Berkhout2012a} and animal behavior \cite{Alexander1996a,Gustison2016a}. Understanding why these systems deviate from optimality is a fundamental research problem \cite{Perez-Escudero2009a}.
In the domain of human communication, it has been argued that languages are shaped by cost-cutting pressures \cite{Zipf1949a,Ferrer2015b,Gibson2019a} but evaluations of the degree of optimality of languages have been rather scarce and limited to a few Indo-European languages in most of the cases \cite{Hawkins1998a,Ferrer2004b,Tily2010a,Gulordava2015,Gulordava2016a}.  
Seven fundamental questions for a research program on language efficiency are: 
\begin{enumerate}
\item
Are there universal optimization principles in languages? Or, are there languages showing no optimization at all?
\item
What is the degree of optimization of languages?
\item
What are the most and the least optimized languages of the world?
\item
Under what conditions is optimization stronger?
\item
Why do languages deviate from optimality?
\item
What are the best optimality scores for cross-linguistic research?
\item
What is the contribution of function words to the optimization of languages?
\end{enumerate}

Here we approach these questions from the perspective of the word order of a sentence, one of the main components of the complexity of human languages \cite{joshi85,Fenk-Oczlon2008a}. 
As word order optimization is a multiconstraint satisfaction problem \cite{Ferrer2012d,Gildea2020a,Hahn2020a}, we focus on one of its dimensions, namely the distance between syntactically related words in a sentence. We introduce a new optimality score $\Omega$ and investigate its theoretical and empirical properties. 

\begin{figure}
\centering
\includegraphics[width = \linewidth]{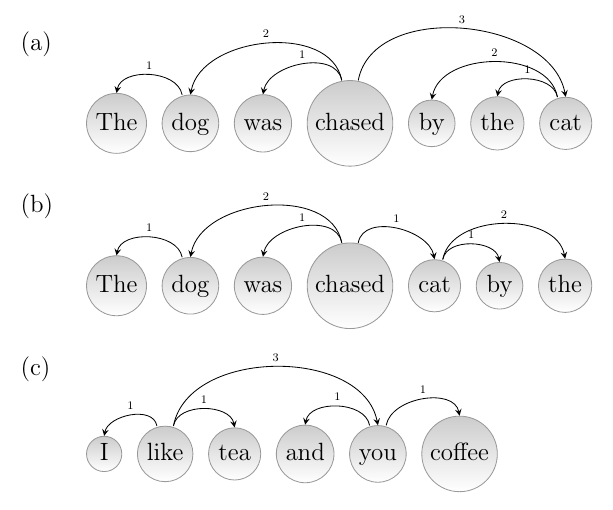}
\caption{\label{syntactic_dependency_tree_figure} Examples of syntactic dependency structures. Edges go from a head word to a dependent word \cite{Tesniere1959a,melcuk88}. Here edge labels indicate dependency distances (in words). Link direction is irrelevant for calculating edge distance. $D$, the true sum of dependency distances, is compared against two baselines assuming that the network is fixed: $D_{min}$, the minimum value of $D$ and, $D_{rla}$, the average value of $D$ over all linear arrangements. $\Omega$ is the new optimality score. (a) A sentence with $n=7$, $D = 10$, $D_{rla}= 16$, $D_{min}=8$, 
and $\Omega= 3/4$. Borrowed from Example 1 in \url{https://universaldependencies.org/introduction.html}. (b) An "ungrammatical" reordering of the sentence in (a) giving $D = D_{min}=8$ and $\Omega= 1$. (c) A sentence where $n=6$, $D = D_{min}=7$, $D_{rla} = 35/3$, 
and $\Omega = 1$. Borrowed from Example 12 in \url{https://universaldependencies.org/u/overview/specific-syntax.html}. }
\end{figure}


The problem of word order in sentences goes beyond human language: the syntactic structure of a sentence is actually a particular case of spatial or geographical network, a network where vertices are embedded into a space \cite{Barthelemy2011a, Reuven2010a_Chapter8}.
That structure can be represented as a network, typically a directed tree, where vertices are words and links indicate syntactic dependencies between words (Fig. \ref{syntactic_dependency_tree_figure}) \cite{Tesniere1959a,melcuk88}. The space has one dimension, defined by the linear order of the elements of the sentence as in Fig. \ref{syntactic_dependency_tree_figure}.

Popular examples of spatial networks in two dimensions are subway networks or road networks \cite{Barthelemy2011a}. Edge distance or edge length, defined as the distance between the two linked vertices, is biased towards low values in syntactic dependency networks \cite{Liu2017a,Temperley2018a} and other spatial networks \cite{Gastner2006a, Ercsey2013a}. Indeed, the probability that two connected words are at a certain distance decays with distance \cite{Ferrer2004b}, as in networks of individuals communicating with cell phones \cite{Lambiotte2008a} or real neural networks \cite{Ercsey2013a}. In syntactic dependency networks, the distance between vertices is typically measured in words (consecutive words are at distance 1, words separated by just one word are at distance 2, and so on) and the distance of an edge is defined as the distance between the two words that are syntactically related \cite{Liu2017a} (Fig. \ref{syntactic_dependency_tree_figure}). Since the discovery that dependency distances are smaller than expected by a shuffling of the words of a sentence \cite{Ferrer2004b}, edge distances have been shown to be significantly small in real syntactic dependency networks from many languages with the help of random baselines of varying complexity \cite{Gildea2007a, Liu2008a, Park2009a, Gildea2010a, Ferrer2013c, Futrell2015a, Futrell2020a}. These findings provide direct support for a principle of dependency distance minimization (DDm) \cite{Ferrer2004b} that would stem from a general principle of distance minimization (Dm) in languages \cite{Ferrer2017c}. Other instances of Dm would be swap distance minimization in word order \cite{Ferrer2016c} or topological distance minimization \cite{Ferrer2003a}. DDm predicts the scarcity of crossings between syntactic dependencies in languages \cite{Gomez2016a,Gomez2019a}.
Such pressure to reduce dependency distances is believed to originate from decay of activation and increased chances of interference as the distance between syntactically connected words increases \cite{Liu2017a,Temperley2018a}. Thus $D$, the sum of the dependency lengths of a sentence (Fig. \ref{syntactic_dependency_tree_figure}), is a measure of cognitive cost.



Although dependency distances are significantly small, they are not minimum \cite{Ferrer2004b, Gildea2007a, Park2009a, Gildea2010a, Tily2010a, Futrell2015a, Gulordava2015, Futrell2020a}. Assuming that the network is constant, $D$ cannot be smaller than $D_{min}$, the minimum value of $D$ achievable by rearranging the words of that sentence while keeping the dependencies between them (Fig. \ref{syntactic_dependency_tree_figure}).
This baseline led to the first optimality score for dependency distances \cite{Ferrer2004b,Tily2010a,Gulordava2016a},  
\begin{equation}
\Gamma = \frac{D}{D_{min}},
\label{Gamma_equation}
\end{equation} 
i.e. a measure of dependency distance in units of $D_{min}$. For fully optimized sentences, $\Gamma = 1$.  
$\Gamma$ was found to exceed $1.2$ on average for sufficiently long sentences in Romanian \cite{Ferrer2004b}. 
The finding that $\Gamma$ tended to be higher in longer sentences was interpreted as less evidence of DDm in longer sentences \cite{Ferrer2004b}. However, we will show how this interpretation is misleading.

The fact that actual dependencies are not minimum has also been confirmed with more restrictive baselines where only projective linearizations of the sentence are allowed \cite{Gildea2007a, Park2009a, Gildea2010a, Tily2010a, Futrell2015a, Gulordava2015}. These results indicate that $D$ is greater than $D_{min}$ for sufficiently large sentences but no quantitative measurement of the degree of optimality has been made.
As in many biological problems \cite{Perez-Escudero2009a}, a crucial question is: if DDm is a fundamental principle of languages, why are real languages not reaching the theoretical minimum sum of dependency lengths? (Question 5). Here we aim to provide a coherent answer to this question. 

In this article, we present $\Omega$, a new measure of the degree of optimality of syntactic dependency lengths, that takes a value of 1 when distances are minimized, negative values when distances are maximized, and is expected to be zero in shuffled sentences. We explain its advantages over $\Gamma$ and other scores (Question 6). Our mathematical analysis reveals that raw $D$ and its variants are poor optimality scores and also poor indicators of pressure for DDm although they are the scores used in the overwhelming majority of studies on DDm \cite{Liu2017a,Temperley2018a}.    
We use $\Omega$ to investigate DDm in a sample of \Languages languages from \Families families, \IncreaseFutrellEtAlLanguage more languages than in the most recent large-scale study of dependency distances \cite{Futrell2020a} and \IncreaseGulordavaMerloOptimality more languages than in the broadest previous study of the degree of optimality of dependency distances \cite{Gulordava2016a}.
We find massive evidence that $\Omega$ is significantly large with respect to shuffled sentences (Question 1). 
We introduce a hierarchical method to rank languages by $\Omega$ that is a significant improvement over previous attempts based on a total ordering of languages by mean dependency distance \cite{Liu2008a}. Thanks to this method, we identify highly and less optimized languages and uncover more than 140 previously unreported relationships of the form {\em language $x$ is more optimized than language $y$} (Question 3).
We also find that languages are optimized to a $70\%$ on average (Question 2) and confirm two predictions of a mathematical theory of word order \cite{Ferrer2014a, Ferrer2019a}. First, we support the prediction that pressure to minimize dependency lengths is larger for longer sentences, in particular, we find that $\Omega$ tends to grow with sentence length (Question 4) 
although a naive interpretation of the massive evidence of a positive correlation between dependency distance and sentence length \cite{Ferrer2004b,Ferrer2013c,Jiang2015a,Park2009a,Futrell2015a} suggests that $\Omega$ should decrease as sentence length increases.
Second, we support the prediction that DDm is more likely to be surpassed by other word order principles in short sequences, i.e. negative and significantly small values of $\Omega$ are found in sentences of $3$ or $4$ words (Question 5). Thanks to the statistical power of $\Omega$, we find more evidence of anti-DDm effects in short sentences than in previous research \cite{Ferrer2019a}. 
Finally, we show that function words contribute to increase the degree of optimality of dependency distances (Question 7).

%% file: optimality_metric.tex
The syntactic structure of a sentence is two-fold: a network, i.e. a tree, and a linear arrangement of the vertices that is defined by the ordering of the words of the sentence (Fig. \ref{syntactic_dependency_tree_figure}). 
We aim to investigate optimality scores that measure the extent to which $D$ has been optimized, assuming that the network remains constant and the linear arrangement can vary. The simplest score would consist of using $D$ itself, but $D$ is uninformative unless a baseline or reference is taken into account. Two baselines have been considered for $D$: $D_{min}$ and $D_{rla}$, namely the minimum and the average of $D$ over the $n!$ different linear arrangements of a given sentence of $n$ words \cite{Ferrer2004b}. $D_{min}$ is the solution of the minimum linear arrangement problem of computer science on a tree \cite{Diaz2002,Chung1984}.
A uniformly random linear arrangement (rla) is a linear arrangement 
chosen randomly in such a way that every possible arrangement has the same probability, namely $1/n!$.
Such random arrangements are our random baseline and our null hypothesis. Fig. \ref{syntactic_dependency_tree_figure} shows a sentence where $D>D_{min}$ and a sentence where $D = D_{min}$, while $D < D_{rla}$ in both cases.
 
$D_{rla}$ is a shorthand for $\E_{rla}[D]$, the expected value of $D$ in a uniformly random linear arrangement of a certain tree.
It is well-known that \cite{Ferrer2004b,Zornig1984a}
\begin{equation}
D_{rla}  = \frac{1}{3}(n^2-1).
\label{sum_of_dependency_lengths_random_equation} 	
\end{equation}
To estimate the optimality of a sentence, at some point one has to choose between $D$ and its variants. One is the average $D$ over the $n-1$ edges of a tree, i.e. $\meanD = D/(n-1)$ \cite{Ferrer2004b}. Another one is Hudson's $D_0$ \cite{Hudson1995a,Hiranuma1999a,Eppler2004a}, that measures the distance between two linked words as the number of intervening words (then consecutive words are at distance 0 and not at distance 1 as in $D$). For this reason, one has $D_0 = D - (n - 1)$. We are interested in optimality scores whose values do not depend on the choice of one of these possibilities. 

\begin{figure}
\centering
\includegraphics[width = \linewidth]{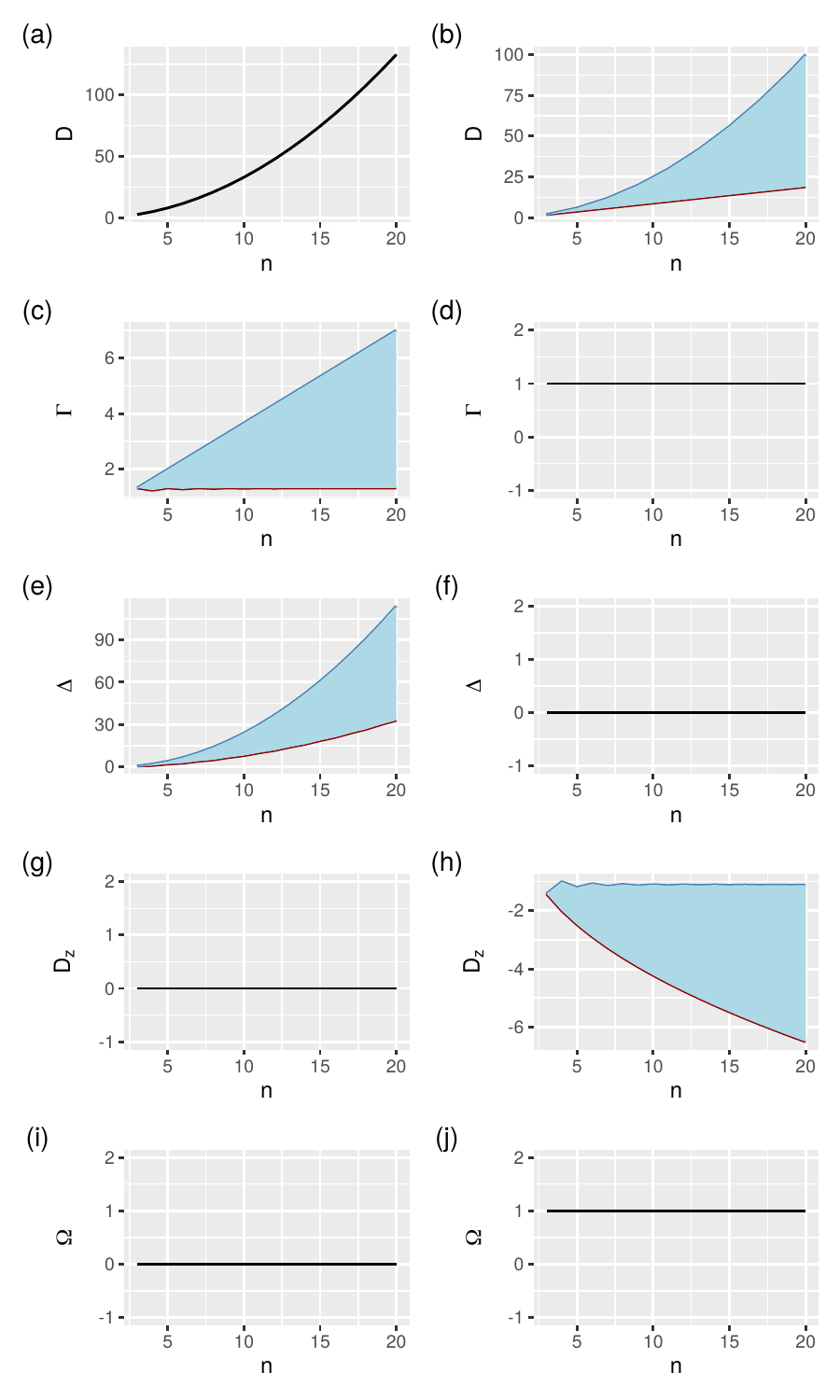}
\caption{\label{scores_figure} The theoretical properties of each optimality score as a function of $n$, the number of words of the sentence (one distinct score each row). Left: the expected value of the score under the null hypothesis. Right: the value of the score in a minimum linear arrangement. The light blue areas in between curves indicate the interval of variation of the score for a given $n$. (a), (b) for $D$. (c), (d) for $\Gamma$. (e), (f) for $\Delta$. (g), (h) for $D_z$. (i), (j) for $\Omega$.}
\end{figure}

To estimate the optimality of a language, one has to aggregate the scores from sentences of a treebank, namely a collection of syntactic dependency trees. The scores can be aggregated by taking their mean. 
We use $\left<X\right>$ to refer to the average value of some score $X$ in some treebank. 
Using $\meanD$ to assess the degree of optimality of a language is problematic because sentence length and the kind of tree usually varies in a real treebank. $D_{rla}$ is a moving baseline because it increases as the number of words in the sentence increases (Eq. \ref{sum_of_dependency_lengths_random_equation} and Fig. \ref{scores_figure} A). $D_{min}$ is also a moving baseline not only because it depends on $n$ but also on the kind of tree (Fig. \ref{scores_figure} B): for a tree of size $n$, the range of variation of $D_{min}$ satisfies \cite{Iordanskii1974a,Esteban2016a} 
\begin{equation}
D_{min}^{linear} = n - 1 \leq D_{min} \leq D_{min}^{star} = \left\lfloor \frac{n^2}{4} \right\rfloor \leq D_{rla},
\label{range_of_variation_of_D_min_equation}
\end{equation}
where $D_{min}^{linear}$ and $D_{min}^{star}$ are the values of $D_{min}$ for a linear and a star tree, respectively. 
For these reasons, interpreting when an ensemble of sentences are fully optimized by means of $\meanD$, even when the sentence lengths are the same, is not straightforward. The values of $D$ should be normalized somehow prior to aggregating them. Similar arguments can be made for $\meanD$ and $D_0$. The problem of $D$ and its simple variants is that they are neither constant under a minimum linear arrangement (Fig. \ref{scores_figure} B) nor stable under the null hypothesis (Fig. \ref{scores_figure} A). 

 
We aim to explore new scores for individual sentences that reduce the problems of existing measures when producing a global optimality score 
by aggregating their values in a treebank. Some existing scores will be reviewed and a new score will be presented with the help of Fig. \ref{scores_figure} (see Appendix \ref{maths_appendix} for further mathematical details and the formulae behind the curves in that figure).
    
The oldest score, $\Gamma$ (Eq. \ref{Gamma_equation}), satisfies $\Gamma \geq 1$ \cite{Ferrer2020a}. A recent optimality score, defined as \cite{Gulordava2015}
\begin{equation}
\Delta = D - D_{min},
\end{equation}
satisfies $\Delta \geq 0$. One has $\Delta = 0$ and $\Gamma = 1$ only in a fully optimized sentence \cite{Ferrer2020a}.
Therefore these scores satisfy constancy under minimum linear arrangement: regardless of $n$ or the kind of tree, $\Delta = 0$ and $\Gamma = 1$ when the linear arrangement is optimal (Fig. \ref{scores_figure} D and Fig. \ref{scores_figure} F).   
However, neither 
$\Gamma$ nor $\Delta$ are stable under random linear arrangement: their expected value depends on $n$ or the kind of tree (Figs. \ref{scores_figure} C and E). Therefore, knowing how far $D$ is from $D_{rla}$ (the random baseline) based on $\optimalityLGamma$ or $\optimalityLDelta$ is not straightforward.

\begin{figure}
\centering
\includegraphics[width = \linewidth]{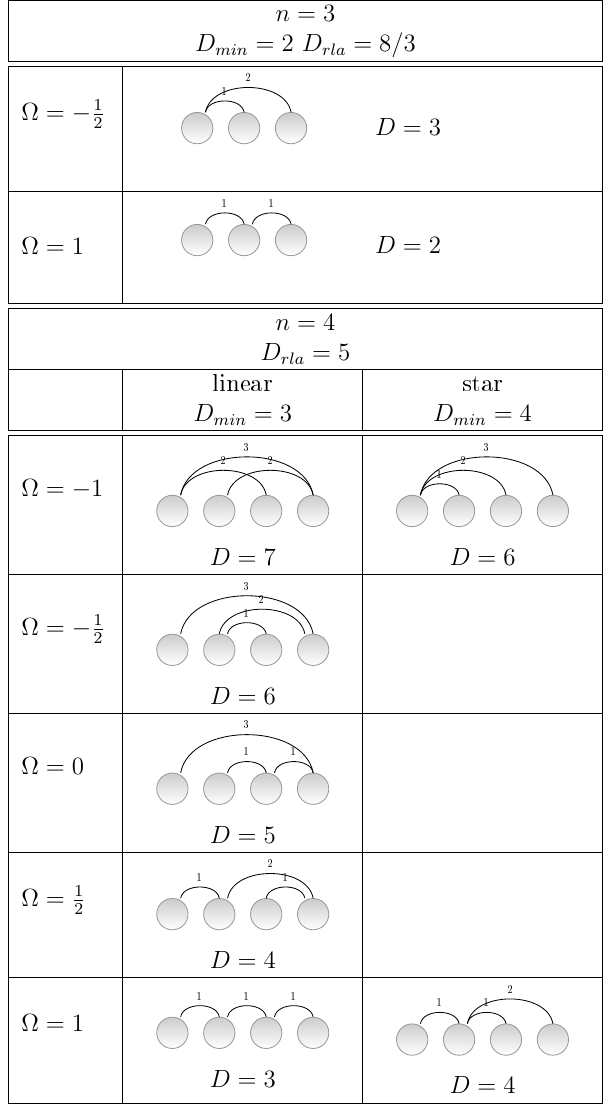}
\caption{\label{examples_of_optimality_measure_figure} All the possible values of $\Omega$ that can be obtained for sentences of lengths $n = 3$ and $n = 4$. Edge lengths (in words) are shown on top of each edge. For $n=3$, the only possible trees are linear trees and star trees simultaneously. For $n=4$, the only possible trees are either linear or star. For every tree and every possible value of $\Omega$, a linear arrangement achieving is shown. }
\end{figure}

An attractive possibility is $D_z$, a $z$-score of $D$, that has been introduced recently and that is defined as \cite{Ferrer2018a}
\begin{equation}
D_z = \frac{D - D_{rla}}{\sigma_{rla}[D]},
\label{z_scored_sum_of_edge_lengths_equation}
\end{equation}
where $\sigma_{rla}[D]$ is the standard deviation of $D$ in uniformly random linear arrangements of a tree. 
$z$-scoring is a common technique to find hidden patterns in heterogeneous data \cite{Cocho2015a,Morales2016a,monroe_colaresi_quinn_2017}. 
$D_z$ is a function of $n$ and $\left<k^2\right>$, the 2nd moment of vertex degrees about zero \cite{Ferrer2018a}, and thus may allow one to control better for the effect of tree size $n$ and the distribution of vertex degrees. Interestingly, $D_z$ is stable under the null hypothesis, namely its expected value in that situation is zero (Fig. \ref{scores_figure} G) but it is not constant under minimum linear arrangement (Fig. \ref{scores_figure} H). 

\subsection*{A new optimality score}

We introduce a new optimality score $\Omega$ that exhibits constancy under minimum linear arrangement, taking a constant value of 1 (Fig. \ref{scores_figure} J) and stability under random linear arrangement, having an expected value of 0 (Fig. \ref{scores_figure} I). We define $\Omega$ as  
\begin{equation}
\Omega = \frac{D_{rla} - D}{D_{rla} - D_{min}}.
\label{raw_optimality_metric_equation}
\end{equation}
In the syntactic dependency trees in Fig. \ref{syntactic_dependency_tree_figure}, $\Omega = 3/4$ for (a) whereas $\Omega = 1$ both for the ``ungrammatical'' reordering in (b) and the other sentence in (c). 
Fig. \ref{examples_of_optimality_measure_figure} shows all the values that $\Omega$ can take for small trees 
along with examples of linear arrangements that achieve each of these values. It is easy to see that the presentation covers all the possible values checking the corresponding figures in \cite{Ferrer2019a}. 

$\Omega$ also satisfies the following three properties. First,
when interpreted as a function of the random variable $D$, $\Omega$ is invariant under linear transformation, namely, $\Omega[aD+b]=\Omega[D]$, where $a$ and $b$ are constants. Thus, $\Omega$ satisfies invariance under proportional scaling, i.e. $\Omega[aD] = \Omega[D]$, and invariance under translation, i.e. $\Omega[D+b] = \Omega[D]$ (Appendix \ref{maths_appendix}). 
$D_z$ is invariant under linear transformation whereas $\Delta$ and $\Gamma$ are not (Appendix \ref{maths_appendix}). $\Delta$ satisfies invariance under translation but not under proportional scaling. Thus $\Delta$ would give different values if $D$ were replaced by $\meanD$ to control for the natural growth of $D$ as $n$ increases. 
Finally, $\Gamma$ satisfies invariance under proportional scaling but not under translation.

Second, while $D$, $D_0$, $\meanD$, $\Gamma$ and $\Delta$ take positive values, $D_z$ and $\Omega$ are signed. $\Omega$ takes positive values when $D < D_{rla}$ and negative values when $D > D_{rla}$, namely when $D$ is above the random baseline, as illustrated in Fig. \ref{examples_of_optimality_measure_figure} for small trees. 
This sign is of great importance for a theory of word order: it is crucial to detect when DDm is surpassed by other word order principles \cite{Ferrer2014a,Ferrer2013f,Ferrer2019a}. Such a detection by means of $\Delta$ or $\Gamma$ is more difficult because $D_{rla}$ depends on $n$ as we have seen. 

Third, $\Omega$ is bounded under maximum linear arrangement. As $n$ increases, the maximum values of $D$, $\Delta$ and $\Gamma$ and $D_z$ diverge (Appendix \ref{maths_appendix}). In contrast, 
$-c \leq \Omega$ where $c$ is a constant that satisfies $\cUpperBoundOmegaMin \leq c \leq 5$. We conjecture that $c = 2$ (Appendix \ref{maths_appendix}).  

$\Omega$ is analogous to the Pearson correlation coefficient. First, $\Omega$ is invariant under linear transformation as the Pearson correlation coefficients and similar correlation metrics \cite{DeGroot2002a}.
Second, Pearson $r$ and $\Omega$ are zero under the null hypothesis.
Third, $r$ and $\Omega$ (Eq. \ref{raw_optimality_metric_equation}) are normalized differences between a sample average value and the expectation under some null model. In $\Omega[\meanD]$, the sample average is $\meanD$ and the expectation under the null hypothesis is $D_{rla}/(n-1)$.
However, $-1 \leq r \leq 1$ while $-c \leq \Omega \leq 1$. 
Newman's modularity $Q$ \cite{Newman2006a} shares with $\Omega$ the second and the third property as well as the property of not being symmetrically distributed around zero ($-1/2 \leq Q \leq 1$) \cite{Brandes2008a}.

Invariance under linear transformation is a fundamental property for research on dependency distance: it implies that $\Omega$ yields the same value if $D$ is replaced by
\begin{itemize}
\item
$D_0$, where $a=0$ and $b = -(n-1)$ (translation).
\item
$\meanD$, where $a = 1/(n-1)$ and $b=0$ \cite{Ferrer2004b} (proportional scaling).
\item
$\Delta$, where $a=0$ and $b = -(n-1)$ (translation).
\item
$D_z$, where $a=1/\sigma_{rla}[D]$ and $b = - D_{rla}/\sigma_{rla}[D]$.  
\end{itemize}   
Crucially, such invariance under linear transformation implies that the estimates of the optimality of a language based on $\Omega$ remain unaltered if $D$ is replaced by some linear transformation of $D$.
For instance, if $D$ is replaced by $D_0$ when controlling for the different conventions when defining dependency distance, the degree of optimality of Chinese based on $\optimalityL$ remains unchanged while that based on $\optimalityLGamma$ is likely to change. Similarly, if $D$ is replaced by $\meanD$ to control for the natural growth of $D$ as sentence length increases, the degree of optimality of Chinese based on $\optimalityL$ remains unchanged again while that based on $\optimalityLDelta$ is likely to change.
Furthermore, linear transformations, when applied to scores that do not satisfy invariance under linear transformation, may alter the power of the statistical tests used to answer research questions (e.g., Questions 1 and 4).  

\subsection*{Summary}

Table \ref{scores_table} reviews the properties of all the optimality scores. $\Omega$ is the only score that incorporates the minimum and the random baseline ($D_{rla}$ and $D_{min}$). For these reasons, $\Omega$ is a dual optimality score, $\Delta$, $\Gamma$ and $D_z$ are singular optimality scores ($\Delta$, $\Gamma$ only incorporate the minimum baseline, $D_{min}$; $D_z$ only incorporates the random  baseline, $D_{rla}$) while $D$ and its linear transformations (e.g., $D_0$) are not proper optimality scores (they do not incorporate any baseline).   

We are aware of another score \cite{Lei2018a}
\begin{equation*}
NDD = \left|\log\frac{\meanD}{\sqrt{\pi^r n}}\right|,
\end{equation*}
where $\pi^r$ is the position of the root of the syntactic dependency structure in the linear arrangement ($1 \leq \pi^r \leq n$). $NDD$ was designed with two main aims in mind: reducing the dependence of $D$ on sentence length $n$, an aim shared with $\meanD$, as well as serving as a measure of syntactic complexity \cite{Lei2018a}. In contrast, the primary goal of the scores reviewed in Table \ref{scores_table} is to measure the degree of optimization of languages based on DDm. Furthermore, the complexity of the formula in the definition of $NDD$ turns the mathematical analysis particularly complex. Crucially, it suffers the same limitations of $D$ and its linear transformations (Table \ref{scores_table}; Appendix \ref{maths_appendix}).

\begin{table*}
\centering
\caption{\label{scores_table}Comparison of the properties of each dependency distance score. }
\begin{ruledtabular}
\begin{tabular}{lrrrrr}
Properties &                                  $D_0$, $D$, $\meanD$, $NDD$ & $\Gamma$       & $\Delta$          & $D_z$     & $\Omega$ \\
\hline
Sign & $+$ & $+$ & $+$ & $\pm$ & $\pm$ \\
Minimum baseline  &                            & $\checkmark$  & $\checkmark$  & $\checkmark$  & $\checkmark$ \\
Random baseline   &                            &               &               &               & $\checkmark$ \\
Constancy under minimum linear arrangement    &                            & $\checkmark$   & $\checkmark$ &              & $\checkmark$ \\
Stability under random linear arrangement     &                           &                &              & $\checkmark$ & $\checkmark$ \\
Invariance under translation                  &                            &                & $\checkmark$ & $\checkmark$ & $\checkmark$ \\
Invariance under proportional change of scale &                            & $\checkmark$   &              & $\checkmark$ & $\checkmark$ \\
Invariance under linear transformation        &                            &                &              & $\checkmark$ & $\checkmark$ \\
Bounded under maximum linear arrangement      &                            &                &              &               & $\checkmark$ \\
\end{tabular}
\end{ruledtabular}
\end{table*}


%% file: results.tex
\begin{table*}
\centering
\caption{\label{datasets_table} The annotation criteria (rows) and the collection (column) of each dataset. See Appendix \ref{materials_appendix} for further details on the collections and the annotation criteria. }
\begin{ruledtabular}
\begin{tabular}{llccc}
                                                   & Universal Dependencies & HamleDT & Parallel Universal Dependencies \\
\hline
Universal Dependencies                             & UD  &          & PUD \\
Surface-syntactic Universal Dependencies           & SUD &          & PSUD \\
Universal Stanford Dependencies                    &     & Stanford \\
Prague Dependencies                                &     & Prague \\
\end{tabular}
\end{ruledtabular}
\end{table*}

We investigate $\Omega$ in \Languages languages from \Families families. The input for the analyses is organized into datasets. A dataset results from the combination of a collection of sentences from different languages and some annotation criteria to define the syntactic dependency structure of sentences (Table \ref{datasets_table}). 
The UD and the SUD datasets contain \LanguagesUD distinct languages. The Prague and the Stanford datasets contain \LanguagesHamleDT distinct languages.
We control for the bias towards Indo-European languages (\IndoeuropeanLanguages languages) and for vertical transmission by grouping languages according to their linguistic family. Languages from the same family are not statistically independent since the features we observe may originate from a common ancestor \cite{Roberts2013b}. 

For a given language, we analyze $\optimalityL$, the average value of $\Omega$ over all sentences from a certain dataset. The sign of $\optimalityL$ is indicative of the kind of constraint on dependency distance. If $\optimalityL > 0$ and $\optimalityL$ is significantly large, then this is evidence of the action of DDm.
If $\optimalityL < 0$ and $\optimalityL$ is significantly small, then this is evidence of anti-DDm effects \cite{Ferrer2019a}. Significance is determined with respect to (randomly) shuffled sentences.

In linguistics, the term {\em doculect} refers to the variety of a language that is documented in a given resource \cite{Cysouw2013a}. 
A novel and a newspaper article written in German are samples of the same language but are not the same doculect. In our datasets, each language is represented by a series of sentences that define a doculect. In the analyses to follow, it is important to bear in mind that we are analyzing doculects rather than languages and that this challenges the comparability of the results obtained for different languages. The majority of doculects correspond to written language and thus all the results to follow are eminently relative to the written modality of these languages. The analyses below reflect a trade-off between reaching as many languages as possible and minimizing the distances between the corresponding doculects. When the latter is critical, we will resort to parallel corpora.

\begin{figure}
\centering
\includegraphics[width = \linewidth]{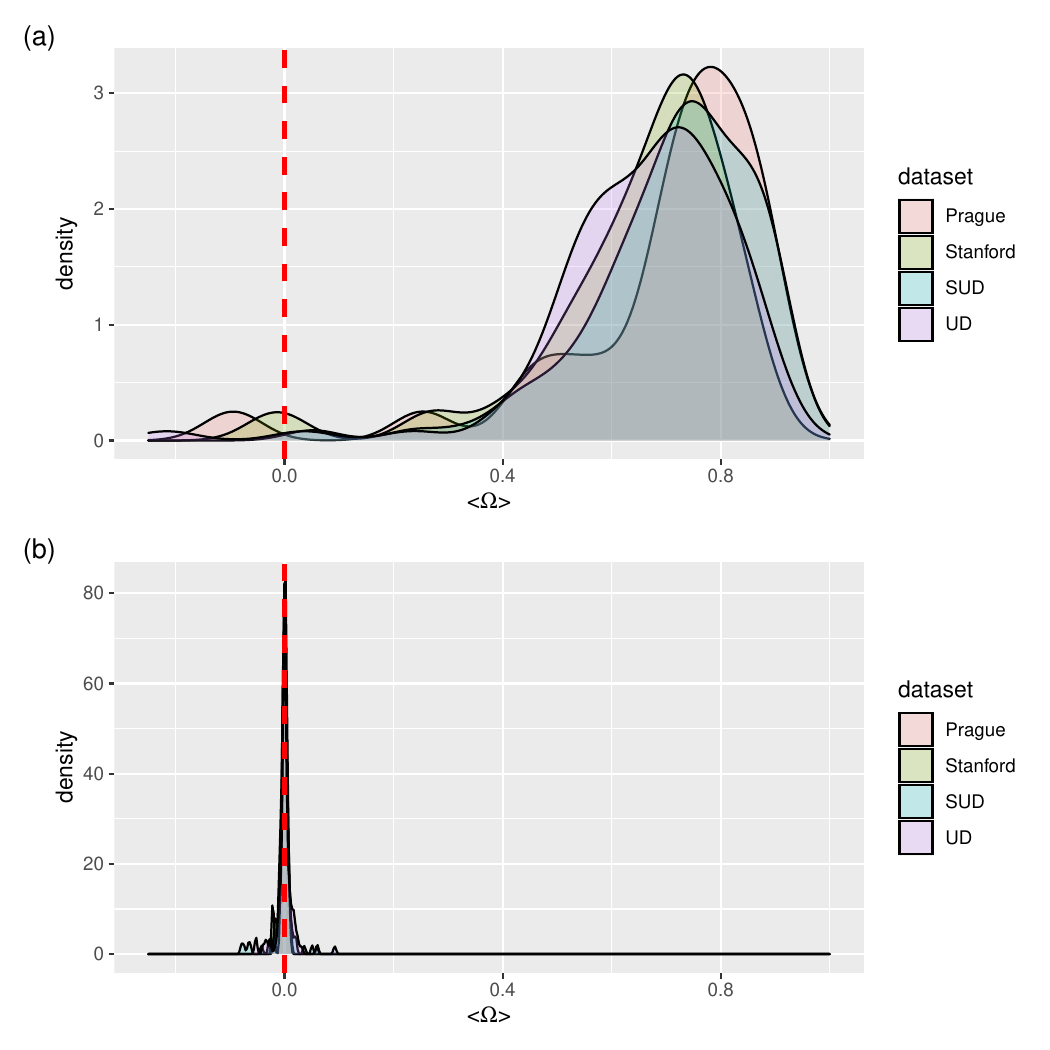}
\caption{\label{probability_density_figure} The probability density of $\optimalityL$, the average $\Omega$ in the sentences of a language, for languages in different datasets. The vertical dashed line indicates the expected value under the null hypothesis. A. Original sentences. B. Shuffled sentences. }
\end{figure}

\subsection*{The optimality of syntactic dependency distances}

Under the null hypothesis of a uniformly random linear arrangement, the probability distribution of $\optimalityL$ should be centered at zero (as expected by the stability under linear arrangement property of $\Omega$). In contrast, the probability density of $\optimalityL$ across languages is skewed towards positive values peaking at about 0.7 (Fig. \ref{probability_density_figure}). After shuffling the words of every sentence, the probability density peaks at zero as expected.

\begin{table}
\centering
\caption{\label{optimality_by_family_table} Summary of the statistical properties of $\optimalityL$ by dataset and language family. {\em All} indicates all the languages of the dataset, regardless of their family. }
\begin{ruledtabular}
\begin{tabular}{llrrrrr}
Dataset & Family & Langs. & Min & Mean & Median & Max \\
\hline
\input{figures/optimality_by_family_table}
\end{tabular}
\end{ruledtabular}
\end{table}

We define the percentage of optimization as $100 \cdot \Omega$. Notice that such percentage can be negative.
In each dataset, languages are optimized at least to a $65\%$ on average when all languages are considered (Table \ref{optimality_by_family_table}). In addition, the median indicates that half of languages are optimized to a $71\%$ at least. The statistics are strongly influenced by the abundance of Indo-European languages in the datasets: \IndoeuropeanLanguagesUD out of \LanguagesUD languages in UD/SUD and \IndoeuropeanLanguagesHamleDT out of \LanguagesHamleDT in Prague/Stanford (Table \ref{optimality_by_family_table}). High values of $\optimalityL$ are found in other families (the average $\optimalityL$ is greater than $50\%$ in the majority of families). All the values of $\optimalityL$ are positive except in a few cases: 
a Pama-Nyungan language (Warlpiri) for UD (but not SUD) and a Dravidian language (Telugu) for Stanford and Prague. 

\subsection*{The significance of syntactic dependency distances}

$\optimalityL$ is significantly large in the overwhelming majority of languages: there are two exceptions in UD (Telugu and Warlpiri), no exception in SUD and one exception (Telugu) in Prague and Stanford (Table \ref{patterns_in_optimality_table}). Interestingly, we did not find any language with a significantly small $\optimalityL$ although $\optimalityL<0$ in a few languages as reported above.  

\begin{table}
\centering
\caption{\label{patterns_in_optimality_table} The significance of $\optimalityL$. $l_0$ is the number of languages available from the dataset, 
$f_H$, is the number of languages where $\optimalityL$ is significantly {\em large} after applying Holm's correction to the languages counted in $l_0$, and Exceptions is the number of the languages counted in $l_0$ where $\optimalityL$ is not significantly {\em large}. The number attached to the language name indicates the corrected $p$-value.}
\begin{ruledtabular}
\begin{tabular}{lrrrll}
Dataset & $l_0$ & $f_H$ & Exceptions & Family & Languages \\
\hline
\input{figures/patterns_in_optimality_table}
\end{tabular}
\end{ruledtabular}
\end{table}

\subsection*{Syntactic dependency distance versus sentence length}

Here we investigate $\optimalityn$, the average value of $\Omega$ in sentences of length $n$ of a  given language. In sufficiently long real sentences, $D < D_{rla}$ on average \cite{Ferrer2004b,Ferrer2013c,Jiang2015a}, and thus $0 < \optimalityn$ is expected. Accordingly, $\optimalityn$ is higher in longer sentences (Fig. \ref{heat_map_figure}). 
For sufficiently small $n$ (e.g., $n \in \{3,4\}$), negative values of $\optimalityn$ are found in some languages. 
Indeed, the correlation between $\optimalityn$ and $n$ is significantly large in the overwhelming majority of languages (Table \ref{correlation_with_sentence_length_table}).
Such a global tendency suggests that DDm may turn out to not be the dominating principle in short sequences. 

\begin{figure*}
\centering
\includegraphics[width=\linewidth]{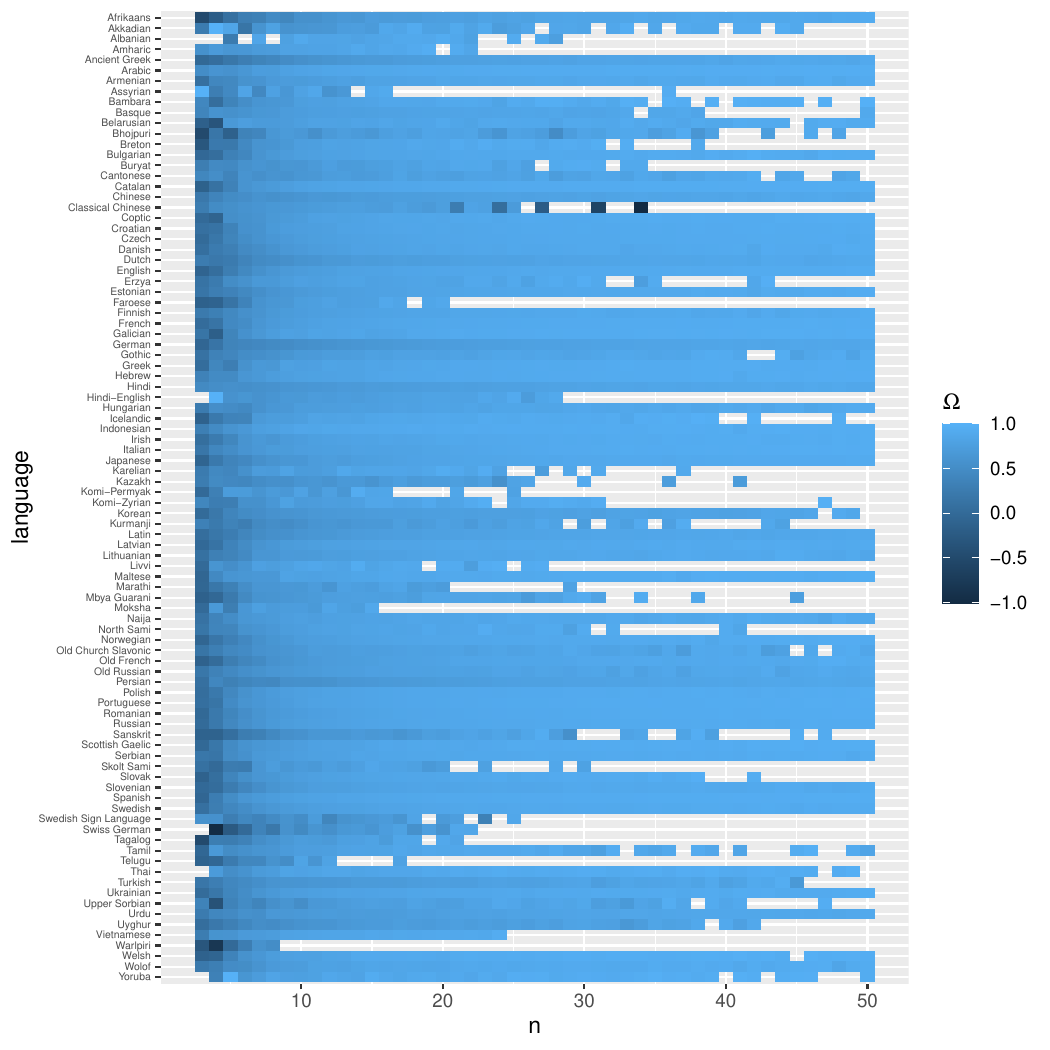}
\caption{\label{heat_map_figure} $\optimalityn$, the average value of $\Omega$ in sentences of length $n$ of a language in UD for $3 \leq n \leq 50$. Sentences longer than 50 words are excluded to ease visualization. 
}
\end{figure*}
 
\begin{table}
\centering
\caption{\label{correlation_with_sentence_length_table} The significance of the Kendall $\tau$ correlation between $\optimalityn$ and sentence length $n$. $l_0$ is the number of languages available from the dataset, 
$f_H$ is the number of languages where the correlation is significantly {\em large} after applying Holm's correction to the languages counted in $l_0$, and Exc. (Exceptions) is the number of the languages counted in $l_0$ where the correlation is not significantly {\em large}. The number attached to the language name indicates the corrected $p$-value. }
\begin{ruledtabular}
\begin{tabular}{lrrrrll}
Dataset & $l_0$ & $f_H$ & Exc. & Family & Languages \\
\hline
\input{figures/correlation_with_sentence_length_table_Omega_greater}
\end{tabular}
\end{ruledtabular}
\end{table}

\subsection*{Anti dependency distance minimization effects in short sequences}

In some languages, $\optimalityn$ is significantly small in sentences of length $n = 3$ or $n = 4$ (Table \ref{anti_dependency_distance_table}). In UD, 12 languages for $n = 3$ and 3 with $n = 4$. Interestingly, the effect disappears completely for $n > 4$ consistently with previous research on DDm \cite{Lei2020a}.
It is unlikely that the results are due to a common descent \cite{Roberts2013b}: $\optimalityn$ is significantly small in languages from distinct families and the finding does not cover the whole family provided that the family is represented by more than one language {\em a priori} (see Appendix \ref{materials_appendix} for the composition of each linguistic family). For instance, in UD with $n = 3$, the 12 languages concerned are 9 Indo-European languages out of \IndoeuropeanLanguagesUD, 
one Austronesian language out of \AustronesianLanguagesUD (Tagalog), one Dravidian language out of \DravidianLanguagesUD (Telugu) and one isolate (Japanese).

\begin{table*}
\centering
\caption{\label{anti_dependency_distance_table} Anti dependency distance minimization in short sentences (significantly low $\optimalityn$ for small $n$). 
$n$ is the sentence length (in words). $l_0$ is the number of languages available from the dataset, 
and $f_H$ is the number of languages where $\optimalityn$ is significantly {\em small} after applying Holm's correction to the languages counted in $l_0$. 
The number attached to the language name indicates the magnitude of the corrected $p$-value. It is obtained after rounding $-\log_{10}(p\mathrm{-value})$ to leave just one decimal digit. Then the significance level $\alpha = 0.05$ gives 1.3.
}
\begin{ruledtabular}
\begin{tabular}{llrrrp{3in}}
$n$ & Dataset & $l_0$ & $f_H$ & Family & Languages \\
\hline
\input{figures/anti_dependency_distance_table_Omega}
\end{tabular}
\end{ruledtabular}
\end{table*}

\subsection*{The ranking of languages based on their degree of optimality}

\begin{figure}
\centering
\includegraphics[width=\linewidth]{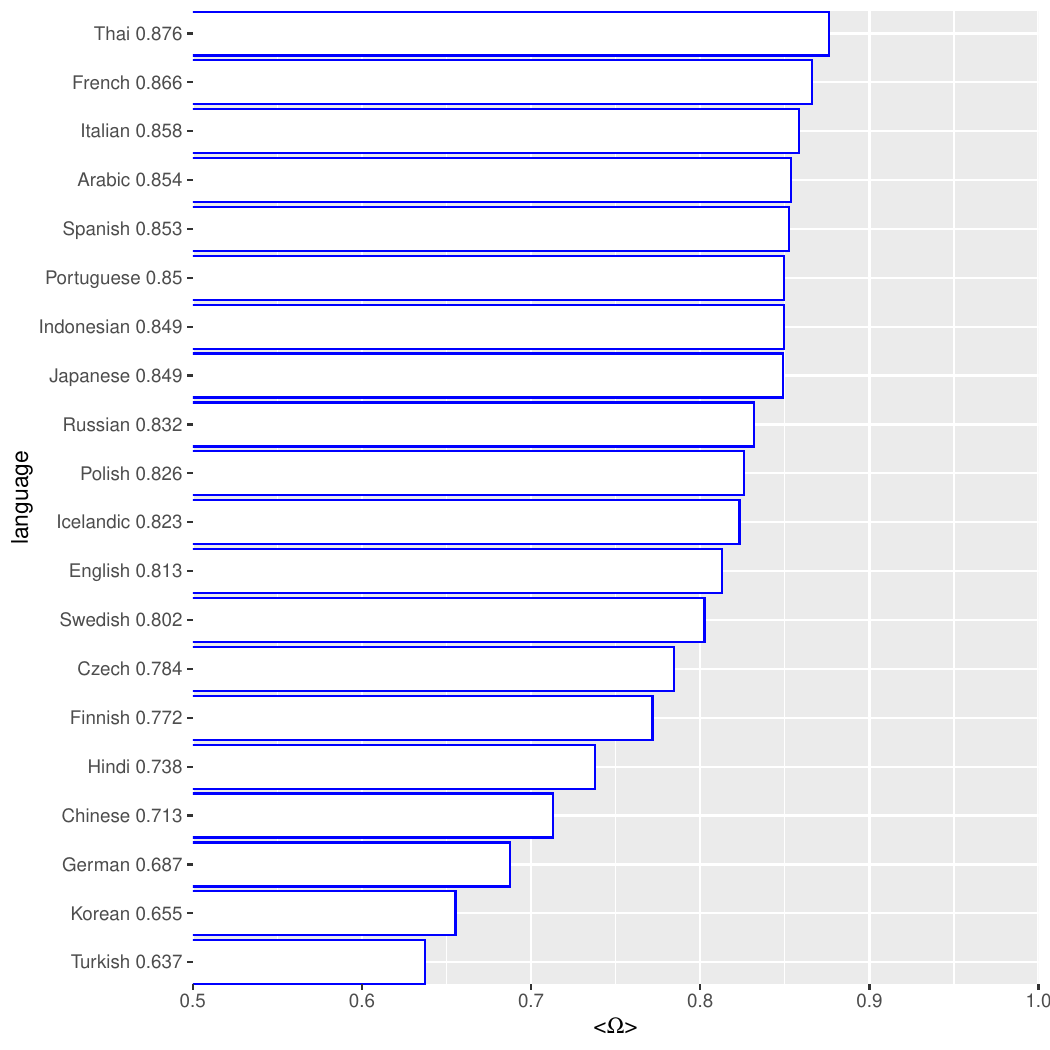}
\caption{\label{ranking_PUD_figure} The ordering of languages in PUD according to $\Omega$. The number attached to each language indicates the value of $\optimalityL$. }
\end{figure}

What are the most and the least optimized languages?
It is possible to sort the languages in a dataset by $\optimalityL$ to obtain a ranking of languages according to dependency distance, e.g. from $\optimalityL = 0.9$ down to $\optimalityL = - 0.22$ 
in UD. Such an ordering is problematic for various reasons. First, it may depend on the dataset used \cite{Passarotti2016a,Osborne2019a}. 
Second, there is a problem of homogeneity or comparability, namely the source of the sentences for each language is not homogeneous (in one language the source may be newspaper articles while in another language it could be literary works).    
Such a lack of homogeneity of the doculects representing each language may distort the ordering. $\optimalityL$ is an aggregate measure over sentences of varying length and we have seen that $\optimalityn$ depends on sentence length (Fig. \ref{heat_map_figure} and Table \ref{correlation_with_sentence_length_table}) in such a way that $\Omega$ can be negative in short sentences. It is well-known that the distribution of sentence length as well as other quantitative features depend on modality or genre \cite{Rudnicka2018a}. Accordingly, modality and genre influence dependency distances \cite{Wang2017a}.
A widely adopted solution in quantitative linguistics for these two problems is the use of parallel corpora, namely, collections of sentences that have the same content, often obtained by translation from one language to the remainder of languages in the sample. By doing so, one reduces the distance between the {\em doculects} of each language as much as possible.
Here we use the Parallel Universal Dependencies corpus (PUD), that contains some sentences of a subset of languages from the UD collection (Appendix \ref{materials_appendix}). PUD has the further advantage of warranting an unprecedented degree of homogeneity in terms of wordness criteria across languages (Appendix \ref{materials_appendix}).
Figure \ref{ranking_PUD_figure} shows a first approximation to the ranking of languages in PUD. Assuming that the ordering is total, the most optimized language would be Thai with $\optimalityL = 0.88$ and the least optimized language would be Turkish with $\optimalityL = 0.64$ (Figure \ref{ranking_PUD_figure}). Interestingly, the Romance languages, i.e. French, Italian, Spanish and Portuguese, cluster in positions 2 to 7 (Fig. \ref{ranking_PUD_figure}). 
However, the ordering in Figure \ref{ranking_PUD_figure} is not total. For instance, although the value of $\optimalityL$ of Italian is greater than that of Arabic, Spanish, Portuguese, Indonesian and Japanese the difference is not significant (Appendix \ref{results_appendix}). 
For that reason, we represent the partial order as a network where there is an arc from language $x$ to language $y$ whenever $\optimalityL_x$ is significantly larger than $\optimalityL_y$ after controlling for multiple comparisons.
Figure \ref{Hasse_diagrams_figure} A shows the corresponding Hasse diagram, a compressed representation of the network that is obtained by applying transitive reduction to the original network \cite{Clough2014a} (Appendix \ref{methods_appendix}). The original network has 149 arcs while the Hasse diagram has just 46 arcs. If a language $y$ can be reached from a language $x$ following one or more arcs in the diagram then $x$ is significantly more optimized than $y$.
Based on that diagram, it cannot be concluded that Thai is the most optimized language, but rather that Thai is more optimized than all the languages in the sample excluding Italian and French, totalling 17 languages. Similarly, it cannot be concluded that Turkish is the least optimized language but rather that all languages in the sample are more optimized than Turkish with the only exception of Korean. 
PUD has allowed us to control for content (turning the source homogeneous). Now the original network has 155 arcs while the Hasse diagram has just 57 arcs. Our next step is to control for annotation style by means of PSUD, a variant of PUD where universal dependencies have been replaced with surface-syntactic universal dependencies \cite{sud}. Figure \ref{Hasse_diagrams_figure} B confirms that Korean and Turkish are among the least optimized languages (in Figure \ref{Hasse_diagrams_figure} A, German is not more optimized than Korean) but shows that Thai is more optimized only with respect to 11 languages. Globally, the ranking of languages that PUD yields is significantly correlated with that of PSUD (Fig. \ref{correlogram_figure}). 

\begin{figure}
\centering
\includegraphics[width=\linewidth]{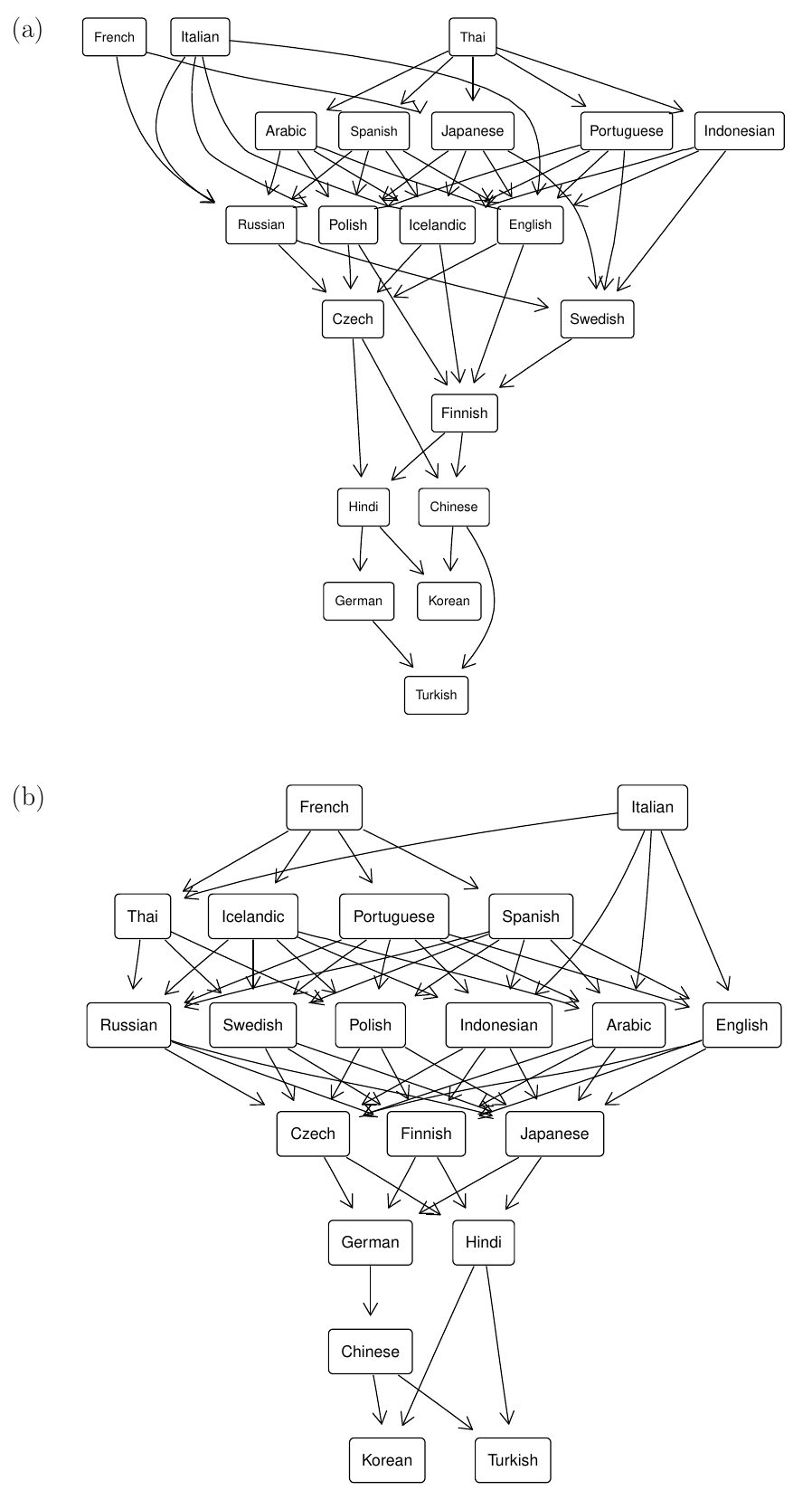}
\caption{\label{Hasse_diagrams_figure} A Hasse diagram where an arrow pointing from one language to another indicates that the former is significantly more optimized than the latter according to $\Omega$. (a) PUD. (b) PSUD. }
\end{figure}

Our control for annotation style implies a control for the way function words are annotated. UD follows the content-head criterion, namely it prefers content words as heads while SUD does not: adpositions, subordinating conjunctions, auxiliaries, and copulas are heads in SUD \cite{sud}. Now we go a step further on function words and control for the proportion of function words, which varies across languages. At one end of the continuum of variation, one finds Asian isolating languages such as Chinese or Thai, that leave many grammatical relationships unexpressed as in other languages of Mainland Southeast Asia \cite{Bisang2006a}. 
Could it be that the ranking of languages that we have obtained is severely distorted by the scarcity of function words and the corresponding high degree of indeterminateness in certain languages? 
The removal of function words implies a reduction of $\optimalityL$ in $95\%$ of the languages both for PUD and PSUD (the only exceptions are Czech for PUD and Korean for PSUD; see Appendix \ref{results_appendix} for further details). However, the relative ordering of the languages is preserved significantly. 
Globally, the ranking of languages that results from removing all function words in PUD is strongly correlated with the original ranking by PUD or PSUD (Spearman correlation $> 0.85$) whereas the ranking that results from removing all function words in PSUD is more weakly correlated with these original rankings (Spearman correlation $< 0.48$) (Fig. \ref{correlogram_figure}). 
The correlation analysis confirms the robustness of the ranking method to a high degree if one takes into account that the removal of all function words in PSUD is a special case because of the statistical problems arising in that condition (a detailed explanation is given in Appendices \ref{materials_appendix} and \ref{results_appendix}).
Notice that the ranking by PSUD without function words is less strongly correlated with the original ranking by PSUD than with the other configurations based on PUD with and without function words (Fig. \ref{correlogram_figure}), suggesting that removing function words in PSUD turns the optimality of languages more PUD-like and less PSUD-like. 
Appendix \ref{results_appendix} shows complementary figures on the effect of removing function words that have been omitted here for the sake of brevity. 

\begin{figure}
\centering
\includegraphics[width=\linewidth]{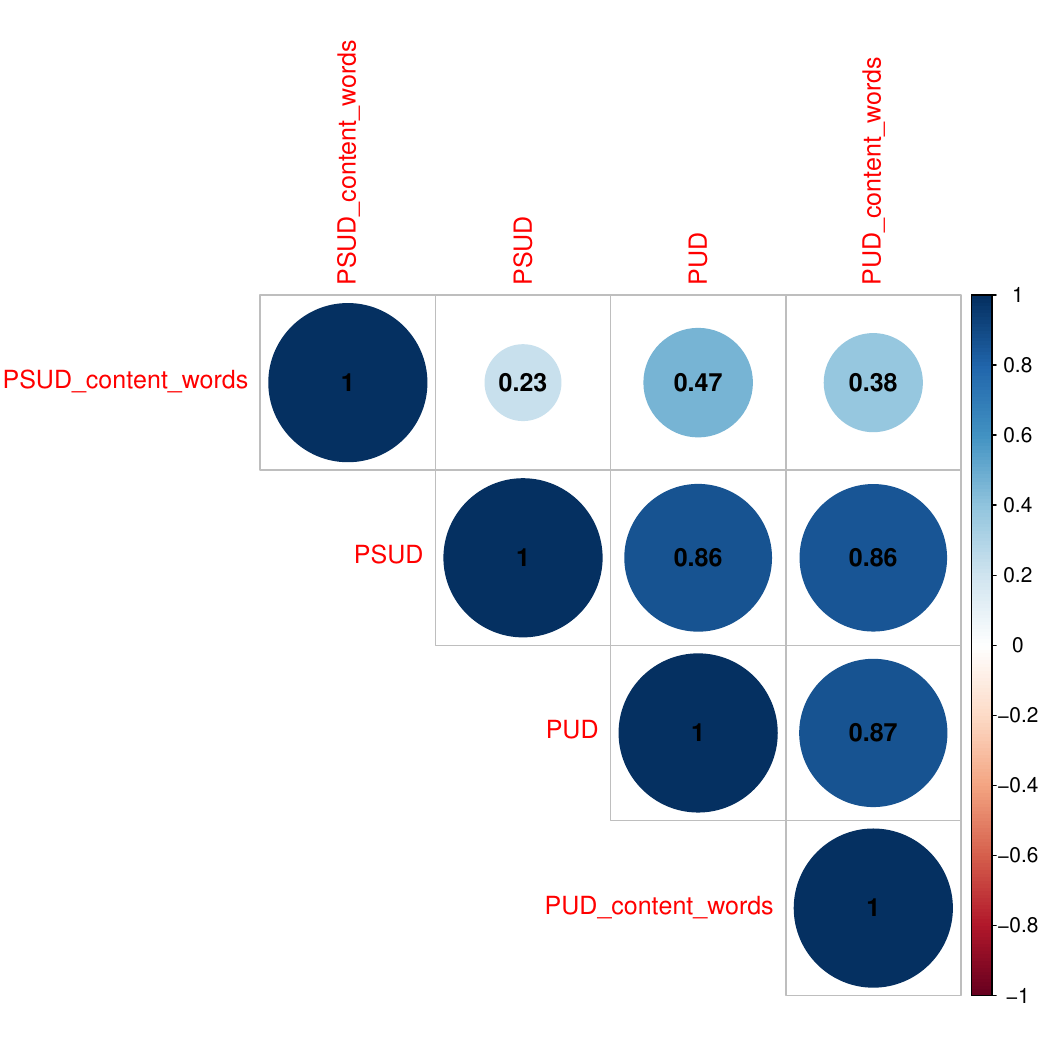}
\caption{\label{correlogram_figure} Correlogram of the Spearman rank correlation between the values of $\optimalityL$ of a language under two different conditions. The conditions that are considered are PUD annotation, PSUD annotation, PUD annotation with content words only and PSUD annotation with content words only. The numbers inside the circles indicate the value of the correlation. All correlations are significant. Significance is assessed after applying a Holm correction for multiple comparisons (Appendix \ref{methods_appendix}). }
\end{figure}

\begin{table}
\centering
\caption{\label{surrogate_table} The Spearman rank correlation ($\rho$) between $\Omega$ and a surrogate score. $p$ is the p-value of the corresponding two-sided significance test. }
\begin{ruledtabular}
\begin{tabular}{llccc}
Collection & Score & $\rho$ & $p$ \\
\hline
\input{figures/surrogate_table}
\end{tabular}
\end{ruledtabular}
\end{table}

A pressing question is if $\Omega$ could be replaced successfully by a simpler score in spite of all the claims in Section \ref{optimality_metric_section} about the outstanding mathematical and statistical properties of $\Omega$ compared to other scores. We approach this questions first by looking at the correlation between the sequential ranking of languages that $\Omega$ produces and the sequential ranking of languages produced by another score. Table \ref{surrogate_table} shows that only the ranking of languages by $D_z$ or $\Gamma$ are correlated significantly with that of $\Omega$; the best surrogate would depend on the annotation style. Consequently, $D$, $\bar{d}$ and $\Delta$ must be discarded as proxies for $\Omega$.
As the sequential ranking may sort arbitrarily languages that do not differ significantly in the value of the score, we consider again a hierarchical ranking. For simplicity, we focus on $D$ because it is a widely used score \cite{Gildea2007a,Futrell2015a,Futrell2020a}. As expected from Table \ref{surrogate_table}, the hierarchical ranking changes substantially (Fig. \ref{Hasse_diagrams_D_figure}). $D$ contradicts $\Omega$ in many ways (Fig. \ref{Hasse_diagrams_D_figure}). $D$ concludes that the most optimized language is Finnish and the least optimized language is Japanese regardless of the annotation style. Besides, the Romance languages have moved now to lower positions and are considered to be less optimized than English.  

\begin{figure}
\centering
\includegraphics[width=\linewidth]{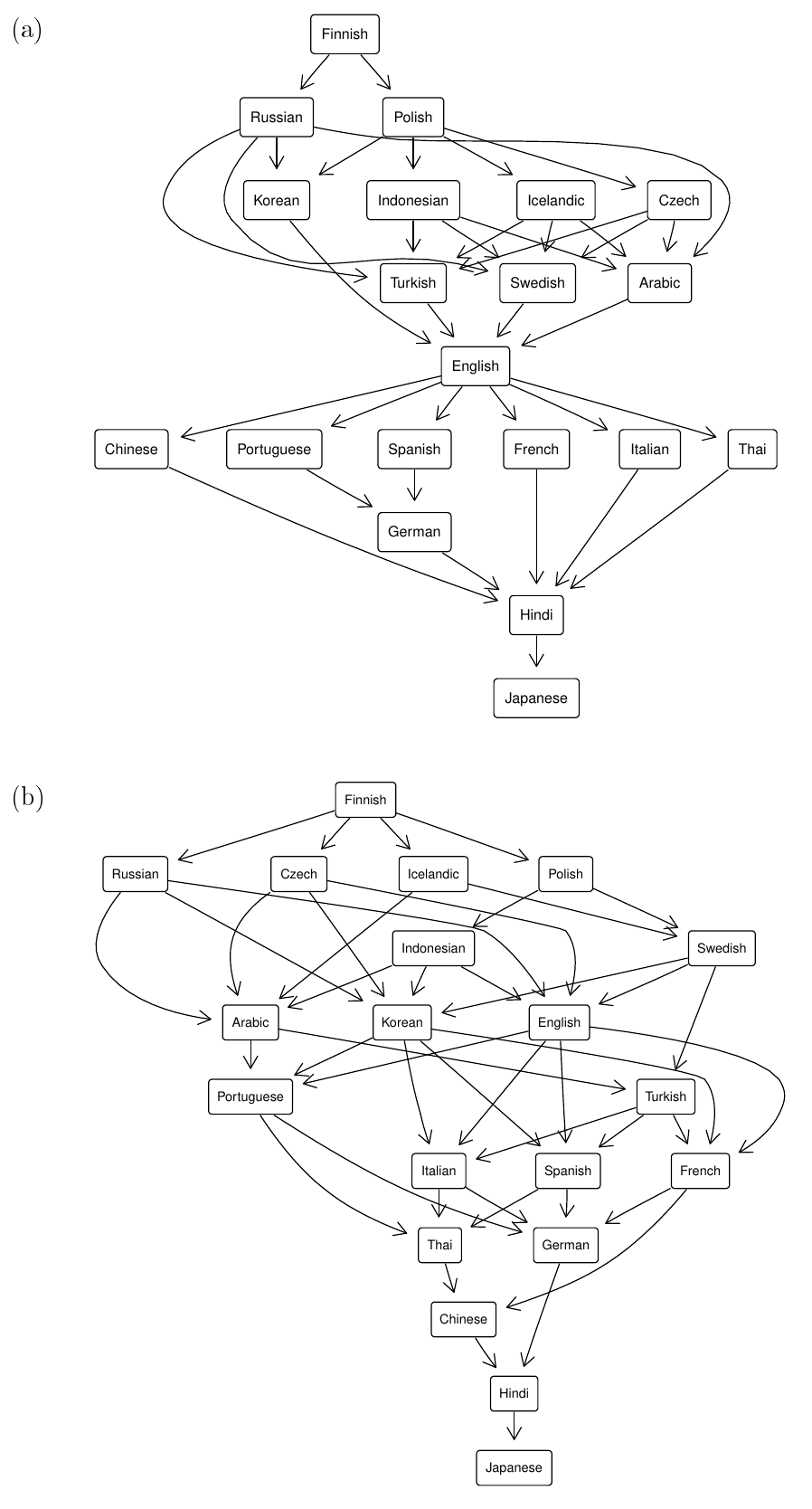}
\caption{\label{Hasse_diagrams_D_figure} The Hasse diagrams that would be obtained if $\Omega$ was replaced by $D$ in Fig. \ref{Hasse_diagrams_figure}. $\Omega$ is actually replaced by $-D$ and then the languages on top are the ones that have shorter dependency distances and the languages at the bottom have longer dependencies. (a) PUD. (b) PSUD. }
\end{figure}

%% file: figures/optimality_by_family_table.tex
UD & All & 92 & -0.22 & 0.66 & 0.7 & 0.9 \\
 & Indo-European & 51 & 0.24 & 0.68 & 0.7 & 0.9 \\
 & Uralic & 11 & 0.54 & 0.63 & 0.62 & 0.77 \\
 & Afro-Asiatic & 7 & 0.63 & 0.78 & 0.79 & 0.89 \\
 & Altaic & 3 & 0.49 & 0.5 & 0.5 & 0.51 \\
 & Sino-Tibetan & 3 & 0.51 & 0.62 & 0.62 & 0.72 \\
 & Austronesian & 2 & 0.4 & 0.63 & 0.63 & 0.86 \\
 & Dravidian & 2 & 0.04 & 0.39 & 0.39 & 0.75 \\
 & Niger-Congo & 2 & 0.79 & 0.83 & 0.83 & 0.87 \\
 & Austro-Asiatic & 1 & 0.79 & 0.79 & 0.79 & 0.79 \\
 & Basque & 1 & 0.71 & 0.71 & 0.71 & 0.71 \\
 & Japanese & 1 & 0.73 & 0.73 & 0.73 & 0.73 \\
 & Korean & 1 & 0.6 & 0.6 & 0.6 & 0.6 \\
 & Mande & 1 & 0.66 & 0.66 & 0.66 & 0.66 \\
 & Mongolic & 1 & 0.59 & 0.59 & 0.59 & 0.59 \\
 & Other & 1 & 0.69 & 0.69 & 0.69 & 0.69 \\
 & Pama-Nyungan & 1 & -0.22 & -0.22 & -0.22 & -0.22 \\
 & Sign Language & 1 & 0.51 & 0.51 & 0.51 & 0.51 \\
 & Tai-Kadai & 1 & 0.88 & 0.88 & 0.88 & 0.88 \\
 & Tupian & 1 & 0.57 & 0.57 & 0.57 & 0.57 \\
SUD & All & 92 & 0.05 & 0.71 & 0.74 & 0.91 \\
 & Indo-European & 51 & 0.24 & 0.74 & 0.77 & 0.91 \\
 & Uralic & 11 & 0.61 & 0.7 & 0.71 & 0.79 \\
 & Afro-Asiatic & 7 & 0.53 & 0.77 & 0.86 & 0.9 \\
 & Altaic & 3 & 0.42 & 0.44 & 0.43 & 0.48 \\
 & Sino-Tibetan & 3 & 0.59 & 0.66 & 0.65 & 0.74 \\
 & Austronesian & 2 & 0.52 & 0.7 & 0.7 & 0.88 \\
 & Dravidian & 2 & 0.05 & 0.37 & 0.37 & 0.7 \\
 & Niger-Congo & 2 & 0.83 & 0.87 & 0.87 & 0.91 \\
 & Austro-Asiatic & 1 & 0.82 & 0.82 & 0.82 & 0.82 \\
 & Basque & 1 & 0.66 & 0.66 & 0.66 & 0.66 \\
 & Japanese & 1 & 0.73 & 0.73 & 0.73 & 0.73 \\
 & Korean & 1 & 0.56 & 0.56 & 0.56 & 0.56 \\
 & Mande & 1 & 0.74 & 0.74 & 0.74 & 0.74 \\
 & Mongolic & 1 & 0.55 & 0.55 & 0.55 & 0.55 \\
 & Other & 1 & 0.8 & 0.8 & 0.8 & 0.8 \\
 & Pama-Nyungan & 1 & 0.42 & 0.42 & 0.42 & 0.42 \\
 & Sign Language & 1 & 0.72 & 0.72 & 0.72 & 0.72 \\
 & Tai-Kadai & 1 & 0.89 & 0.89 & 0.89 & 0.89 \\
 & Tupian & 1 & 0.61 & 0.61 & 0.61 & 0.61 \\
Stanford & All & 30 & -0.01 & 0.65 & 0.71 & 0.86 \\
 & Indo-European & 21 & 0.27 & 0.68 & 0.71 & 0.86 \\
 & Uralic & 3 & 0.6 & 0.67 & 0.71 & 0.72 \\
 & Dravidian & 2 & -0.01 & 0.37 & 0.37 & 0.75 \\
 & Afro-Asiatic & 1 & 0.84 & 0.84 & 0.84 & 0.84 \\
 & Altaic & 1 & 0.52 & 0.52 & 0.52 & 0.52 \\
 & Basque & 1 & 0.74 & 0.74 & 0.74 & 0.74 \\
 & Japanese & 1 & 0.4 & 0.4 & 0.4 & 0.4 \\
Prague & All & 30 & -0.09 & 0.7 & 0.77 & 0.9 \\
 & Indo-European & 21 & 0.25 & 0.72 & 0.77 & 0.9 \\
 & Uralic & 3 & 0.56 & 0.69 & 0.71 & 0.78 \\
 & Dravidian & 2 & -0.09 & 0.34 & 0.34 & 0.78 \\
 & Afro-Asiatic & 1 & 0.88 & 0.88 & 0.88 & 0.88 \\
 & Altaic & 1 & 0.6 & 0.6 & 0.6 & 0.6 \\
 & Basque & 1 & 0.77 & 0.77 & 0.77 & 0.77 \\
 & Japanese & 1 & 0.73 & 0.73 & 0.73 & 0.73 \\

%% file: figures/patterns_in_optimality_table.tex
UD & 92 & 90  & 2  & Dravidian (1) & Telugu$^{0.1}$ \\
& & & & Pama-Nyungan (1) & Warlpiri$^{1}$ \\
SUD & 92 & 92  & --- & --- & --- \\
Prague & 30 & 29  & 1  & Dravidian (1) & Telugu$^{1}$ \\
Stanford & 30 & 29  & 1  & Dravidian (1) & Telugu$^{0.7}$ \\

%% file: figures/correlation_with_sentence_length_table_Omega_greater.tex
UD & 92 & 88  & 4  & Afro-Asiatic (1) & Assyrian$^{0.1}$ \\
& & & & Indo-European (1) & Albanian$^{0.1}$ \\
& & & & Pama-Nyungan (1) & Warlpiri$^{0.1}$ \\
& & & & Sino-Tibetan (1) & Classical Chinese$^{0.2}$ \\
SUD & 92 & 89  & 3  & Indo-European (1) & Albanian$^{0.1}$ \\
& & & & Pama-Nyungan (1) & Warlpiri$^{0.4}$ \\
& & & & Uralic (1) & Komi-Permyak$^{0.1}$ \\
Prague & 30 & 30  & --- & --- & --- \\
Stanford & 30 & 30  & --- & --- & --- \\

%% file: figures/anti_dependency_distance_table_Omega.tex
3  & UD & 87 & 12  & Austronesian (1) & Tagalog$^{3.1}$ \\
 & & & & Dravidian (1) & Telugu$^{1.8}$ \\
 & & & & Indo-European (9) & Afrikaans$^{3.1}$ Bhojpuri$^{3.1}$ Breton$^{1.6}$ German$^{3.1}$ English$^{3.1}$ Old French$^{3.1}$ Norwegian$^{2.5}$ Sanskrit$^{2.5}$ Slovak$^{3.1}$ \\
 & & & & Japanese (1) & Japanese$^{3.1}$ \\
   & SUD & 87 & 4  & Austronesian (1) & Tagalog$^{3.1}$ \\
 & & & & Indo-European (3) & Bhojpuri$^{3.1}$ Persian$^{2.6}$ Urdu$^{3.1}$ \\
   & Prague & 30 & 3  & Dravidian (1) & Telugu$^{3.5}$ \\
 & & & & Indo-European (2) & Hindi$^{3.5}$ Persian$^{3.5}$ \\
   & Stanford & 30 & 5  & Dravidian (1) & Telugu$^{3.5}$ \\
 & & & & Indo-European (3) & Czech$^{3.5}$ German$^{3.5}$ Slovak$^{3.5}$ \\
 & & & & Japanese (1) & Japanese$^{3.5}$ \\
4  & UD & 91 & 3  & Indo-European (2) & Swiss German$^{3}$ Sanskrit$^{2.2}$ \\
 & & & & Pama-Nyungan (1) & Warlpiri$^{2.3}$ \\
   & SUD & 91 & 2  & Indo-European (2) & Swiss German$^{3}$ Sanskrit$^{2.7}$ \\
   & Prague & 30 & 3  & Dravidian (1) & Telugu$^{3.5}$ \\
 & & & & Indo-European (2) & Ancient Greek$^{2}$ Persian$^{1.5}$ \\
   & Stanford & 30 & 0  & --- & --- \\

%% file: figures/surrogate_table.tex
PUD & $D$ & 0.146 & $ 0.54 $ \\
 & $\bar{d}$ & -0.389 & $ 0.09 $ \\
 & $\Delta$ & -0.26 & $ 0.27 $ \\
 & $\Gamma$ & -0.774 & $ 9.10 \cdot 10^{-5} $ \\
 & $D_z$ & -0.899 & $ 1.18 \cdot 10^{-6} $ \\
PSUD & $D$ & -0.074 & $ 0.76 $ \\
 & $\bar{d}$ & -0.432 & $ 0.06 $ \\
 & $\Delta$ & -0.355 & $ 0.13 $ \\
 & $\Gamma$ & -0.889 & $ 0.00 $ \\
 & $D_z$ & -0.783 & $ 6.31 \cdot 10^{-5} $ \\

%% file: discussion.tex
\subsection*{Question 1. Are there universal optimization principles in languages? Or, are there languages showing no optimization at all?}

Here we have added \IncreaseFutrellEtAlLanguage languages to the broadest previous study of DDm in languages \cite{Liu2008a, Futrell2015a, Futrell2020a}. In our analysis based on 
$\Omega$, we have found that 
dependency distances are not significantly small in only 2 languages: Telugu and Warlpiri (Table \ref{patterns_in_optimality_table}). In contrast, previously existing optimality scores only indicate that dependency distances are not significantly small in Warlpiri (Appendix \ref{results_appendix}), suggesting that $\Omega$ has a slightly greater capacity to detect exceptions than other scores. 
Exceptions have been previously reported but using different methods \cite{Futrell2020a}.
This suggests that conclusions on whether a language has been optimized or not may depend, at least, on the score used, or on the distribution of sentence lengths \cite{Ferrer2013c}. The latter is not a mere statistical issue (controlling for sentence length): we have shown that $\Omega$ is significantly small in short sequences of these two languages (Table \ref{anti_dependency_distance_table}), strongly suggesting that DDm is surpassed by other word order principles in short sentences \cite{Ferrer2019a}. However, all these issues should be the subject of future research.

The languages where $\optimalityL$ is not significantly large (Table \ref{patterns_in_optimality_table}) could be interpreted as challenging the universality of DDm. 
However, such interpretation should be examined in light of two observations. 
First, we should distinguish between principles and their manifestations \cite{Ferrer2012d,Ferrer2015b}. The fact that a principle does not manifest, does not invalidate the universality of the principle \cite{Ferrer2012d}. DDm may be acting but counteracted by other word order principles, e.g., surprisal minimization \cite{Ferrer2013f}. This issue is developed further below. These considerations notwithstanding, the true number of exceptions could be larger but hidden by biases in the families in our dataset. Second, we cannot exclude the possibility that the cases where DDm does not surface are not due to genuine properties of the languages but rather a case of the comparability or homogeneity problem: e.g. a consequence of the annotation criteria beyond those specified in the dataset or biases in the sample of sentences forming the collection (for instance, the author or the genre), that hide the manifestation of DDm (this is issue is also developed below). 

\subsection*{Question 2. What is the degree of optimization of languages?}

In two of the large-scale studies referred above \cite{Liu2008a, Futrell2015a}, the values of the plain score used did not reveal that the effects of DDm were too weak or missing. Our score has indicated negative values or values close to zero in some cases (Table \ref{optimality_by_family_table}).
However, we have found that the majority of languages are highly optimized. 
When evaluating the degree of optimization of languages, there is a trade-off between the breadth of the analysis and comparability. One of the broadest datasets, i.e. UD, allows one to conclude that half of languages are optimized to at least $70\%$ but $\optimalityL$ ranges between $-0.22$ and $0.9$ (Table \ref{optimality_by_family_table}). When controlling for annotation style using SUD, one finds that half of languages are optimized to at least $74\%$ and languages with negative $\optimalityL$ have disappeared.
Although annotation style within a dataset is homogeneous {\em a priori}, the source texts are not. When reducing the comparability problem using PUD, one finds that the degree of optimality ranges between $64\%$ and $88\%$ (Fig. \ref{ranking_PUD_figure}) but more than half of the languages available are Indo-European (only Arabic, Thai, Indonesian, Japanese, Finnish, Chinese, Korean and Turkish are not).   
We hope that our research stimulates the development of large parallel corpora where linguistic families are more balanced.

In his pioneering research, Hawkins concluded that the word order of English is $92\%$ optimal \cite{Hawkins1998a}. Here we have used $\Omega$, an optimality score based on $D$, showing that English is optimized to a $81\%$ (according to PUD, Fig. \ref{ranking_PUD_figure}) or $88\%$ (according to PSUD, Appendix \ref{results_appendix}) with respect to two baselines (the random baseline and the optimal baseline). These numbers must be taken with caution because other sources would give different percentages ($70\%$ according to UD, $80\%$ according to SUD and Prague and $85\%$ according to Stanford). 

$\Omega$ could be one of the dimensions of new research on patterns in the distribution of languages in bidimensional maps of word order features \cite{Gerdes2020a}.

\subsection*{Question 3. What are the most and the least optimized languages of the world? Can we define a ranking of languages?}

We have provided a first approximation to the ranking of languages by dependency distances using a small but parallel sample of languages, finding that Romance languages are among the most optimized languages, while Korean and Turkish, two agglutinating languages \cite{Kim2009a}, are among the least optimized languages (Fig. \ref{ranking_PUD_figure}). Such ranking is heavily limited by the number of languages sampled and not indicative of the optimality of language in absolute terms because dependency distances are just one aspect of the optimality of word order \cite{Ferrer2013f, Gildea2020a, Hahn2020a} and, in turn, word order is not the only aspect that determines the degree of optimization of a language \cite{Ferrer2015b,Fedzechkina2017a,Gibson2019a}.
For instance, Turkish and Korean order the subject, verb and object in a cognitively natural order, i.e. SOV, differing from the default of Romance languages, i.e. SVO \cite{Goldin-Meadow2008a}, and are endowed with complex morphology \cite{Kim2009a}, a feature that 
can reduce the effort of the listener \cite{Fedzechkina2012a,Ros2015a} and may ease learning in native speakers \cite{Dale2012a,Kempe2018a}.

Furthermore, we cannot exclude the possibility that sentences in the PUD collection are affected by some bias in the original language or in the translator. Put differently, the parallel corpus may have diminished the homogeneity or comparability problem but not cleared it completely.  Finally, the ranking relies on the granularity of wordness criteria (discussed in Appendix~\ref{materials_appendix}) used to assign ``words'' to vertices of the syntactic dependency structure as well as the unit of measurement of distance, i.e. words. 
A more precise but far more complicated measurement could be made in syllables or phonemes \cite{Ferrer2014e, Ferrer2021a}.
Therefore, any statement or inference that one language is more (or less) optimized than another in the entire article must be understood as relying exclusively on $\Omega$ (even when $\Omega$ is not mentioned for the sake of brevity) and conditioned on the collection of sentences used, the granularity of words and the unit of distance.  

Notwithstanding the above considerations, our ranking outperforms previous ordering attempts. We have put forward a hierarchical ranking rather than a sequential ordering \cite{Liu2008a}. A normalized measure of dependency distance as $\Omega$ has never been used \cite{Liu2008a,Gildea2010a}. 
To our knowledge, the first attempt to rank languages by dependency distance is found in the pioneering research of H. Liu, who sorted 20 languages by mean dependency distance using non-parallel corpora and implicitly assuming that the order of languages was total in a mathematical sense \cite[Figures 7-8]{Liu2008a}. We have found that English is more optimized than German in agreement with previous research using non-parallel corpora and $D$ instead of $\Omega$ \cite{Gildea2010a} (as well as in the pioneering research of \cite{Liu2008a}). 
The consensus in Fig. \ref{Hasse_diagrams_figure} goes much further indicating which languages are more likely 
\begin{itemize}
\item 
To be less optimized than German: Turkish, that is the only language in the intersection between the set of languages that are less optimized than German in Fig. \ref{Hasse_diagrams_figure} (a), namely Turkish, with those that are less optimized than German in Fig. \ref{Hasse_diagrams_figure} (b), i.e. Chinese, Turkish and Korean.
\item
To have an intermediate level of optimization between English and German: Czech and Finnish (the intersection between Czech, Finnish, Hindi from Fig. \ref{Hasse_diagrams_figure} (a) and Czech, Finnish and Japanese from Fig. \ref{Hasse_diagrams_figure} (b)).
\item
To be more optimized than English: Portuguese, Italian, French and Spanish, by similar arguments. 
\end{itemize}
Taking a reductionistic approach to the complexity of languages, we have investigated the effect of removing all function words assuming that nothing else needs to be done for a fair comparison of languages once they are removed. We have shown that the ranking of languages is robust with respect to the proportion of function words in the language (Fig. \ref{correlogram_figure}). 
  
Furthermore, the diagram unveils a myriad of dominance relationships that have never been explored before, e.g., Russian is likely to be more optimized than Chinese or Korean based on $\Omega$. 
To understand the power of the new ranking method that we have developed, notice that the parallel data that we have used comprises 20 languages and for that reason, there are {\em a priori} 
\begin{equation}
2{20 \choose 2} = 380
\end{equation}
questions that can be made on whether language $x$ is more optimized than language $y$, which may potentially result in
\begin{equation}
{20 \choose 2} = 190
\end{equation}
relationships indicating that language $x$ is more optimized than language $y$.
\cite{Gildea2010a} only checks 1 out of 190 (German versus English; two Germanic languages) with a score ($D$) that is statistically problematic as we explain in our article. By analyzing the structure of the graph, one concludes that Figure 7 (a) shows that language $x$ is more optimized than language $y$ for 149 pairs of languages while Figure 7 (b) does it for 155 pairs of languages. This is a quantum leap for research in quantitative typology.

We hope that our research stimulates the development of parallel syntactic dependency treebanks comprising a sample of languages as large as possible. 
 
\subsection*{Question 4. Under what conditions is optimization stronger?}

From a theoretical standpoint, two predictions have been made: one is that pressure for DDm should be stronger in longer sentences, and the other is that DDm is more likely to be surpassed in short sequences by other principles \cite{Ferrer2014a, Ferrer2019a}. The second prediction is based on the first prediction and the fact that DDm is in conflict with other principles of word order \cite{Ferrer2014a,Ferrer2013f,Ferrer2019a}. If we focus on the optimal placement of a single head and its dependents, it turns out that the head must be put at the center according to DDm but it should be put at one of the ends according to the principle of surprisal minimization (Sm) or predictability maximization (PM): last to minimize the uncertainty about the head, or first to minimize the uncertainty about the dependents \cite{Ferrer2013f}. Interestingly, putting the head at one of the ends maximizes dependency lengths, i.e. the optimal placement of the head according to Sm/PM is the worst case for DDm \cite{Ferrer2013f}.

Our analysis has corroborated these two predictions. As for the first prediction, the finding that $\Omega$ is higher in longer sentences (Fig. \ref{heat_map_figure} and Table \ref{correlation_with_sentence_length_table}), indicates that longer sentences are more optimized. Such a tendency defies a trivial explanation. In shuffled sentences (no DDm), $\Omega$ is expected to be $0$ (recall the stability under random linear arrangement). As in the other (preexisting) scores lower values indicate higher optimization, the corresponding trend for the other scores would be a tendency of the score to reduce in longer sentences, but that was found in no score and no language except for $D_z$ (Appendix \ref{results_appendix}). This is consistent with the fact that,
in real sentences, $D$, one of the components of $\Omega$, as well as $\meanD$, tends to grow with sentence length \cite{Ferrer2004b, Park2009a, Ferrer2013c, Jiang2015a, Futrell2015a, Ferrer2020a}. The same tendency was found for $\Gamma$ in Romanian \cite{Ferrer2004b} and for $\Delta$ in Latin and ancient Greek \cite{Gulordava2015} (as for the latter notice that $D_{min}$ was defined assuming projectivity). Our analyses confirm a tendency of all scores except $\Omega$ and $D_z$ to increase in longer sentences: in all languages concerning $D$ and $\meanD$ and $\Delta$, and in fraction of them for $\Gamma$ (Appendix \ref{results_appendix}). Thus, by evaluating optimality based on $D$, $\Delta$ or $\Gamma$, rather than with $\Omega$, one could wrongly conclude that longer sentences are less optimized. Such disagreement with respect to $\Omega$ can be explained by the fact that these scores do not include the random baseline (Table \ref{scores_table}).

As for the second prediction, we found that $\Omega$ is significantly small in short sentences (sentences of length 3 or 4) for some languages, contradicting DDm (Table \ref{anti_dependency_distance_table}). 
Interestingly, we found more languages showing anti-DDm effects than in previous research \cite[Table 3]{Ferrer2019a}: the number of languages where this happens grows in all cases (SUD had not been considered before). That could be explained by an increase in the number of languages in the UD dataset with respect to \cite{Ferrer2019a}, but not in the Prague and Stanford datasets, whose number of languages remains the same. Therefore, $\Omega$ (in particular, $\optimalityn$) appears to have more statistical power to unveil anti-DDm effects than all the other scores (Appendix \ref{results_appendix}) as well as the statistic used in previous research \cite{Ferrer2019a}. 
Finally, our findings clarify the discussion on the existence of anti-DDm effects in treebanks \cite{Lei2020a}: they do exist in sentences of length 3 or 4.

\subsection*{Question 5. Why do languages deviate from optimality?}

A crucial question is: if dependency length minimization is a fundamental principle of languages, why are real languages not reaching the theoretical minimum sum of dependency lengths? 
The first answer stems from the arguments above: a conflict between DDm and other word order principles. 
Another possible answer is that language production is an inherently online and incremental process under severe memory limitations, in line with Christiansen and Chater's now-or-never bottleneck \cite{Christiansen2015a}: 
real speakers produce sentences on the fly while algorithms to calculate minimum linear arrangements take whole sentences as input \cite{Chung1984, Shiloach1979}.

Finally, from an evolutionary standpoint, deviations from optimality can originate for many reasons that are fuel for future research, e.g. stochasticity, 
insufficient time to reach the optimum or the existence of local maxima 
\cite{Perez-Escudero2009a}. Furthermore, parts of a system with a smaller contribution to the cost function are expected to show a largest deviation from the optimum \cite{Perez-Escudero2009a}. 

\subsection*{Question 6. What are the best optimality scores for cross-linguistic research?} 

In this article, we investigated mathematical properties of $\Omega$ and compared them against other scores. Our empirical analysis suggests that $\Omega$ has a greater capacity to detect languages where DDm is weak or surpassed by other word order principles (Question 1), unveil higher optimization in longer sentences (Question 4) and detect anti-DDm effects in short sequences (Question 4). 
These findings, in combination with the many theoretical caveats of raw dependency distances that we have unveiled in Section \ref{optimality_metric_section}, strongly suggest that  
raw dependency distances are a poor reflect of DDm and that scores such as $\Omega$ are crucial for progress in research on that optimization principle and related memory constraints \cite{Liu2008a,Jiang2015a}. Finally, we have shown that in terms of the ranking of languages by optimality, $\Omega$ cannot be replaced by the widely-used $D$ \cite{Gildea2007a,Gildea2010a,Futrell2015a,Futrell2020a}, the closely-related $\bar{d}$ \cite{Ferrer2004b} (neither by $\Delta$ \cite{Gulordava2015}). We have shown the hierarchical ranking of languages by $D$ is actually distorted substantially with respect to the original ranking by $\Omega$. All these findings on the superiority of $\Omega$ strongly suggest that the theoretical properties of $\Omega$ actually (Section \ref{optimality_metric_section}) imply a remarkable difference with respect to other scores in a practical setting. For this reason, it is safer to conclude that Romance languages are more optimized then English based on $\Omega$ (Fig. \ref{Hasse_diagrams_figure})) rather than the other way around based on $D$ (Fig. \ref{Hasse_diagrams_D_figure})).
However, further theoretical and empirical research should be carried out in relation to {$\Omega$ and alternative optimality} scores. 

\subsection*{Question 7. What is the contribution of function words to the optimization of languages?} 

We have investigated the role of function words in DDm in two ways. First, we have investigated the effect of how function words are linked considering two rather opposite annotation styles: UD, which is content-head, and SUD, which is function-head \cite{sud}. In addition to confirming that dependency distances are shorter in SUD than in UD \cite{Osborne2019a} (Table \ref{optimality_by_family_table}), we have found that all languages (with no exception) are significantly optimized only when SUD is used (Table \ref{patterns_in_optimality_table}) and that SUD shows fewer languages with anti-DDm effects (Table \ref{anti_dependency_distance_table}), as expected by the direct influence of DDm in the annotation criteria of SUD \cite{sud}.
Second, we have also shown that the removal of function words reduces the optimality of dependency distances in practically all languages of the parallel collection. 
The outcomes of these two approaches to the role of function words strongly suggest that function words help languages to reduce dependency distance. Further research is needed to clarify if dependency distance minimization is a driving force during grammaticalization \cite{Hopper1991a}.

\subsection*{Limitations, connections with different fields and future work}

The score has implications for many fields of language research and network science based on its power and also on current limitations.

\paragraph{Dependency linguistics and dependency treebanks}

In this article we have investigated the optimality of languages using doculects as proxies. A large collaborative effort is needed to reduce the distance between the doculects and to reduce the scarcity of samples from more naturalistic context, e.g. oral language. 
As for reducing the distance between doculects, a crucial task is increasing the number of languages available in the parallel collections and reinforcing the homogeneity of wordness criteria (Appendix \ref{materials_appendix}).

Our findings have implications for the debate on the proper syntactic annotation \cite{sud,Osborne2019a,Yan2019a}. We have shown that SUD shows DDm effects more easily (see Question 7 above). Therefore, whether DDm is a universal (exceptionless) phenomenon on a global scale or the scope of anti-DDm effects may depend on the annotation criteria. Further research on the foundations of these criteria is necessary \cite{Ferrer2002f,Futrell2019a}.

\paragraph{Quantitative typology} We have already shown the capacity of the the new score to rank languages. 
We have investigated the optimality of languages syncronically, based on current data about them. One of the least optimized languages is Warlpiri, a non-configurational language \cite{Hale1983a}. The finding, though based on non-parallel data, suggests that $\Omega$ is able to quantify extreme cases of word order freedom \cite{Evans2009a} and thus could open new avenues for research in linguistics.
When using parallel data, we have found that Korean and Turkish, two verb-final languages, are consistently the least optimized languages (Fig. \ref{Hasse_diagrams_figure}). The combination of $\Omega$ and parallel data would allow one to quantify, in a statistically rigorous way, the degree of optimization of verb-initial, verb-medial and verb-final languages.
Diacrony, namely the evolution of the optimality of a language over time, is one of the next challenges \cite{BingliLiu2017a}. 

\paragraph{Historical linguistics} 
It has been shown that dependency distances have been reducing over time in English using a variant of $\Gamma$ \cite{Tily2010a} and also using $D$ and $NDD$ \cite{Lei2020a}. 
$\Delta$ has unveiled a similar optimization process in Latin and Ancient Greek \cite{Gulordava2015}. In contrast, dependency distances have been increasing over time in Chinese using a variant of $\meanD$ \cite{BingliLiu2013a,BingliLiu2017a}. An intriguing question is whether all these conclusions will remain valid when revised using a dual optimality score such as $\Omega$ (Table \ref{scores_table}). The next frontier will be unveiling how DDm has shaped languages throughout phylogeny \cite{Dunn2011a}.

\paragraph{Cognitive science and clinical linguistics}

A limitation of the view of dependency distance as a cognitive cost is that its evidence comes only from corpus studies, computational simulations, and experiments on comprehension \cite{Liu2017a,Temperley2018a}.
As far as we know, only partial information about dependency distances (not distances over the whole sentence as measured by $D$) has been used in experiments on comprehension with real subjects \cite{Ferrer2014f}. 
$\Omega$, as a measure the degree of optimality of a whole sentence from a dependency distance perspective, inherits all limitations of past research supporting dependency distance minimization.
Such limitation also concerns our score $\Omega$, that has been supported using a combination of corpus analysis and various kinds of theoretical approaches. 
Furthermore, it is yet to be demonstrated that $\Omega$ approximates the real cognitive effort of the human processor better than the existing alternative measures. Future research should provide evidence that $\Omega$ is more strongly correlated with behavioral measures (e.g., reading or reaction times) and neural measures in the brain (e.g., event-related potentials or advanced neuroimaging methods).   

In addition to functioning as a measure of the degree of optimization of a language, $\Omega$ can also shed light on a critical question for clinical applications \cite{Cohen2014a}: 
whether dependency distances reflect the linguistic competence of an individual speaker. Research on a potentially positive correlation between dependency distances and language proficiency in second language learners has led to inconsistent conclusions using mean dependency distances \cite{Ouyang2017a,Komori2019a}. Such inconsistencies may have originated from the limitations of the score that has been used. More solid and coherent conclusions may be reached using $\Omega$.

\paragraph{Network science} 
Our theoretical research is relevant for network science in general. We have put forward a new way of normalizing physical distances. The analysis can be extended to arbitrary graphs or spaces with more dimensions, e.g. the 2-dimensional layouts of the popular transportation networks 
\cite{Barthelemy2011a}, and it may guide future research on the normalization of topological distances. It has been proposed that these distances should be normalized as \cite{Zamora-Lopez2019a} 
\begin{eqnarray*}
\lambda' & = & \frac{L}{L_{min}} \\ 
\lambda'' & = & \frac{L}{L_{max} - L_{min}},
\end{eqnarray*}
where $L$ is the observed topological distance and $L_{min}$ and $L_{max}$ are the minimum and the maximum topological distance of a network with the same number of vertices and edges. Notice that $\lambda'$ is the counterpart of $\Gamma$ in topological distance and is thus a singular optimality score for not incorporating the random baseline. Interestingly, $\lambda''$ does not incorporate any random baseline but is defined over two baselines ($L_{min}$ and $L_{max}$). Powerful normalizations of $L$ may be developed following the mathematical structure of $\Omega$, e.g., incorporating some random graph model \cite{Newman2010a,Barthelemy2011a}.

%% file: matmethods.tex

Each dataset is a combination of a collection and an annotation style. See Appendix \ref{materials_appendix} for further details about them. 
\textcolor{red}{Preprocessed data and code are deposited in GitHub (\url{https://github.com/lluisalemanypuig/optimality-syntactic-dependency-distances}).} 

The analyses are restricted to $n \geq 3$ because a measure of optimality such as $\Omega$ is useful when $D$ can vary given $n$ (for $n = 1$, $D = 0$; for $n = 2$, $D = 1$). See Appendix \ref{maths_appendix} for further details and the implications of such a constraint.  

The optimality score $\Omega$ is defined on two baselines: the random baseline ($D_{rla}$) and the optimal baseline ($D_{min}$). Both baselines are unrestricted in the sense that all shufflings are allowed (the $n!$ permutations have non-zero probability).  
Random and optimal baselines that assume projectivity or consistent branching, hence disallowing certain shufflings, have dominated research on DDm
\cite{Liu2008a,Gildea2007a, Park2009a,Gildea2010a,Futrell2015a,Futrell2020a} 
and optimality scores \cite{Tily2010a,Gulordava2015} for historical reasons. 
However, these linguistic constraints (e.g., projectivity) could be an epiphenomenon of DDm \cite{Gomez2019a,Gomez2016a,Ferrer2013e,Ferrer2008e} and thus could mask the effect of dependency distance minimization (DDm). Our concern is supported by the fact that planarity, a relaxation of projectivity, reduces the statistical power of a test of DDm \cite{Ferrer2019a}. From a theoretical standpoint, separating DDm from these additional linguistic constraints compromises the parsimony of word order theory \cite{Ferrer2016a}.


The words of a sentence are shuffled replacing the linear arrangement of the sentence with a uniformly random linear arrangement using Durstenfeld's algorithm \cite{Durstenfeld1964a}.


The positive association between $\Omega$ and $n$ is analyzed through a one-sided correlation test with Kendall's $\tau$ as statistic. Kendall's $\tau$ is chosen for its advantages over the traditional Pearson correlation: while Pearson correlation is a measure of linear association, Kendall's $\tau$ correlation is a measure of non-linear (but monotonic) association \cite{Gibbons2010a,Embrechts2002a}.   
For a certain language, the test is applied on the average value of $\Omega$ as a function of $n$, namely, an estimate of the conditional expectation of $\Omega$ given $n$, which is in direct correspondence with the heatmaps in Fig. \ref{heat_map_figure}. 
Without such preprocessing, the test would be sensitive to the distribution of sentence lengths, which is in turn dependent on features of the source texts such as genre and modality \cite{Rudnicka2018a}.
The same procedure is used for the association between other scores and $n$.

We use the significance level of 0.05. 
All statistical tests are one-sided. The Monte Carlo test explained in Appendix \ref{methods_appendix} is used to check if $\optimalityL$ or $\optimalityn$ are significantly large (or significantly small) with respect to the random baseline. A Monte Carlo test is also used to assess whether the values of $\optimalityL$ of two languages differ significantly, a crucial step for the construction of the Hasse diagram. Monte Carlo tests are chosen to turn the tests non-parametric (assumption free) while maximizing statistical power. Appendix \ref{methods_appendix} explains how the Hasse diagram is built in detail, how the Holm correction is implemented and how the $p$-values from Monte Carlo tests are preprocessed before applying the Holm correction.

%% file: SI_maths.tex
\subsection{The new optimality score}

For a tree $t$ of $n$ vertices, our optimality metric is  
\begin{equation}
\Omega^t = \frac{D_{rla} - D^t}{D_{rla} - D_{min}^t}
\label{raw_optimality_metric_tree_equation}
\end{equation}
with $n \geq 3$ and $D_{rla}$ defined as in Eq. \ref{sum_of_dependency_lengths_random_equation}. 

\subsubsection{The maximum value of $\Omega$}

Showing that $\Omega^t \leq 1$ for $n\geq 3$ is straightforward.  
By definition, $D_{min}^t \leq D^t$ and then  
\begin{equation*}
D_{rla} - D^t  \leq D_{rla} - D_{min}^t.
\end{equation*} 
Dividing both sides of the previous inequality by $D_{rla} - D_{min}$ and recalling that $D_{rla} \geq D_{min}$ \cite{Esteban2016a}, 
\begin{equation*}
\frac{D_{rla} - D^t}{D_{rla} - D_{min}^t} = \Omega^t \leq 1.
\end{equation*}  
Therefore, $\Omega^t$ is constant under minimum linear arrangement. 

\subsubsection{A lower bound of $\Omega$}

$D_{max}^t$, the maximum value of $D^t$ that a tree $t$ can achieve over the $n!$ linear arrangements of its vertices, satisfies \cite{Ferrer2020a}
\begin{eqnarray}
D_{max}^t & \leq & D_{max}^{b-bistar} \nonumber \\ 
          & =    & \frac{1}{4}\left[3(n-1)^2 + 1 - n \bmod 2\right],
\label{extreme_sum_of_edge_lengths_equation}
\end{eqnarray}
where $D_{max}^{b-bistar}$ is the value of $D_{max}^t$ of a balanced bistar ({\em b-bistar}) tree. 
The next theorem gives a lower bound of $\Omega^t$ as a function of the size of the tree that follows from Eq. \ref{extreme_sum_of_edge_lengths_equation}.

\begin{thm}
Let $\Omega_{min}^t$ be the minimum value of $\Omega$ of a tree $t$ over its $n!$ linear arrangements. 
Let $\alpha$ be the minimum of $\Omega_{min}^t$ over all trees of $n$ vertices, i.e.
\begin{equation*}
\alpha = \min_{t \in \tau} \left\{ \Omega_{min}^t \right\},
\end{equation*}
where $\tau$ is the set of all unlabelled trees of $n$ vertices.
We have that 
\begin{equation*}
- \frac{5n-8-5(n \bmod 2)}{n+2- n\bmod 2} \leq \alpha
\end{equation*}
and then $\Omega_{min}^t \geq -5$ for $n \geq 3$. 
\label{lower_bound_of_Omega_min_theorem}
\end{thm}
\begin{proof}
We have that
\begin{equation}
\Omega^t = \frac{D_{rla}-D^t}{D_{rla}-D_{min}^t},
\label{Omega_specific_tree_equation}
\end{equation}
where $D^t$ is the value of $D$ in some linear arrangement of $t$. 
The fact that the denominator of Eq. \ref{Omega_specific_tree_equation} is positive for any linear arrangement ($D_{rla} - D_{min}^t \geq 0$ \cite{Esteban2016a}) and that the numerator can be negative depending on the linear arrangement 
($D_{rla} \leq D_{max}^t$ by the definition of $D_{rla}$ and $D_{max}^t$), yields that
\begin{equation}
\Omega_{min}^t = \frac{D_{rla}-D_{max}^t}{D_{rla}-D_{min}^t} \leq 0,
\label{Omega_min_equation}
\end{equation}
where the numerator is negative and the denominator positive.
Then we have that 
\begin{eqnarray*}
\alpha & \geq & \frac{\min_{t \in \tau} \left\{D_{rla} - D_{max}^t \right\}}{\min_{t \in \tau} \left\{D_{rla} - D_{min}^t \right\} } \\
       & =    & \frac{D_{rla} - \max_{t \in \tau} \left\{D_{max}^t \right\}}{D_{rla} - \max_{t \in \tau} \left\{D_{min}^t \right\} } \\
       & =    & \frac{D_{rla} -  D_{max}^{b-bistar}}{D_{rla} - D_{min}^{star}} \\
       & =    & Z_1
\end{eqnarray*}
thanks to $D_{max}^t \leq D_{max}^{b-bistar}$ (Eq. \ref{extreme_sum_of_edge_lengths_equation})
and $D_{min}^t \leq D_{min}^{star}$ \cite{Esteban2016a}.
Alternatively,
\begin{eqnarray*}
\alpha & \geq & \min_{t \in \tau} \left\{ \frac{D_{rla}}{D_{rla} - D_{min}^t}\right\} - \max_{t \in \tau} \left\{ \frac{D_{max}^t}{D_{rla} - D_{min}^t}\right\} \\
       & \geq & \frac{D_{rla}}{\max_{t \in \tau} \left\{D_{rla} - D_{min}^t\right\}} - \frac{\max_{t \in \tau} \left\{D_{max}^t\right\}}{\min_{t \in \tau} \left\{D_{rla} - D_{min}^t\right\}}  \\ 
       & =    & \frac{D_{rla}}{D_{rla} - \min_{t \in \tau} \left\{ D_{min}^t\right\}} - \frac{\max_{t \in \tau} \left\{D_{max}^t\right\}}{D_{rla} - \max_{t \in \tau} \left\{D_{min}^t\right\}} \\ 
       & =    & \frac{D_{rla}}{D_{rla} - D_{min}^{linear}} -  \frac{D_{max}^{b-bistar}}{D_{rla} - D_{min}^{star}} \\
       & =    & Z_2
\end{eqnarray*}
thanks to $D_{max}^t \leq D_{max}^{b-bistar}$ (Eq. \ref{extreme_sum_of_edge_lengths_equation})
and $D_{min}^{linear} \leq D_{min}^t \leq D_{min}^{star}$ \cite{Esteban2016a}.
Combining both lower bounds of $\Omega^t$,
\begin{eqnarray*}
\alpha & \geq & \max(Z_1, Z_2) \\
       & =    & Z_1      
\end{eqnarray*}
because $D_{min}^{linear} \leq D_{min}^{star}$ implies $Z_2 \leq Z_1$.

On the one hand, Eq. \ref{sum_of_dependency_lengths_random_equation} and Eq. \ref{extreme_sum_of_edge_lengths_equation}, give
\begin{equation*}
D_{rla} - D_{max}^{b-bistar} = \frac{1}{12}(-5n^2 + 18n - 16 + 3 (n \bmod 2)).
\end{equation*}
On the other hand, Eq. \ref{sum_of_dependency_lengths_random_equation} and \cite{Iordanskii1974a,Ferrer2013b,Esteban2016a}
\begin{equation}
D_{min}^{star} = \frac{1}{4}(n^2 - {n \bmod 2}),
\end{equation}
give 
\begin{equation*}
D_{rla} - D_{min}^{star} = \frac{1}{12}[n^2 - 4 + 3(n \bmod 2)].
\end{equation*}
Then
\begin{equation*}
Z_1 = \frac{-5n^2 + 18n - 16 + 3 (n \bmod 2)}{n^2 - 4 + 3(n \bmod 2)},
\end{equation*}
namely  
\begin{equation*}
Z_1 = - \frac{5n-8}{n+2}
\end{equation*}
when $n$ is even and
\begin{equation*}
Z_1 = - \frac{5n-13}{n+1},
\end{equation*}
when $n$ is odd. 
Finally,
\begin{equation*}
Z_1 = - \frac{5n-8-5(n \bmod 2)}{n+2-n \bmod 2}.
\end{equation*}
A simple analysis of the final expression of $Z_1$ indicates that $\Omega^t \geq -5$ for $n \geq 3$. 
\end{proof}

\subsubsection{The minimum $\Omega$ up to some tree size and an ansatz}

Figure \ref{Omega_min_figure} A shows the exact value of $\alpha$
as a function of $n$ up to $n_{max} = \nMaxOmegaMin$. $\alpha$ was calculated exactly using the algorithm explained below. In all cases, $\Omega_{min}^t$ was minimized by bistar trees. 
In addition, a bistar tree defines a class of trees that includes the balanced bistar tree and the star tree \cite{Ferrer2020a}. When $\Omega_{min}^t \leq 0$, the former tree minimizes the numerator of $\Omega_{min}$ and the latter minimizes its denominator (proof of theorem \ref{lower_bound_of_Omega_min_theorem}). These facts allow one to formulate the following ansatz: for any $n \geq 3$, $\Omega_{min}^t$ is minimized by bistar trees. Notice that, for certain values of $n$, the balanced bistar tree does	 not minimize $\Omega_{min}^t$ (Figure \ref{Omega_min_figure} A).

Applying this hypothesis, Figure \ref{Omega_min_figure} A shows the predicted value of $\alpha$ obtained assuming that $\Omega_{min}^t$ is also minimized by bistar trees for $n > n_{max}$. 
Such a prediction is 
\begin{equation} 
\alpha^{bistar} = \min_{t \in \tau^{bistar}} \left\{ \Omega_{min}^t \right\}, 
\label{setup_equation} 
\end{equation}
where $\tau^{bistar}$ are all the unlabelled bistar trees of $n$ vertices. 

Eq. \ref{setup_equation} that can be expressed equivalently as  
\begin{equation*} 
\alpha^{bistar} = \min_{\left\lceil\frac{n}{2}\right\rceil \leq k_1 \leq n - 1} \left\{ \frac{D_{rla} - D_{max}^{bistar(k_1)}}{D_{rla} - D_{min}^{bistar(k_1)}} \right\},
\end{equation*}
where $bistar(k_1)$ is a bistar tree whose most connected vertex has degree $k_1$. 
Then $D_{max}^{bistar(k_1)}$ and $D_{min}^{bistar(k_1)}$ can be calculated using formulae with two parameters, $n$ and $k_1$ \cite{Ferrer2020a}.  
With this approach, it is easy to check numerically that (Fig. \ref{Omega_min_figure} B) 
\begin{equation*} 
\lim_{n \rightarrow \infty} \alpha^{bistar} = -2
\end{equation*}
and 
\begin{equation*}
\alpha^{bistar} \leq \Omega_{min}^{linear}, 
\label{linear_tree_is_worse_equation}
\end{equation*}
where $\Omega_{min}^{linear}$ is the value of $\Omega_{min}^t$ of a linear tree. $\Omega_{min}^t$ is easy to calculate given the formulae for $D_{min}^{linear}$ and $D_{max}^{linear}$ in Table 1 of \cite{Ferrer2020a}.

One the one hand, we have seen that $-5 \leq \Omega_{min}^t$. On the other hand,  
$\alpha$ is minimized at $n = \nMaxOmegaMin$ when $n \leq \nMaxOmegaMin$ (Figure \ref{Omega_min_figure} A). That point corresponds to a bistar tree with $D_{min} = 84$, $D_{max} = 396$ and $D_{rla}=575/3$, hence $\Omega_{min}^t = -\frac{613}{323} \approx - 1.89$ for that tree. 
This allows one to conclude that $-c \leq \Omega_{min}^t$, where $c$ is some constant such that $\frac{613}{323} \leq c \leq 5$. We conjecture that $c = 2$.  

\begin{figure}[tbh]
\centering
\includegraphics[width = \linewidth]{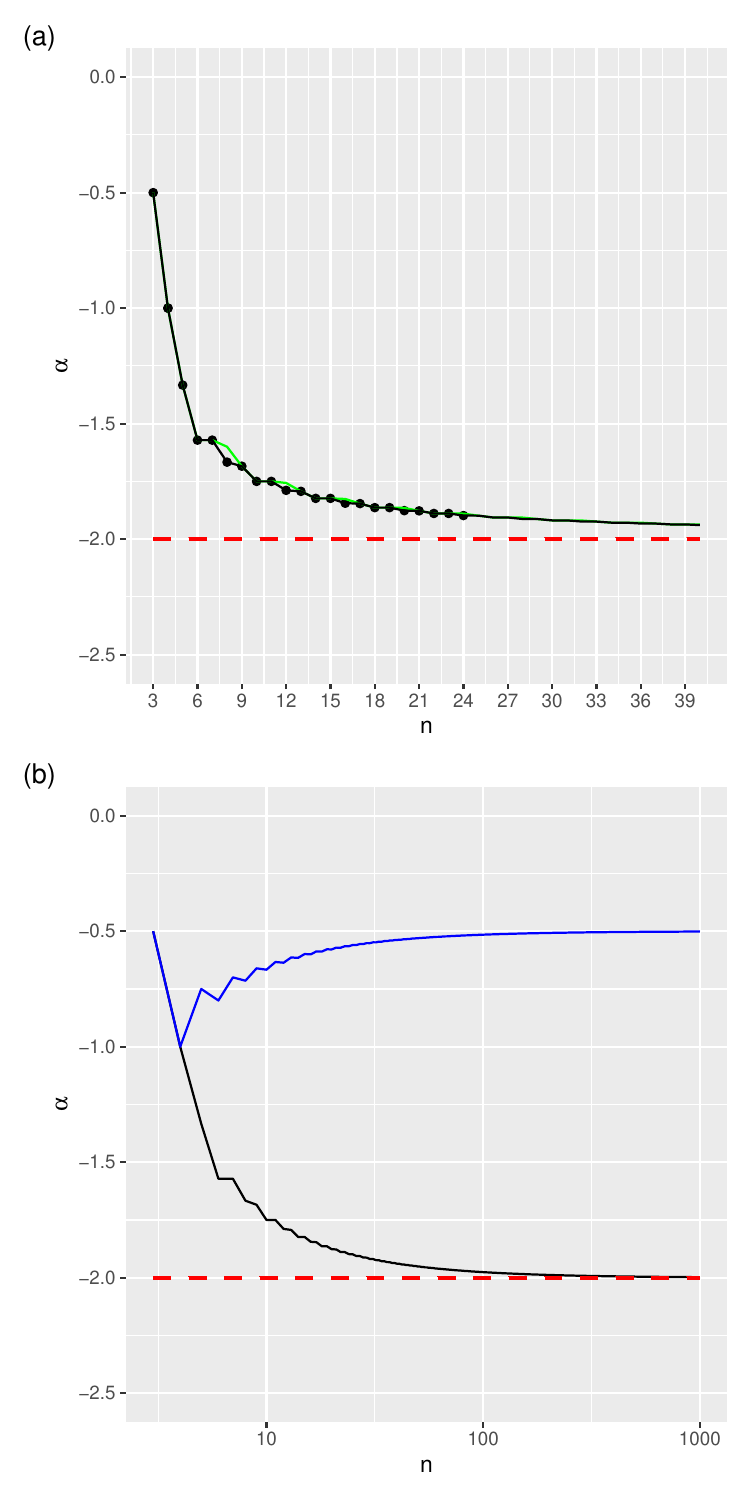}
\caption{\label{Omega_min_figure} The minimum value of $\Omega_{min}^t$ versus $n$, the number of vertices of the tree. (a) $\alpha$, the true minimum (circles), and the ansatz, $\alpha^{bistar}$ based on the smallest $\Omega_{min}^t$ that a bistar tree can achieve (black solid line) matches $\alpha$ perfectly while $\Omega_{min}^{b-bistar}$ (green solid line) is above the ansatz for certain values of $n$. As a guide to the eye, the asymptotic lower bound of $\alpha^{bistar}$ (red dashed line) is also shown. (b) $\alpha^{bistar}$ (black solid line), $\Omega_{min}^{linear}$ (blue solid line) and the asymptotic lower bound (red dashed line). }
\end{figure}

For a given $n$, we calculated $\alpha$ exactly with the following algorithm:
\begin{enumerate}
\item
Initialize $\alpha$ with $\alpha^{bistar}$. The choice of $\alpha^{bistar}$ as the initial value to surpass is justified by the ansatz.   
\item
\label{all_trees_step}
For each distinct unlabelled tree $t$ that it is neither a bistar nor a linear tree
  \begin{enumerate}
  \item
  If $t$ is a caterpillar, calculate $D_{min}^t$ using the fast formulae in Table 1 of \cite{Ferrer2020a}; otherwise calculate $D_{min}^t$ using Shiloach's algorithm \cite{Shiloach1979,Esteban2015a}. 
  \item
  Calculate 
  \begin{equation*}
  \Omega_{lower}^{t} = \frac{D_{rla} - D_{max}^{b-bistar}}{D_{rla}-D_{min}^t},
  \end{equation*}
  which is a lower bound of $\Omega_{min}^{t}$ that is obtained replacing $D_{max}^t$ by $D_{max}^{b-bistar}$ in $\Omega_{min}^t$ (equation \ref{Omega_min_equation}). 
  \item
  If $\alpha \geq \Omega_{lower}^{t}$, namely if it is worth computing $D_{max}^t$ to calculate $\Omega_{min}^t$,     
    \begin{enumerate}
    \item
    \label{D_max_step}
    Calculate $D_{max}^t$ by brute force using a constraint model (described below). 
    \item
    Calculate $\Omega_{min}^t$ (equation \ref{Omega_min_equation}) applying the values of $D_{min}^t$ and $D_{max}^t$ that have been obtained in the preceding steps. 
    \item 
    Set $\alpha$ to $\min(\alpha, \Omega_{min}^t)$.
    \end{enumerate}
  \end{enumerate}
\end{enumerate}
In step \ref{all_trees_step}, the linear tree is excluded by the initial value of $\alpha$ and Eq. \ref{linear_tree_is_worse_equation}.
For step \ref{all_trees_step}, we used an algorithm to generate all unlabelled trees in time proportional to the number of trees \cite{Wright1986a}. 
To calculate $D_{max}$ in step \ref{D_max_step}, we used a constraint model with the following constraints:
\begin{itemize}
\item
A linear arrangement and its reverse give the same $D$. 
\item
Swapping the positions of leaves attached to the same internal vertex does not change $D$. 
\item
Let $\pi^v$ be the position of vertex $v$ in a linear arrangement. Let us consider two leaves, $l_i$ and $l_j$, attached to internal vertices $v_i$ and $v_j$ respectively. Then the positions of these two leaves cannot be swapped if the exchange does not increase the sum of their lengths, namely if  
\begin{equation*}
|\pi^{v_i} - \pi^{l_i}| + |\pi^{v_j} - \pi^{l_j}| \geq |\pi^{v_i} - \pi^{l_j}| + |\pi^{v_j} - \pi^{l_i}|. 
\end{equation*} 
\item
$D_{max} \geq D_{rla}$ by definition of $D_{rla}$ and $D_{max}$ \cite{Ferrer2020a}.
\end{itemize}
Replacing the constraint $D_{max} \geq D_{rla}$ by $D_{max} \geq {n \choose 2}$ \cite{Ferrer2020a} did not help to speed up the calculations. 
The constraints were expressed in the MiniZinc language \cite{Nethercote2007a,Stuckey2014a} and the corresponding model was solved for each relevant tree with the Chuffed solver \cite{chuffed}.

\begin{figure}[tbh]
	\centering
	\includegraphics[width=\linewidth]{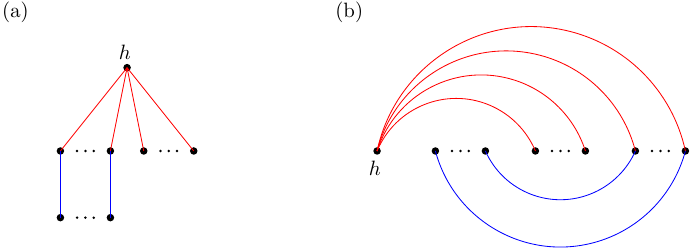}
	\caption{a) A $k$-quasistar tree, that is the fusion of a star tree (edges in red) and a 1-regular tree (edges in blue). b) A maximum linear arrangement of a $k$-quasistar tree.}
	\label{fig:k-quasistar:max_arrangement_Qkl}
\end{figure}

We found that the execution of the solver was excessively long when computing $D_{max}^{t}$ when $t$ was a $k$-quasistar ($k$-quasi) tree and thus we decided to cut computation costs by obtaining a fast formula for $D_{max}^{k-quasi}$. A $k$-quasistar tree of $n$ vertices is a tree that results from merging two graphs, a star tree of $l+k$ edges and a 1-regular graph of $k$ edges (a 1-regular graph is graph such that every vertex has degree 1). The $k$-quasistar is obtained by fusion of each of the $k$ leaves of the star tree with one of the vertices of a distinct edge of the 1-regular graph (Fig. \ref{fig:k-quasistar:max_arrangement_Qkl}(a)). In such a $k$-quasistar, $n= 2k + l + 1$ with $k,l \ge 0$. Notice that such a tree is a generalization of a star tree ($k=0$) and a quasistar tree ($k=1$) \cite{Ferrer2014f}.
The following theorem provides the fast formula that is necessary.
\begin{thm}
For any $k$-quasistar tree of $n$ vertices, 
\begin{equation*}
D_{max}^{k-quasi} = \frac{1}{2}(n - 1 - k)(3k + n).
\end{equation*}
\end{thm}


\begin{proof}
A maximum linear arrangement of star tree is one where its hub is placed at one end of the arrangement and the other vertices are placed at the other end \cite{Ferrer2020a} as in Fig. \ref{fig:k-quasistar:max_arrangement_Qkl}(b). Such arrangement for a star tree of $m$ edges on $z$ positions yields \cite{Ferrer2020a}
\begin{equation}
\label{eq:max_star}
D_{max}^{star}(m, z) = \frac{m(2z-m-1)}{2}.
\end{equation}
To obtain a maximum arrangement for 1-regular graph of $m$ edges over $z$ positions, we take advantage of the fact that its edges are independent (no two edges share a vertex). At any time, one picks the farthest pair of positions available for every edge. First, positions $\{1,z\}$ produce the longest edge, positions $\{2,z-1\}$ produce the second longest edge and so on as in \ref{fig:k-quasistar:max_arrangement_Qkl}(b). Such arrangement yields
\begin{equation}
\label{eq:max_one_reg}
D_{max}^{1-reg}(m, z) = \sum_{i=0}^{m-1} (z-1-2i) = m(z-m).
\end{equation}

Our claim is that the linear arrangement in Fig. \ref{fig:k-quasistar:max_arrangement_Qkl}(b), henceforth $*$, is maximum. More precisely, we claim that $D_{max}^{k-quasi}$ matches
\begin{equation*}
D_{*}^{k-quasi} = D_{max}^{star}(l+k, n) + D_{max}^{1-reg}(k, n - 1)
\end{equation*}
with $n=2k+l+1$. 
The application of Eqs. \ref{eq:max_star} and \ref{eq:max_one_reg} yields
\begin{equation}
D_{*}^{k-quasi}	= \frac{1}{2}(n - 1 - k)(3k + n). \label{maximum_linear_arrangement_equation} 
\end{equation}
We prove, by contradiction, $D_{max}^{k-quasi} = D_{*}^{k-quasi}$.
Suppose that there is some linear arrangement $**$ yielding a sum of distance $D_{**}^{k-quasi}$ such that $D_{**}^{k-quasi} > D_{*}^{k-quasi}$. We use $X$ to indicate an undetermined argument of a function. Decomposing $D_{**}^{k-quasi}$ as
\begin{equation}
D_{**}^{k-quasi} = D_{**}^{star}(l+k, X) + D_{**}^{1-reg}(k, X),
\end{equation} 
we can then write
\begin{widetext} 
\begin{equation*}
D_{**}^{star}(l+k, X) + D_{**}^{1-reg}(k, X)
	> \\
D_{max}^{star}(l+k, n) + D_{max}^{1-reg}(k, n - 1).
\end{equation*}
\end{widetext}
Rearranging the terms of the previous inequality and the fact that $D_{**}^{star}(l+k, X) \le D_{max}^{star}(l+k, n)$, because $D_{max}^{star}(l+k, n)$ is already maximum, we obtain 
\begin{widetext} 
\begin{equation*}
0
	\geq
D_{**}^{star}(l+k, X) - D_{max}^{star}(l+k, n)
	>
D_{max}^{1-reg}(k, n - 1) - D_{**}^{1-reg}(k, X).
\end{equation*}
\end{widetext} 
Therefore, $D_{**}^{1-reg}(k, X) > D_{max}^{1-reg}(k, n - 1)$. This can only happen if the 1-regular graph is arranged over $n$ positions in $**$, namely $D_{**}^{1-reg}(k, X) = D_{**}^{1-reg}(k, n)$ as in Fig. \ref{fig:max_arrangements_2}(b). This leads to the conclusion that both ends of the linear arrangement $**$ must be occupied by vertices of the 1-regular graph since $D_{**}^{1-reg}(k, n - 1) > D_{max}^{1-reg}(k, n - 1)$ is impossible. This restricts the arrangement of the star tree in $**$ by not allowing its hub to be placed at one of the ends of the linear arrangement, which implies $D_{**}^{star}(l+k, X) \leq D_{**}^{star}(l+k, n - 1)$. To see it, notice that, in order to maximize $D_{**}^{star}(l+k, X)$ when the hub cannot be be placed at the ends, the hub should be placed second (resp. penultimate) and none of its leaves should occupy the first (resp. last) position as in Fig. \ref{fig:max_arrangements_2}(a) or its symmetric.
As a result of that, $D_{**}^{k-quasi} \leq U$, where 
\begin{equation*}
U = D_{max}^{star}(l+k, n - 1) + D_{max}^{1-reg}(k, n).
\end{equation*}
Eqs. \ref{eq:max_star} and \ref{eq:max_one_reg} give (recall that $n=2k+l+1$)
\begin{align}
D_{max}^{star}(l+k, n - 1) &= \frac{1}{2}(n - k - 1)(k + n - 2), \nonumber \\
D_{max}^{1-reg}(k, n) &= k(n - k), \nonumber \\
U &= \frac{1}{2}(n - k -1)(3k+n-2). \label{upper_bound_equation}
\end{align}
Importantly, $U \le D_{*}^{k-quasi}$ since $D_{*}^{k-quasi} - U = l \ge 0$ thanks to Eqs. \ref{maximum_linear_arrangement_equation} and \ref{upper_bound_equation}, and thus $*$ is a maximum linear arrangement.
\end{proof}

\begin{figure}[tbh]
	\centering
	\includegraphics[width=\linewidth]{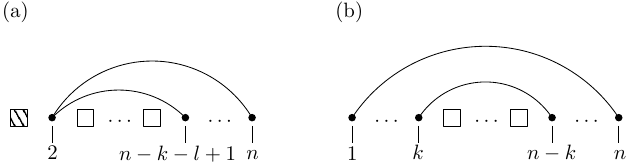}
	\caption{Two maximum linear arrangements. Empty squares indicate empty positions, and dashed squares indicate forbidden positions in the arrangement. a) A maximum linear arrangement of a star tree of $k+l$ edges over $n-1$ positions. b) A maximum linear arrangement of a $1$-regular tree of $k$ edges over $n$ positions.}
	\label{fig:max_arrangements_2}
\end{figure}

\subsection{Constancy under minimum linear arrangement of the optimality scores}

Under minimum linear arrangement, $\Omega^t = \Gamma^t = 1$ and $\Delta^t = 0$ (see above for $\Omega^t$ and \cite{Ferrer2020a} for $\Gamma^t$ and $\Delta^t$).
$D_{z,min}^t$, the value of $D_z^t$ in a minimum linear arrangement satisfies 
\begin{equation*}
D_{z,min}^{linear} \leq D_{z,min} \leq D_{z,min}^{star}
\end{equation*}
and $D_{z,min}^{linear}$ diverges as $n$ tends to infinity \cite{Ferrer2020a}. Formulae for $D_{z,min}^{linear}$ and $D_{z,min}^{star}$ are found in \cite{Ferrer2020a}.

Let us recall the definition of the $NDD$ score \cite{Lei2018a}, i.e. 
\begin{equation*}
NDD = \left|\log\frac{\meanD}{\sqrt{\pi^r n}}\right|.
\end{equation*}
Recall $\meanD = D^t/(n-1)$. Under a minimum linear arrangement, the $NDD^t$ score \cite{Lei2018a} becomes 
\begin{equation*}
NDD_{min}^t = \left| \log \frac{D_{min}^t}{(n-1)\sqrt{\pi_{min}^r n}} \right|
\end{equation*}
and then $NDD^t$ is not constant under linear arrangement due to $\pi_{min}^r$. The following counterexamples clarify it further. 
Consider a linear tree. Suppose that the root is the $i$-th vertex of the linear tree following a depth-first traversal from one of the leaves. Then $\pi_{min}^r = i$ or $\pi_{min}^r = n - i$ in a minimum linear arrangement. If the root is a leaf, then $\pi_{min}^r = 1$ or $\pi_{min}^r = n$. Therefore, $NDD_{min}^{linear}$ is not constant. 
Consider a star tree. Suppose that the root is the hub. Then $\pi_{min}^r = 1$ or $\pi_{min}^r = n$ in a minimum linear arrangement. Suppose that the root is a leaf. Then $2 \leq \pi_{min}^r \leq n$ if the hub has been placed first or $1 \leq \pi_{min}^r \leq n - 1$ if the hub has been placed last. Therefore, $NDD_{min}^{star}$ is not constant either.  

\subsection{Stability under the null hypothesis (random linear arrangements)}

We define $\E_{rla}[\Omega^t]$ and $\E_{rla}[D_z^t]$ as the expected value of $\Omega^t$ and $D_z^t$ in a uniformly random linear arrangement of a given tree. $D_z$ is stable under the null hypothesis because $\E_{rla}[D_z] = 0$ \cite{Ferrer2020a}. We will show that this is also the case for $\Omega^t$.
Recall that $D_{rla} = \E_{rla}[D]$. Given a tree, $D_{min}$ and $D_{rla}$ are constant. Hence  
\begin{eqnarray*}
\E_{rla}[\Omega^t] & = & \E_{rla}\left[\frac{D_{rla} - D^t}{D_{rla} - D_{min}^t}\right] \\
                   & = & \frac{D_{rla} - \E_{rla}[D^t]}{D_{rla} - D_{min}^t} \\
                   & = & 0.
\end{eqnarray*} 
In contrast, neither $\Gamma$ nor $\Delta$ are stable under the null hypothesis. 
Recall Eq. \ref{sum_of_dependency_lengths_random_equation} and that \cite{Iordanskii1974a,Esteban2016a}
\begin{equation}
D_{min}^{linear} = n - 1 \leq D_{min}^t \leq D_{min}^{star} = \left\lfloor \frac{n^2}{4} \right\rfloor \leq D_{rla}.
\label{range_of_variation_of_D_min_tree_equation}
\end{equation}
The expected value of $\Gamma^t$ in a uniformly random linear arrangement of a given tree is  
\begin{equation}
\E_{rla}[\Gamma^t] = \frac{D_{rla}}{D_{min}}.
\label{expected_Gamma_equation}
\end{equation}
The application of Eq. \ref{sum_of_dependency_lengths_random_equation} and Eq. \ref{range_of_variation_of_D_min_tree_equation} to Eq. \ref{expected_Gamma_equation} yields 
\begin{widetext}
\begin{equation*}
\E_{rla}^{star} [\Gamma^t] = \frac{4}{3}\frac{n^2-1}{n^2 - {n \bmod 2}} \leq \E_{rla}[\Gamma^t] \leq \E_{rla}^{linear}[\Gamma^t] = \frac{n+1}{3}. 
\end{equation*}
\end{widetext}
for any tree of $n \geq 2$ vertices.

The application of Eq. \ref{sum_of_dependency_lengths_random_equation} and Eq. \ref{range_of_variation_of_D_min_tree_equation} to $\Delta$ yields 
\begin{equation*}
\E_{rla}^{star}[\Delta^t] \leq \E_{rla}[\Delta^t] \leq \E_{rla}^{linear}[\Delta^t] 
\end{equation*}
with
\begin{eqnarray*}
\E_{rla}^{linear}[\Delta^t] = \frac{1}{3}(n-1)(n-2) \\ 
\E_{rla}^{star}[\Delta^t] = \frac{1}{12}(n^2 - 4 + 3 (n \bmod 2)). 
\end{eqnarray*}

In the case of $NDD$, $\E_{rla}[NDD^t]$, crucial to determine if the metric is stable under the null hypothesis, is difficult to calculate because of the complexity of the formula of $NDD$: the absolute value, the logarithmic transformation, the square root in the denominator of the inner quotient, and the fact that the numerator and the denominator may not be independent. A possible approximation is  
\begin{equation*}
\E_{rla}[NDD^t] \approx \left| \log B \right|.
\end{equation*}
with 
\begin{equation*}
B = \frac{\E_{rla}[\meanD]}{\sqrt{\E_{rla}[\pi^r] n}}.
\end{equation*}
Knowing that $\E_{rla}[\pi^r] = n/2$ and 
\begin{eqnarray*}
\E_{rla}[\meanD] & = & \frac{\E_{rla}[D_{rla}]}{n-1} \\ 
                 & = & \frac{n+1}{3}
\end{eqnarray*}
thanks to Eq. \ref{sum_of_dependency_lengths_random_equation}, one obtains
\begin{equation*}
B = \frac{\sqrt{2}}{3} \left(1 + \frac{1}{n}\right).
\end{equation*} 
Since $B \leq 1$ for $n \geq 1$, $|\log B| = - \log B$ and finally   
\begin{equation}
\E_{rla}[NDD^t] \approx - \log \left[ \frac{\sqrt{2}}{3} \left(1 + \frac{1}{n}\right) \right].
\label{approximation_of_expectation_equation}
\end{equation}
Fig. \ref{NDD_figure} compares this approximation against the exact $\E_{rla}[NDD^t]$ for linear and star trees. Linear and star trees are chosen because they minimize and maximize the variance of $D$ in uniformly random linear arrangements. Eq. \ref{approximation_of_expectation_equation} approximates the true expectation better for large $n$. Critically, Fig. \ref{NDD_figure} indicates that $NDD$ is not stable under the null hypothesis: the expected $NDD$ tends to shrink as $n$ grows and, for a given $n$, its value depends on the kind of tree.

\begin{figure}
\centering
\includegraphics[width = \linewidth]{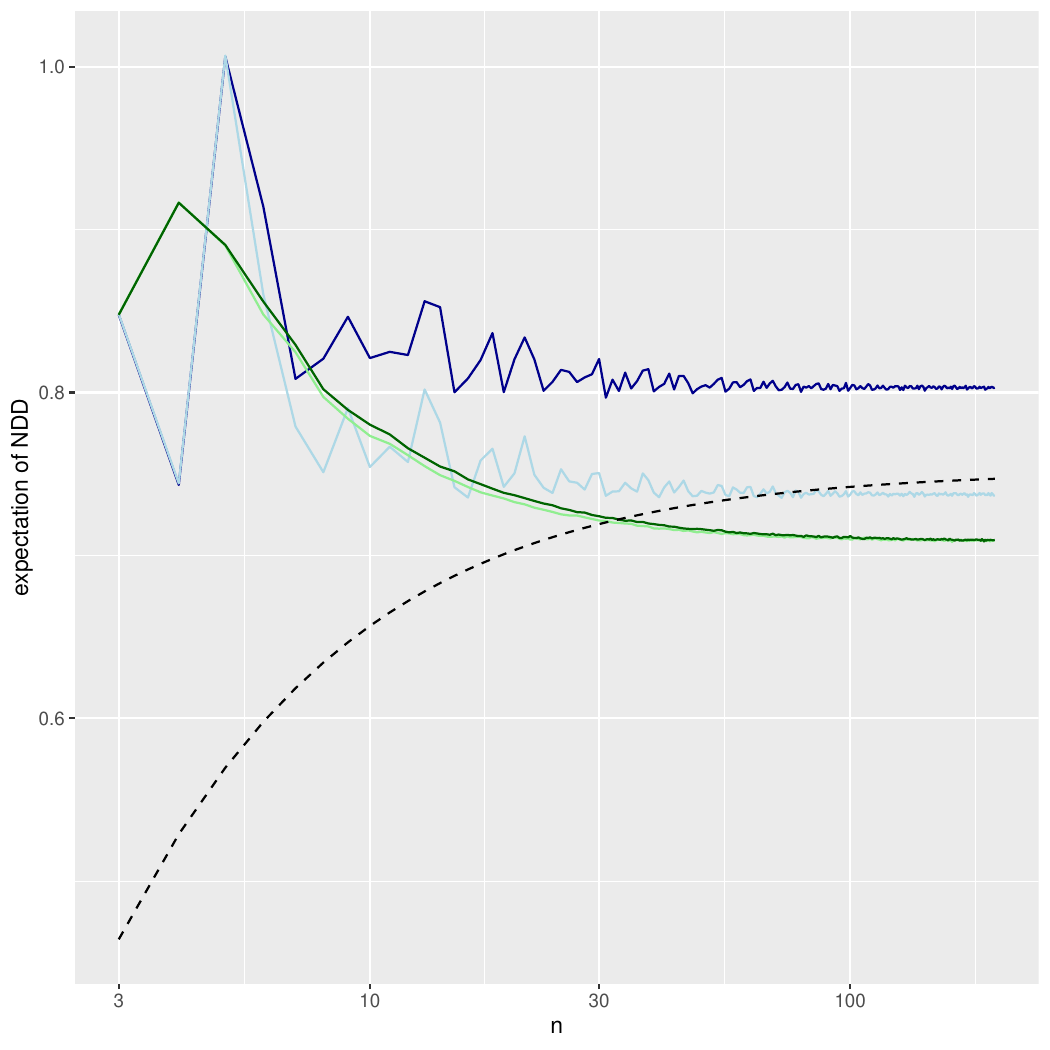} 
\caption{\label{NDD_figure} $\E_{rla}[NDD^t]$ for star trees (blue) and linear trees (green) as a function of $n$ estimated numerically. For reference, the curve for the approximated $\E_{rla}[NDD^t]$ is also shown (dashed line). For star trees, two curves are shown: tree rooted on a leaf (light blue) and a tree rooted on an inner vertex attached to a leaf (dark blue). For linear trees, two curves are shown: tree rooted on a leaf (light green) and a tree rooted on an inner vertex attached to a leaf (dark green). Numerical estimates are obtained generating $10^6$ random linear arrangements for each $n$. }
\end{figure}

\subsection{Boundedness under maximimum linear arrangement}

We have seen above that $\Omega_{min}^t$ is bounded below by a constant when $n$ tends to infinity and thus $\Omega$ is bounded under maximum linear arrangement.  
In contrast, $D_{max}^t$ diverges when $n$ tends to infinity (Eq. \ref{extreme_sum_of_edge_lengths_equation}) and then $D^t$ is not bounded under maximum linear arrangement. The same happens to all other scores as explained in \cite{Ferrer2020a} except for $NDD$. 
Notice that the value of $NDD$ under a maximum linear arrangement of a tree $t$ is 
\begin{equation*}
NDD_{max}^t = \left| \log \frac{D_{max}^t}{(n-1)\sqrt{\pi_{min}^r n}} \right|.
\end{equation*}
It is easy to see that $NDD_{max}^t$ is maximized by a balanced bistar tree rooted at one of the hubs that is placed in first position of the linear arrangement (and then $\pi_{min} = 1$), namely, 
Therefore, 
\begin{equation*}
NDD_{max}^t \leq \left| \log \frac{D_{max}^{b-bistar}}{(n-1)\sqrt{n}} \right|.
\end{equation*}
Hence, $NDD^t$ is not bounded under maximum linear arrangement either. 

\subsection{Invariance under linear transformation}

Both $\Omega^t$ and $D_z^t$ are invariant under linear transformation.  
Let $D$ be the sum of edge lengths of some tree $t$ with $n$ vertices. $D_{rla}$ is a short hand for $\E_{rla}[D]$, the average value of $D$ over all the linear arrangements of $t$. 
Similarly, $D_{min}$ is a shorthand for the minimum value of $D$ over all the linear arrangements (la) of $t$, i.e.
\begin{equation*}
D_{min} = \min_{la}[D]. 
\end{equation*}
Let $aD+b$ be a linear function of $D$, where $a$ and $b$ are constants in the sense that their value does not depend on the linear arrangement for $t$. 
Interpreting $\Omega$ as a function of the random variable $D$ in Eq. \ref{raw_optimality_metric_tree_equation}, one obtains that
\begin{eqnarray*}
\Omega[aD+b] & = & \frac{\E_{rla}[aD + b] - aD - b}{\E_{rla}[aD + b] - \min_{la}[aD + b]} \\
             & = & \frac{a\E_{rla}[D] + b - aD - b}{a\E_{rla}[D] + b - \min_{la}[aD + b]} \\
             & = & \frac{a(D_{rla} - D)}{a(D_{rla} - D_{min})} \\
             & = & \Omega[D].
\end{eqnarray*} 
The same procedure for $D_z$ yields 
\begin{eqnarray*}
D_z[aD+b] & = & \frac{aD + b - \E_{rla}[aD + b]}{\V_{rla}[aD + b]^{1/2}} \\
          & = & \frac{aD + b - a\E_{rla}[D] + b}{a\V_{rla}[D]^{1/2}} \\
          & = & \frac{a(D - \E_{rla}[D])}{a\V_{rla}[D]^{1/2}} \\
          & = & D_z[D]
\end{eqnarray*}  
as expected for a z-score. 
In contrast, neither $\Delta$ nor $\Gamma$ exhibit invariance under linear transformation. 
As for $\Delta$, 
\begin{eqnarray}
\Delta[aD+b] & = & aD + b - \min_{la}[aD + b] \nonumber \\
             & = & a(D - \min_{la}[D]) \nonumber \\
             & = & a\Delta[D] \label{invariance_of_Delta_equation}
\end{eqnarray}
and $a = 1$ is required.   
However, it is easy to see that $\Delta(D+b) = \Delta(D)$, namely invariance under displacement, imposing $a = 1$ on Eq. \ref{invariance_of_Delta_equation}.
$\Delta$ lacks invariance under proportionate scaling, i.e. $\Delta(aD) \neq \Delta(D)$ (unless $a = 1$).
As for $\Gamma$,
\begin{eqnarray}
\Gamma[aD+b] & = & \frac{aD + b}{\min_{la}[aD + b]} \nonumber \\
             & = & \frac{aD + b}{a\min_{la}[D]+b} \nonumber \\
             & = & \frac{aD + b}{aD_{min}+b}. \label{invariance_of_Gamma_equation}
\end{eqnarray}  
The condition $\Gamma[aD+b]=\Gamma[D]$ is equivalent to $D_{min} = D$, indicating that $\Gamma$ is not invariant under linear transformation.
However, it is easy to see that $\Gamma[aD] = \Gamma[D]$, namely invariance under proportionate scaling, imposing $b = 0$ on Eq. \ref{invariance_of_Gamma_equation}.
It is also easy to see that $\Gamma$ lacks invariance under displacement, i.e. $\Gamma[D+b] \neq \Gamma[D]$ (unless $b \neq 0$), imposing $a = 0$ on Eq. \ref{invariance_of_Gamma_equation}.

Regarding $NDD$, the fact that $\meanD$ appears inside a logarithm and that $\pi^r$ is square rooted, trivially implies that it lacks invariance under both displacement and proportionate scaling.

%% file: SI_materials.tex
To investigate the optimality of dependency distances across a wide range of languages and language families we use several different collections of syntactic dependency treebanks that are available in different annotation criteria. A syntactic dependency treebank is an ensemble of sentences, each of them annotated with its corresponding syntactic dependency tree. Trees can be annotated using different annotation criteria, depending on the specific grammatical framework that is being used to represent syntax. 

In particular, our collections of treebanks cover the following annotation criteria:
\begin{itemize}
\item Universal Stanford dependencies \cite{UniversalStanford}. This grammatical framework is a revised version of the Stanford Dependencies for English \cite{Stanford2008}, aiming to extend its applicability to languages beyond English. As mentioned in \cite{Stanford2008}, Stanford dependencies are inspired on lexical-functional grammar \cite{Bresnan00}. They are designed to provide a simple description of the grammatical relations in a sentence that is easy to use in natural language processing (NLP) applications and can be understood even by people without linguistic expertise.
\item Prague Dependencies \cite{PDT20} provide a different view of syntax, based on the functional generative description \cite{Sgall69} of the Prague linguistic tradition \cite{Hajicova95}. The focus is more on providing an accurate and consistent linguistic theory than on simplicity for NLP applications. The differences with respect to the Stanford annotation are substantial, for example, in regard to the annotation of conjunctions and adpositions \cite{HowFarStanfordPrague}. In very rough terms, Stanford Dependencies favor content words as heads, while Prague dependencies often annotate function words (like adpositions) as heads.
\item Universal Dependencies (UD) \cite{ud26_APS}. This is an evolution of the Universal Stanford dependencies with the goal of further extending their cross-linguistic applicability, as the basis of an international cooperative project to create treebanks of the world's languages. Thus, it shares the same principles as Universal Stanford Dependencies, but it has been refined across versions to better accommodate the specificities of diverse languages inside the common framework.
\item Surface-Syntactic Universal Dependencies (SUD). The priority of UD is to maximize parallelism between languages, and for this purpose, it defines annotation criteria that enforce simple tree structures and follow the principle of favoring content words as heads, inherited from Stanford Dependencies (as the syntax of content words varies less than that of function words between languages). This is reasonable for applications like cross-lingual syntactic parsing, but arguably not the best choice for typological research on syntax, as it minimizes differences between languages. For this reason, SUD \cite{sud} is proposed as an alternative where dependency links are defined based on purely syntactic criteria, closer to the Meaning-Text Theory \cite{melcuk88} or the Praguian linguistic tradition \cite{Hajicova95}. However, at the same time, SUD is designed to provide easy convertibility from UD, which earlier frameworks like the Prague dependencies lack.
\end{itemize}

The collections used are the following:
\begin{itemize}
\item Universal Dependencies (UD), version 2.6 \cite{ud26_APS}. This is the largest available collection of syntactic dependency treebanks, both in terms of number of languages and typological diversity. It features \TreebanksUD treebanks from \LanguagesUD different languages, belonging to \Families broad families (see Table \ref{typological_diversity_table}). We use versions of these treebanks with two annotations: UD and SUD. The UD version can be downloaded freely \cite{UD_data}. The SUD version can be obtained from the UD version by applying a conversion script \cite{UD_to_SUD_converter} or downloaded directly \cite{SUD_data}. Both of them can be used under free licenses. 
\item Parallel Universal Dependencies (PUD), version 2.6 \cite{conll2017st_APS}. This is a subset of the UD treebanks that contains parallel annotations for \LanguagesPUD languages. Parallel means that all the treebanks contain the exact same sentences, in the same order, translated into each of the languages by professional translators. This is useful to control for variability due to the genre or content of the text, which can vary across different non-parallel treebanks. Since it is a subset of the UD treebanks, the annotation criteria are the same (i.e., there are versions annotated with UD and SUD), as are the download links.
\item HamleDT, version 2.0 \cite{HamledTStanford}, is a collection of treebanks of \LanguagesHamleDT languages. All of these languages are present in the UD collection except one (Bengali) and thus this collection does not meaningfully extend the linguistic diversity of our study. However, its interest is in the annotation schemes. Each of the \TreebanksHamleDT treebanks in HamleDT comes with two versions, each annotated with a different set of annotation criteria: Universal Stanford dependencies and Prague Dependencies. As mentioned above, the former are a direct ancestor of UD, while the latter provide a different view of syntax, closer in focus to SUD (but substantially different, as SUD had unique design constraints with respect to the original Prague dependencies, like being easily convertible from UD). Thus, using the HamleDT collection is a way to further extend the diversity of syntactic theories represented in our study, making our results more robust with respect to annotation criteria. The HamleDT 2.0 treebanks can be obtained online \cite{HamleDT_data}. While not all of them are distributed under free licenses, for the purposes of reproducing our analysis it suffices to use a stripped version that does not include the words from the sentences (for copyright reasons), but still includes the bare tree structures. This version is distributed freely \cite{stripped_HamleDT}. 
\end{itemize}
Table \ref{typological_diversity_table} summarizes the languages and language families included in each of the collections of treebanks.

Since our trees and definition of dependency distance are based on the concept of word, a concern that may arise is the possible influence of wordness criteria in our measurements -- i.e., ideally, a fair analysis comparing the optimality of dependency distances in different doculects would require all corpora to be annotated with consistent wordness criteria. While it has been argued that there is no such thing as a cross-lingual notion of word \cite{Haspelmath2011}, making this ideal impossible to reach in the absolute, the corpora in the UD collection (including the versions annotated with the UD and SUD criteria, and by extension, the PUD subsets) have been designed to be as close as possible to it.

In particular, the UD tokenization and word segmentation guidelines (which are also followed by SUD~\cite{sud_guidelines}) specify that the basic units of annotation are syntactic words (as opposed to morphological or orthographic words): \url{https://universaldependencies.org/u/overview/tokenization.html}. This implies, among other considerations, that clitics are split away from words (as in Spanish ``m\'iralo'' $\rightarrow$ ``mira lo'') and contractions are broken (as in English ``it's'' $\rightarrow$ ``it is''); while expressions like ``400 000'' are counted as a single word regardless of their spacing. For multiword expressions and compounds, the UD guidelines also provide criteria to annotate them in a consistent way between languages. In languages where there is no clear-cut concept of a discrete word or there is substantial variance in its definition, the UD criterion of maximizing parallelism between languages has been followed (arguably, even at the expense of a certain degree of Euro-centrism) providing as much consistency between languages as the treebank designers could. As an example, we can focus on Japanese, where there are various definitions of word-like entities with different granularities: \cite{Asahara2018} reviews several, including the NINJAL 5-level hierarchy that includes (in order of increasing granularity) Minimum Unit Words, Short Unit Words, Middle Unit Words, Long Unit Words and bunsetsu. While bunsetsu have been the most widely used such unit in Japanese NLP~\cite{Murawaki2019}, UD made the decision of adopting Short Unit Words, as bunsetsu have a larger granularity than the syntactic words used in other UD treebanks. While this decision is not free from controversy (see e.g. \cite{Pringle2016}), \cite{Murawaki2019} argues that in spite of limitations, Short Unit Words are an ``understandable'' solution from the point of view of UD's goals of maximizing cross-lingual consistency and facilitating tasks like Japanese to English machine translation. Similar decisions have been made in other languages, so the wordness criteria in the UD treebanks are as consistent as treebank designers have found possible considering the differences between distant languages, and thus provide a solid foundation to compare the optimality of dependency distances between languages.

{In the case of the HamleDT collection, the wordness criteria that have been followed are significantly less homogeneous, as its corpora come from diverse sources and only syntactic annotation (but not tokenization) was harmonized. As acknowledged in \cite{HamleDT}, this means that there are differences between treebanks, for example in the treatment of multiword expressions.

We now describe our treatment of the treebanks. No dummy root nodes were added to the trees, as is often done for technical reasons in the NLP literature \cite{BalNiv13}, since they do not correspond to actual words, do not take part in the linear arrangement of a sentence as they do not have a real position, and the edges originating from them would be a potential source of spurious anti-DDm effects. All the treebanks from the same language were merged to form a single treebank. 
Then, all of the treebanks underwent the following postprocessing for analysis: following common practice in statistical research on dependency trees \cite{Gomez2016a}, we removed punctuation, as well as tree nodes not corresponding to actual words, like the empty nodes in several UD corpora and null elements in the Bengali, Hindi and Telugu HamleDT treebanks. To ensure that the resulting structure is still a tree after the removal of these nodes, we applied the following transformation: nodes whose heads were deleted were reattached as dependents of their nearest non-deleted ancestor in the original tree. 

In the case of the PUD collection, we also performed additional analyses where we remove function words, thus considering only content words. In these analyses, function word removal was applied right after punctuation removal. We identified function words as those with closed-class universal POS tags, and applied the same transformation described for punctuation above to preserve treeness when possible. However, removal of function words has an additional technical issue: contrary to punctuation, function words often play the role of syntactic root, so their removal results in disconnected tree fragments. This can happen under any annotation criterion (in fact, in extreme cases, some sentences in the corpora are made up exclusively of function words, so their removal results in a sentence of length zero) but it is especially prevalent when SUD annotation criteria are used, as they prefer function words as heads. A very common case is that of sentences headed by an auxiliary verb, which depends on the main verb under UD criteria but is the syntactic root under SUD criteria. Since $\Omega$ is only defined for trees, and there is no way to obtain a tree from a set of disconnected fragments without creating artificial dependencies to link them introducing arbitrary decisions in the data, we disregard these sentences in the analyses without function words.

Finally, in our analyses we also disregard trees with less than three nodes, as they have trivial statistical properties (a tree with one node has no dependencies, a tree with two nodes always has a single dependency with distance 1) and $\Omega$ is undefined for these trees.

\begin{table*}
\centering
\caption[Typological diversity of the treebank collections used in this study.]{\label{typological_diversity_table}The languages in every collection grouped by family. The counts attached to the collection names indicate the number of different families and the number of different languages. The counts attached to family names indicate the number of different languages.   }
\begin{ruledtabular}
\begin{tabular}{llp{4in}}
Collection & Family & Languages \\
\hline
\input{figures/typological_diversity_table}
\end{tabular}
\end{ruledtabular}
\end{table*}

Datasets are formed by combining a collection of treebanks and annotation criteria. The PUD and the PSUD datasets are formed by combining the PUD collection with Universal Dependencies and Surface-Syntactic Universal Dependencies, respectively. Originally, these datasets have $\SentencePUD$ sentences per language. As the analyses are restricted to $n\geq 3$ (see below), some sentences such that $n < 3$ are lost after the postprocessing described above (Table \ref{short_sentences_table}). Since we find that $\Omega$ tends to increase as $n$ increases, $\optimalityL$ may be more biased towards higher values in those languages that have lost more sentences. 
This may be problematic because the differences in $\optimalityL$ between languages are often small (e.g. \ref{ranking_PSUD_figure}). To avoid this problem, the PUD and the PSUD are reparallelized, namely, all sentences that are not found in all the languages after the postprocessing are removed. As a result, all languages have 
$\SentencePUDAfterReparallelization$ sentences and the loss of sentences is tiny (only $0.5\%$ of sentences are lost).

However, this problem is more significant in the analysis without function words for the datasets obtained from the PUD collection. Function word removal not only reduces sentence length and thus produces more sentences with $n < 3$, but it can also result in analyses that are not trees and need to be removed, as explained above. This is compounded by reparallelization (i.e., a sentence that does not preserve treeness after function word removal in one language has to be removed from all languages) leaving 835 sentences per language in the version of the PUD dataset without function words, but only 13 in the equivalent PSUD dataset.

%% file: figures/typological_diversity_table.tex
UD (19, 92) & Afro-Asiatic (7) & Akkadian, Amharic, Arabic, Assyrian, Coptic, Hebrew, Maltese \\
  & Altaic (3) & Kazakh, Turkish, Uyghur \\
  & Austro-Asiatic (1) & Vietnamese \\
  & Austronesian (2) & Indonesian, Tagalog \\
  & Basque (1) & Basque \\
  & Dravidian (2) & Tamil, Telugu \\
  & Indo-European (51) & Afrikaans, Albanian, Ancient Greek, Armenian, Belarusian, Bhojpuri, Breton, Bulgarian, Catalan, Croatian, Czech, Danish, Dutch, English, Faroese, French, Galician, German, Gothic, Greek, Hindi, Hindi-English, Icelandic, Irish, Italian, Kurmanji, Latin, Latvian, Lithuanian, Marathi, Norwegian, Old Church Slavonic, Old French, Old Russian, Persian, Polish, Portuguese, Romanian, Russian, Sanskrit, Scottish Gaelic, Serbian, Slovak, Slovenian, Spanish, Swedish, Swiss German, Ukrainian, Upper Sorbian, Urdu, Welsh \\
  & Japanese (1) & Japanese \\
  & Korean (1) & Korean \\
  & Mande (1) & Bambara \\
  & Mongolic (1) & Buryat \\
  & Niger-Congo (2) & Wolof, Yoruba \\
  & Other (1) & Naija \\
  & Pama-Nyungan (1) & Warlpiri \\
  & Sign Language (1) & Swedish Sign Language \\
  & Sino-Tibetan (3) & Cantonese, Chinese, Classical Chinese \\
  & Tai-Kadai (1) & Thai \\
  & Tupian (1) & Mbya Guarani \\
  & Uralic (11) & Erzya, Estonian, Finnish, Hungarian, Karelian, Komi-Permyak, Komi-Zyrian, Livvi, Moksha, North Sami, Skolt Sami \\
PUD (9, 20) & Afro-Asiatic (1) & Arabic \\
  & Altaic (1) & Turkish \\
  & Austronesian (1) & Indonesian \\
  & Indo-European (12) & Czech, English, French, German, Hindi, Icelandic, Italian, Polish, Portuguese, Russian, Spanish, Swedish \\
  & Japanese (1) & Japanese \\
  & Korean (1) & Korean \\
  & Sino-Tibetan (1) & Chinese \\
  & Tai-Kadai (1) & Thai \\
  & Uralic (1) & Finnish \\
HamleDT (7, 30) & Afro-Asiatic (1) & Arabic \\
  & Altaic (1) & Turkish \\
  & Basque (1) & Basque \\
  & Dravidian (2) & Tamil, Telugu \\
  & Indo-European (21) & Ancient Greek, Bengali, Bulgarian, Catalan, Czech, Danish, Dutch, English, German, Greek, Hindi, Italian, Latin, Persian, Portuguese, Romanian, Russian, Slovak, Slovenian, Spanish, Swedish \\
  & Japanese (1) & Japanese \\
  & Uralic (3) & Estonian, Finnish, Hungarian \\

%% file: SI_methods.tex
\subsection{Statistical tests}

$\optimalityL$ is the mean value of $\Omega$ over the sentences of a language $L$ and $\optimalityn$ is the mean value of $\Omega$ over the sentences of length $n$ (in words) of that language.   
We tested if $\optimalityL$ is significantly large using a Monte Carlo procedure to estimate the $p$-value. In particular we use a permutation test that consists of 
\begin{enumerate}
\item
Precomputing the value of $D_{min}$ for every sentence.
\item
Setting $F$ to $0$. 
\item
Repeating $T$ times
  \begin{enumerate}
  \item
  Shuffling every sentence. 
  \item
  Computing the new value of $D$ for each sentence.
  \item
  Calculating $\optimalityLrla$, the new value of $\optimalityL$.
  \item
  Updating $F$, that is the number of times that $\optimalityLrla \geq \optimalityL$.   
  \end{enumerate}
\item 
Estimating the $p$-value as $F/T$. 
\end{enumerate}
The same algorithm can be adapted to test if $\optimalityL$ is significantly small by redefining $F$ as the number of times that $\optimalityLrla \leq \optimalityL$.
The same kind of tests where were used to assess if $\optimalityn$ is significantly large or significantly small.   

\subsection{Multiple comparisons}

Given a series of unadjusted $p$-values $p_1$, $p_2$,...,$p_i$,...,$p_m$ sorted increasingly, i.e. $p_1 \leq p_2 \leq ...\leq p_i \leq ... \leq p_m$, Holm's correction produces, for each $p_i$, an adjusted $p$-value $q_i$ through \cite[p. 32]{Bretz2011a}
\begin{equation}
q_i = \min\{1, \max[p_i (m + 1 - i), q_{i-1}]\}
\label{adjusted_p_value_equation}
\end{equation}  
with $q_0 = 0$.
This correction does not assume independence between $p$-values \cite{Bretz2011a} and is then suitable for dealing with $p$-values from distinct languages. Languages from the same family are not independent {\em a priori} due to vertical transmission \cite{Roberts2013b} and our sample is biased towards Indo-European languages.

Here as well as in \cite{Ferrer2019a}, Holm's correction is applied over all the languages of the dataset, before languages are split into different families so as to not reduce the power of the correction. Note that there is a confounding statement in Tables 3 and 4 of \cite[Table 3]{Ferrer2019a}: ``after applying the Holm correction to the languages counted in $f$'' should be ``after applying the Holm correction to the languages counted in $l_0$''. 
  
Estimated $p$-values of 0 obtained with the Monte Carlo procedure are underestimations of the real $p$-value. For $p_i = 0$, Eq. \ref{adjusted_p_value_equation} gives $q_i = p_i$. Then, if there are $m_0$ estimated $p$-values that are zero, $q_i = 0$ for each of them and then at least $m_0$ null hypotheses will be rejected (as $\alpha > 0$). To avoid this problem, we correct the zero $p_i$'s before applying Eq. \ref{adjusted_p_value_equation}. The correction relies on three facts about the Monte Carlo tests
\begin{itemize}
\item
The estimated $p$-values are a number in $\{0\} \cup [1/T, 1]$ where $T$ is the number of randomizations of the Monte Carlo test.  
\item
Then the smallest possible non-zero $p$-value is $1/T$.
\item
If the estimated $p$-value is zero then the true $p$-value is likely to be smaller than $1/T$ for sufficiently large $T$.
\end{itemize}
Therefore, to prevent Holm's correction from rejecting too many hypotheses due to inaccurate estimation of the unadjusted $p$-value, $p$-values estimated as zero are replaced by $(1 - \epsilon)/T$ with $\epsilon = 0.01$. A value of $\epsilon > 0$ is needed so that {\em a priori} none of the replaced $p$-values collides with any of the existing non-zero $p$-values.

\subsection{Hasse diagram}

We define a relation between two languages $x$ and $y$ such that $x < y$ if $\optimalityL_x < \optimalityL_y$ and the difference is statistically significant. The difference is statistically significant if the $p$-value of a one-tailed Fisher randomization test does not exceed the significance level $\alpha$ after controlling for multiple comparisons using Holm's correction. The PUD collection has $L = 20$ languages and then the correction is applied over an ensemble of ${L \choose 2}$ $p$-values (each corresponding to a pair of languages $\{x, y\}$ such that $\optimalityL_x < \optimalityL_y$).
Every $p$-value was estimated in three steps:
\begin{itemize}
\item
The sets of values of $\Omega$ of $x$ and $y$, namely $X$ and $Y$ were merged to produce a set $Z$. 
\item
The $p$-value is the proportion of $T$ random subsets of size $|Y|$ with a sum that is greater than that of $Y$. We used $T = 10^7$.
\item
If the $p$-value is zero it is replaced by $(1 - \epsilon)/T$ as explained above. 
\end{itemize}
The relation defines a strict partial order on the PUD dataset: a relation that is irreflexive, transitive and asymmetric (or strictly antisymmetric) \cite{Flaska2007a}. First, the relationship is by definition irreflexive ($x \leq x$ is impossible).
Second, it is transitive, because whenever $x < y$ and $y < z$ then $x < z$ with the significance level used. Finally, it is asymmetric as a consequence of irreflexivity and transitivity \cite{Flaska2007a}.  
The relationship can be represented as network where an arc $y \rightarrow x$ indicates $x < y$. A Hasse diagram is a drawing of the transitive reduction of a partially ordered set \cite[pp.413-414]{Gross2018a}. 
The transitive reduction of the network can be computed with standard algorithms \cite{Aho1972a,Hsu1975a}. When building this kind of diagrams, we control for multiple comparisons and check transitivity as indicated in previous research \cite{Burda2020a}.
Figure \ref{Hasse_diagrams_without_Holm_correction_figure} shows the misleading Hasse diagrams that would be obtained without Holm's correction.
We use the R package hasseDiagram \cite{Ciomek2017a} to perform the transitive reduction and generate the diagram.

\begin{figure}
\centering
\includegraphics[width=\linewidth]{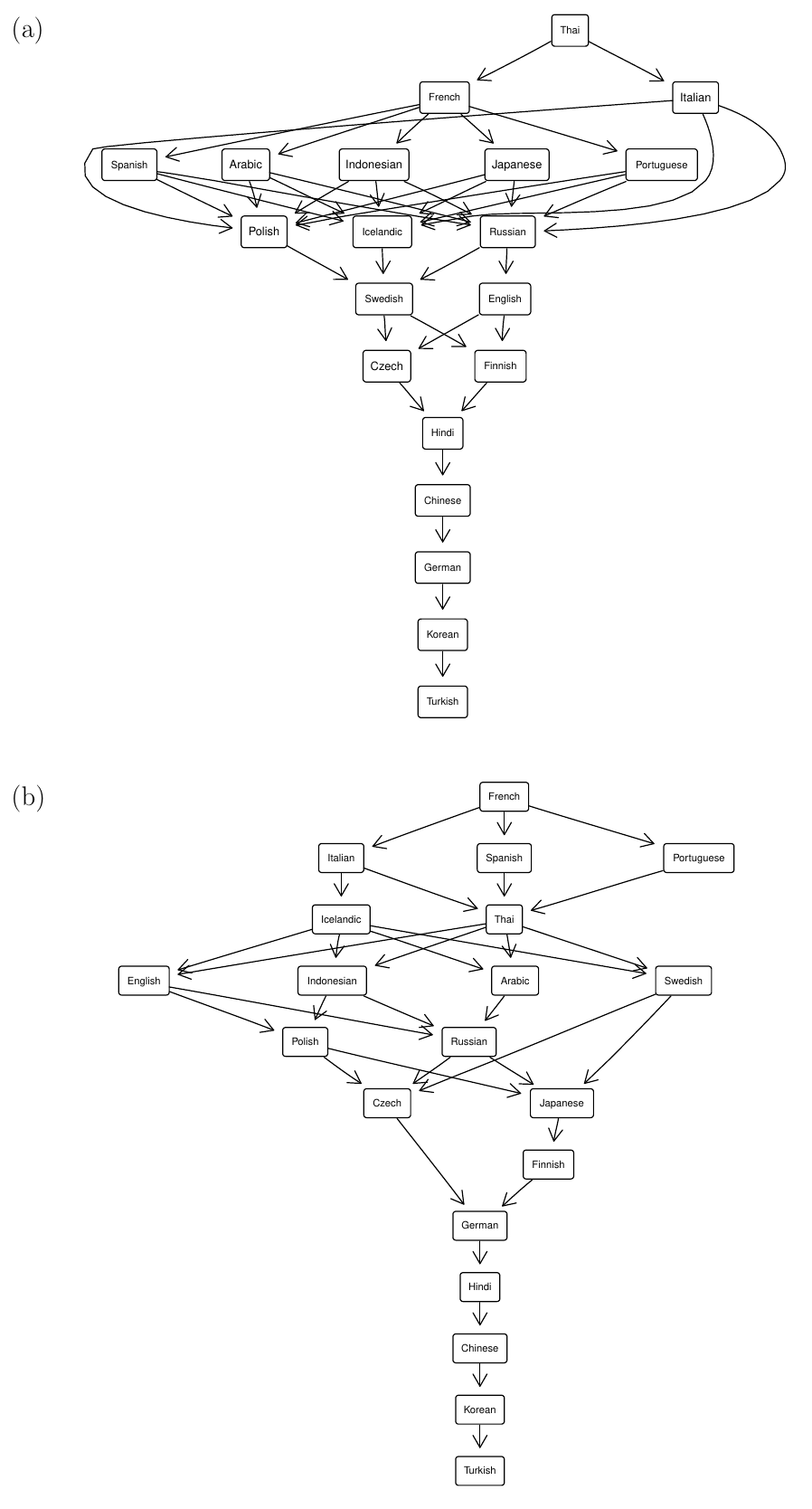}
\caption{\label{Hasse_diagrams_without_Holm_correction_figure} A Hasse diagram where an arrow pointing from one language to another indicates that the former is significantly more optimized than the latter according to $\Omega$. Here the correction for multiple comparisons has been omitted. (a) PUD. (b) PSUD. }
\end{figure}

\subsection{Minimum sentence length}

All the analyses are restricted to $n \geq 3$. When $n \in \{1, 2\}$, $\Omega$ is undefined because $D = D_{rla} = D_{min}$ and then $\Omega = 0/0$.
The problem resembles that of $C_i$, the clustering coefficient of vertex $i$, that is also 0/0 when $k_i$, the degree of $i$, satisfies $k_i < 2$ \cite{Kaiser2008a}.
The average clustering of a graph can be defined as
\begin{equation*}
C = \frac{1}{n} \sum_{i = 1}^n C_i,
\end{equation*} 
where $C_i = \gamma$ is $k_i < 2$. $\gamma = 0$ is the common choice but $\gamma = 1$ has also been used \cite{Brandes2005a}. \cite{Barmpoutis2010a} proposed a hybrid solution: $C_i = 0$ when for isolated vertices ($k_i = 0$) and $C_i = 1$ for leaves ($k_i = 1$).
Let $n_k$ be the number of vertices of degree $k$. 
Alternatively, one can define \cite{Brandes2005a,Kaiser2008a}
\begin{equation*}
C' = \frac{1}{n'} \sum_{k_i \geq 2} C_i,
\end{equation*} 
where $n' = n - n_1 + n_2$ is the number of vertices where $k_i \geq 2$. One has that \cite{Kaiser2008a}
\begin{equation*}
C' = \frac{1}{1 - \theta} C,
\end{equation*}
where $\theta$ is the fraction vertices where $C_i$ gives $0/0$, i.e. $\theta = 1 - n'/n$. 

Let $\Omega_i$ and $n_i$ be value of $\Omega$ and the sentence length of the $i$-th sentence from a sample of $N$ sentences. Let $N_n$ be the number of sentences of length $n$. Then  
\begin{equation}
\optimalityL = \frac{1}{N'} \sum_{n_i \geq 3} \Omega_i,
\label{average_omega_equation}
\end{equation} 
where $N' = N - N_1 - N_2$.
Therefore, $\optimalityL$ is the counterpart of $C'$, because $\optimalityL$ is the average $\Omega$ restricted to sentence lengths where $\Omega$ is defined. The counterpart of $C$ would be a variant of $\optimalityL$ including sentences of length $n < 3$, i.e. 
\begin{equation}
\optimalityL_{n \geq 1} = \frac{1}{N} \left[N_1 \gamma_1  + N_2 \gamma_2 + \sum_{n_i \geq 3} \Omega_i\right],
\label{average_omega_all_equation}
\end{equation}  
where $\gamma_\nu$ is the value of $\Omega_i$ when $n_i = \nu$.  
Inserting  
\begin{equation*}
N' \optimalityL = \sum_{n_i \geq 3} \Omega_i,
\end{equation*} 
into Eq. \ref{average_omega_all_equation} one obtains
\begin{equation}
\optimalityL_{n \geq 1} = (1 - \theta)\optimalityL + \frac{1}{N}(N_1 \gamma_1 + N_2 \gamma_2),
\label{advanced_average_omega_all_equation}
\end{equation}
where 
\begin{equation*}
\theta = \frac{N_1 + N_2}{N}
\end{equation*}
is the proportion of sentences where $n_i < 3$.
If $\gamma_1 = \gamma_2 = \gamma$, 
\begin{eqnarray}
\optimalityL_{n \geq 1} & = & (1 - \theta)\optimalityL + \theta \gamma.
\label{advanced_average_omega_all_convention_equation}
\end{eqnarray}
Including sentences with $n < 3$ is a potential source of arbitrariness. Choosing the value of $\gamma_1$ and $\gamma_2$ is arbitrary. One could reasonably argue that $\gamma_1 = \gamma_2 = 1$ because $D = D_{min}$ when $n < 3$. But one could also argue that $\gamma_1 = \gamma_2 = 0$ because $D = D_{rla}$ when $n < 3$. Whether sentences with $n=1$ should be included is another source of arbitrariness. One may argue that $\optimalityL$ should not be replaced by $\optimalityL_{n \geq 1}$ but rather by $\optimalityL_{n \geq 2}$.

Having in mind such potential limitations, we aim to analyze the consequences of replacing $\optimalityL$ by $\optimalityL_{n \geq 1}$ (or $\optimalityL_{n \geq 2}$). 
First, the replacement will not change the $p$-values of the test of significance of $\optimalityL$ because $\optimalityL_{n \geq 1}$ is a linear function of $\optimalityL$ (Eq. \ref{advanced_average_omega_all_equation}) and both the intercept and the slope are constant (the shuffling of the sentences preserves the distribution of sentence lengths as well as the tree of every sentence). Therefore, the replacement will not alter the conclusions about the languages where $\optimalityL$ is significantly large (or low), independently from the true values of $N_1$ and $N_2$ and also independently from the choice of the $\gamma_1$ and $\gamma_2$.
Second, the conclusions about about the presence of anti-DDm effects will not change because $\optimalityn$ is not affected for $n \geq 3$ and $\optimalityn$ is constant for $n < 3$ ($\optimalityn = \gamma_n$ and then both the left and right $p$-values will be 1).
Third, the replacement may change the value of the optimality score in languages where $N_1 + N_2 > 0$ (Eq. \ref{advanced_average_omega_all_equation}). Table \ref{short_sentences_table} indicates that $\theta > 0$ in the vast majority of languages but $\theta$ is a small number. Then Eq. \ref{advanced_average_omega_all_convention_equation} indicates that the convention $\gamma_1 = \gamma_2 = 0$ or $\gamma_1 = \gamma_2 = 1$ would imply a little difference between $\optimalityL$ and $\optimalityL_{n \geq 1}$. Notice that the situation is totally different with respect to $C$ versus $C'$ because $\theta$ is large in real networks \cite{Kaiser2008a}. 
Third, a sensitive issue is the ranking of languages and the Hasse diagram, where differences in $\optimalityL$ between languages are often small (e.g. \ref{ranking_PSUD_figure}). We believe that the decision of restricting the analyses to $n \geq 3$ with reparallelization (see above) is a reasonable choice because $\theta$ is particularly small in PUD and PSUD and the comparison of languages does not depend on the arbitrary decisions reviewed above.

\begin{table}
\centering
\caption[???]{\label{short_sentences_table} Summary of the statistics of $\theta$, the proportion of sentences of length $n < 3$ of a language, in the different datasets: the number of languages in the dataset, the number of languages where $\theta >0$ as well as the median, the average and the maximum value of $\theta$. For each language, $\theta$ is computed after the postprocessing of the treebanks described above but before the reparallelization process of PUD and PSUD. }
\begin{ruledtabular}
\begin{tabular}{llllll}
Dataset & Languages & $\theta >0$ & Median & Mean & Maximum \\
\hline
\input{figures/short_sentences}
\end{tabular}
\end{ruledtabular}
\end{table}

%% file: figures/short_sentences.tex
UD & 92 & 86 & 0.02369 & 0.03441 & 0.15813 \\
SUD & 92 & 86 & 0.02369 & 0.03441 & 0.15813 \\
Stanford & 30 & 29 & 0.02923 & 0.05069 & 0.44851 \\
Prague & 30 & 29 & 0.03023 & 0.05141 & 0.44844 \\
PUD & 20 & 14 & 0.001 & 0.00105 & 0.003 \\
PSUD & 20 & 14 & 0.001 & 0.00105 & 0.003 \\

%% file: SI_results.tex
Figs. \ref{heat_map_SUD_figure}, \ref{heat_map_Prague_figure} and \ref{heat_map_Stanford_figure} show the tendency of $\optimalityn$ to increase as sentence length increases in the SUD, Prague and Stanford datasets. 

\begin{figure*}
\centering
\includegraphics[width=\linewidth]{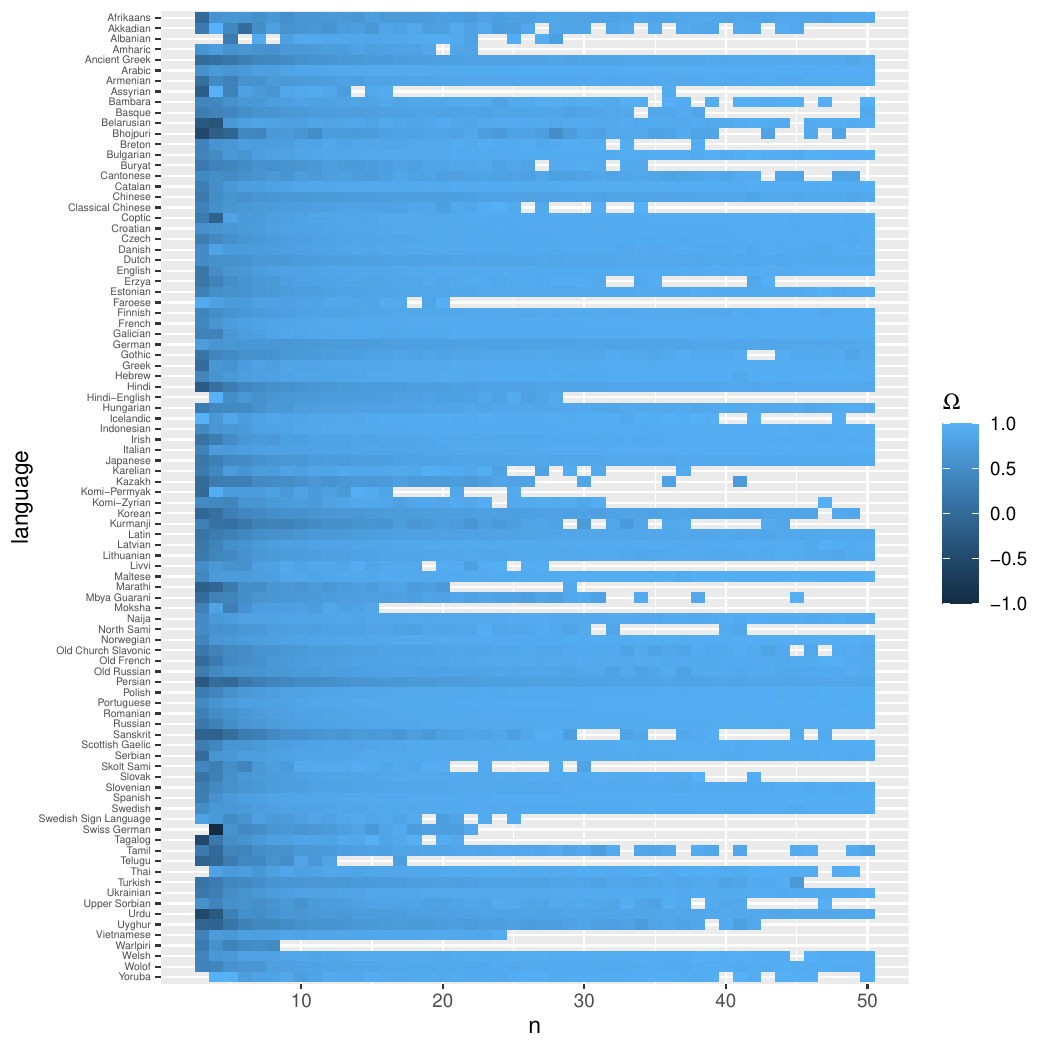}
\caption{\label{heat_map_SUD_figure} $\optimalityn$, the average value of $\Omega$ in sentences of length $n$ of a language in SUD for $3 \leq n \leq 50$. Sentences longer than 50 words are excluded to ease visualization. }
\end{figure*}

\begin{figure}
\centering
\includegraphics[width=\linewidth]{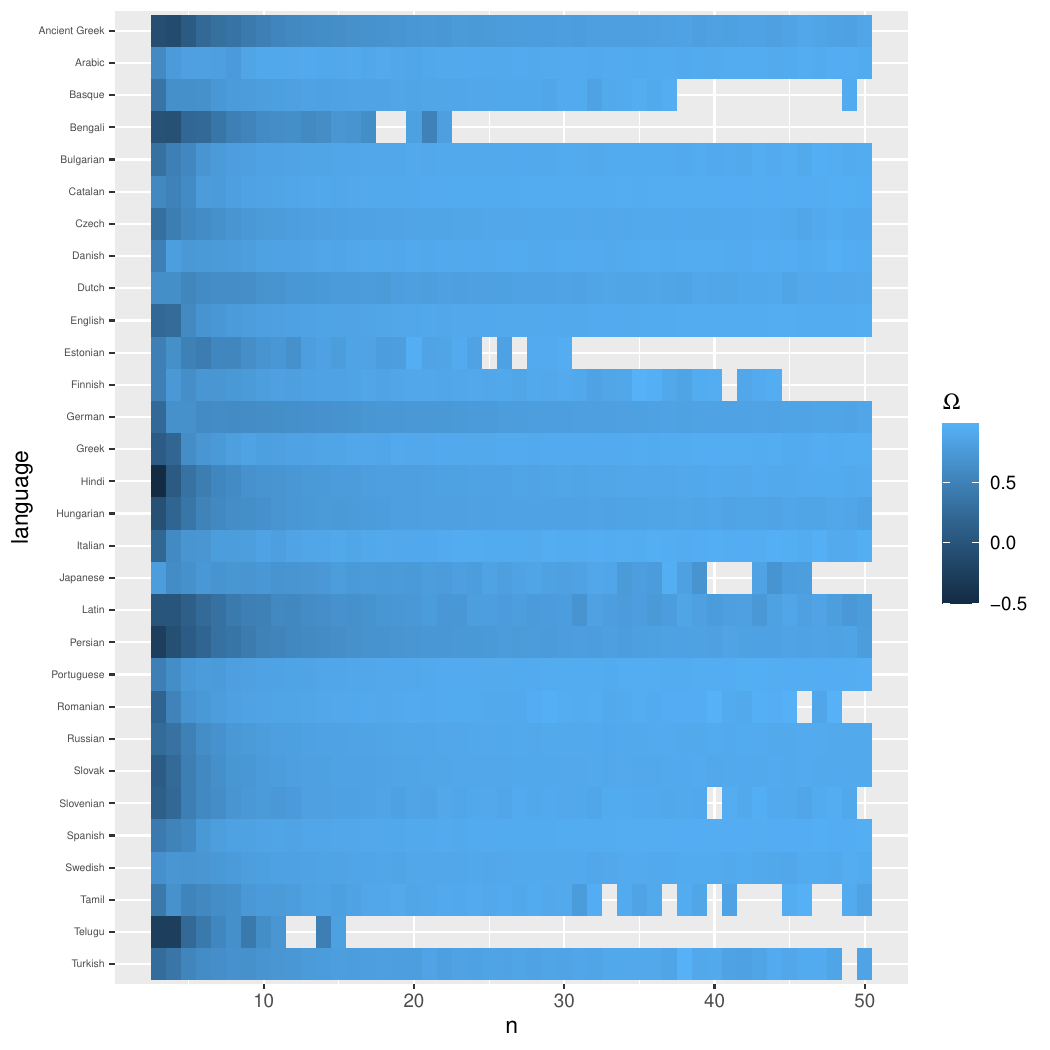}
\caption{\label{heat_map_Prague_figure} The same as in Fig. \ref{heat_map_SUD_figure} for the Prague dataset. }
\end{figure}

\begin{figure}
\centering
\includegraphics[width=\linewidth]{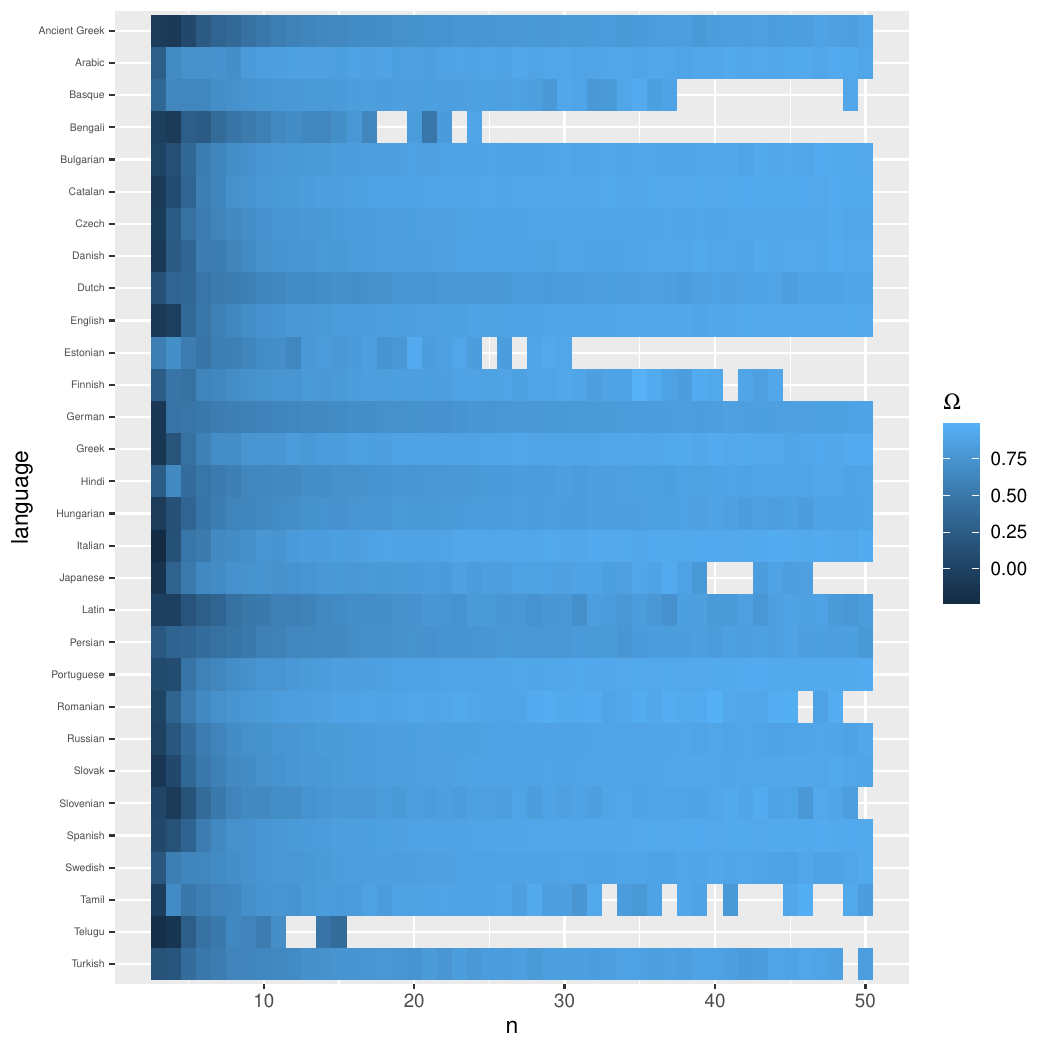}
\caption{\label{heat_map_Stanford_figure} The same as in Fig. \ref{heat_map_Stanford_figure} for the Stanford dataset. }
\end{figure}

\begin{table*}
\centering
\caption{\label{patterns_in_optimality_D_table} The significance of $\optimalityLD$. $l_0$ is the number of languages available from the dataset, 
$f_H$, is the number of languages where $\optimalityLD$ is significantly {\em small} after applying Holm's correction to the languages counted in $l_0$, and Exceptions is the number of the languages counted in $l_0$ where $\optimalityLD$ is not significantly {\em small}. The number attached to the language name indicates the corrected $p$-value.}
\begin{ruledtabular}
\begin{tabular}{lrrrrrrp{4in}}
Dataset & $l_0$ & $f_H$ & Exceptions & Family & Languages \\
\hline
\input{figures/patterns_in_optimality_table_D}
\end{tabular}
\end{ruledtabular}
\end{table*}

In Section \ref{results_section}, we report that $\optimalityL$ is significantly large in the overwhelming majority of languages. Table \ref{patterns_in_optimality_D_table} indicates that $\optimalityLD$ is significantly small in all languages except one case: Warlpiri following the UD annotation criteria. The same conclusion is reached for $\optimalityLmeanD$, $\optimalityLDelta$, $\optimalityLGamma$ and $\optimalityLDz$ and thus the corresponding tables are omitted.

In Section \ref{results_section}, we report that the correlation between $\optimalityn$ and $n$ is significantly large (and thus positive) in the overwhelming majority of languages. As in the other (preexisting) scores a lower value indicates higher optimization, one would expect that the correlation between each of the other scores and $n$ was significantly small (and thus negative) for consistency with $\Omega$. However, we find that only $D_z$ meets that expectation. For $D_z$, correlations for all languages are significantly small (even after controlling for multiple comparisons). For the other preexisting scores, the correlation between the score and $n$ is never significantly small after controlling for multiple comparisons. 
Notice that $\optimalitynD$ and $\optimalitynmeanD$ are equivalent for testing the correlation with $n$: the Kendall $\tau$ correlation will be the same for both as $\meanD = D/(n-1)$.

From a complementary perspective, the correlation between $n$ and $\optimalitynD$, $\optimalitynmeanD$, $\optimalitynDelta$ is always significantly large (and thus positive) in all languages after controlling for multiple comparisons. 
In contrast, the positive correlation between $n$ and $\optimalitynGamma$ is weaker (Table \ref{correlation_with_sentence_length_Gamma_greater_table}). 
The correlation between $n$ and $\optimalitynDz$ is never significantly large as expected.

\begin{table*}
\centering
\caption{\label{correlation_with_sentence_length_Gamma_greater_table} When the Kendall $\tau$ correlation between $\optimalitynGamma$ and sentence length $n$ is significantly {\em large}. $l_0$ is the number of languages available from the collection, 
$f_H$, is the number of languages where the correlation is significantly {\em large} after applying Holm's correction to the languages counted in $l_0$, and Exceptions is the number of the languages counted in $l_0$ where the correlation is not significantly {\em large}. The number attached to the language name indicates the corrected $p$-value. }
\begin{ruledtabular}
\begin{tabular}{lrrrrp{4in}}
Collection & $l_0$ & $f_H$ & Exceptions & Family & Languages \\
\hline
\input{figures/correlation_with_sentence_length_table_Gamma_greater}
\end{tabular}
\end{ruledtabular}
\end{table*}


In Section \ref{results_section}, we report that $\optimalityn$ is significantly small in certain languages when $n=3$ or $n=4$. Tables \ref{anti_dependency_distance_D_table} and \ref{anti_dependency_distance_D_z_table} indicate that $\optimalitynD$, $\optimalitynmeanD$, $\optimalitynDelta$, $\optimalitynGamma$ and $\optimalitynDz$ are consistently significantly large also in certain languages. {\em A priori}, $\optimalitynD$ and $\optimalitynmeanD$ are equivalent for testing significance (recall $\meanD = D/(n-1)$) and thus only $\optimalitynD$ is shown in the tables.
In addition, the results for $\Delta$ and $\Gamma$ are the same as those of $D$ and thus not shown either. 

Interestingly, the number of languages with evidence of anti-dependency distance effects based on $\Delta$, $\Gamma$ (Table \ref{anti_dependency_distance_D_table}) is smaller than that of $D_z$ (Table \ref{anti_dependency_distance_D_z_table}) which is in turn smaller than that of  $\Omega$ (Table \ref{anti_dependency_distance_table}).
The results indicate that all the other scores have a lower capacity to detect anti-dependency distance minimization effects than $\Omega$.
  
\begin{table*}
\centering
\caption{\label{anti_dependency_distance_D_table} Anti dependency distance minimization in short sentences (significantly large $\optimalitynD$ for small $n$). 
$n$ is the sentence length (in words). $l_0$ is the number of languages available from the dataset, 
and $f_H$, is the number of languages where $\optimalitynD$ is significantly small after applying Holm's correction to the languages counted in $l_0$. 
The number attached to the language name indicates the magnitude of the corrected $p$-value. It is obtained after rounding $-\log_{10}(p\mathrm{-value})$ to leave just one decimal digit. Then the significance level $\alpha = 0.05$ gives 1.3.
}
\begin{ruledtabular}
\begin{tabular}{llrrrp{4in}}
$n$ & Dataset & $l_0$ & $f_H$ & Family & Languages \\
\hline
\input{figures/anti_dependency_distance_table_D} 
\end{tabular}
\end{ruledtabular}
\end{table*}

\begin{table*}
\centering
\caption{\label{anti_dependency_distance_D_z_table} Anti dependency distance in short sentences (significantly large $\optimalitynDz$ for small $n$). The format is the same as in Table \ref{anti_dependency_distance_D_table}.
}
\begin{ruledtabular}
\begin{tabular}{llrrrp{4in}}
$n$ & Dataset & $l_0$ & $f_H$ & Family & Languages \\
\hline
\input{figures/anti_dependency_distance_table_D_z}
\end{tabular}
\end{ruledtabular}
\end{table*}

Fig. \ref{ranking_PSUD_figure} shows the ranking of languages that is obtained when universal dependencies are replaced by surface-syntactic universal dependencies \cite{sud} in PUD.
\begin{figure}
\centering
\includegraphics[width=\linewidth]{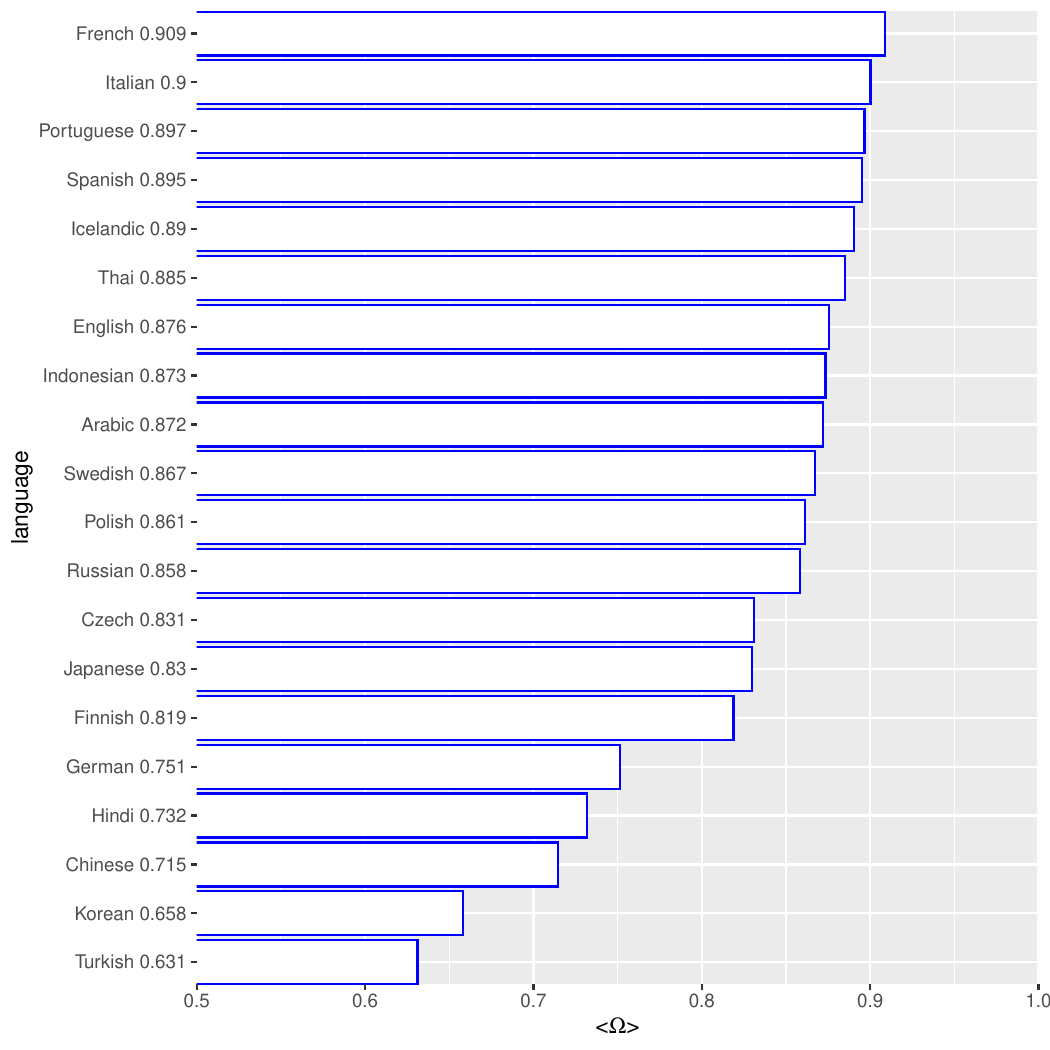}
\caption{\label{ranking_PSUD_figure} The ranking of languages in PSUD according to $\Omega$. The number attached to each language indicates the value of $\optimalityL$. }
\end{figure}

Fig. \ref{rankings_content_words_figure} shows the rankings of languages that result from removing all function words. 
\begin{figure}
\centering
\includegraphics[width=\linewidth]{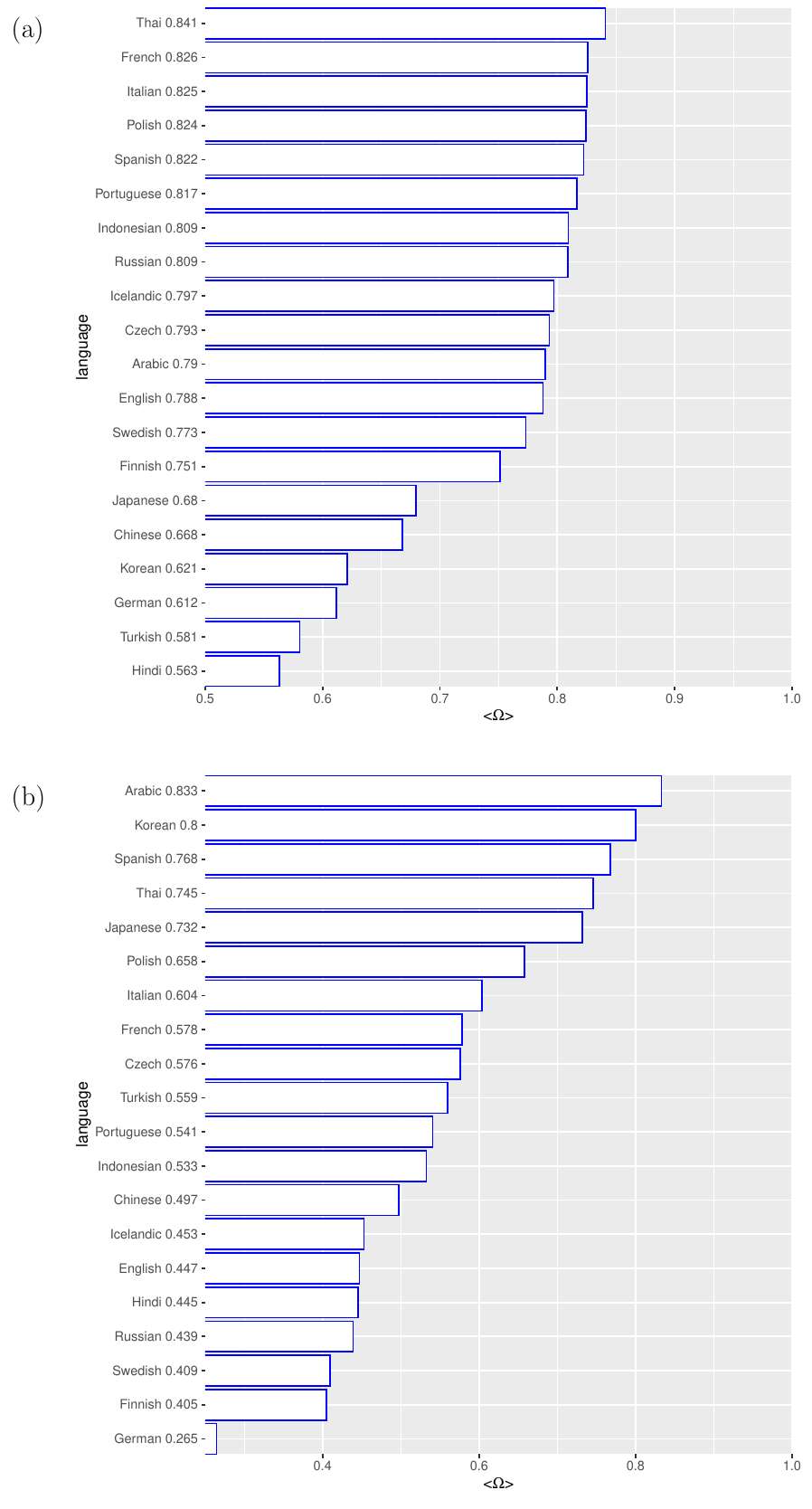}
\caption{\label{rankings_content_words_figure} The ranking of languages according to $\Omega$ when function words are removed. The number attached to each language indicates the value of $\optimalityL$. (a) PUD (b) PSUD. }
\end{figure}
With respect to Figs. \ref{ranking_PUD_figure} and \ref{ranking_PSUD_figure}, notice that $\optimalityL$ drops in all languages except one (Czech for PUD and Korean for PSUD). The removal of function words has a stronger impact on PSUD, where $8$ languages end up with $\optimalityL < 1/2$. 

It is important to bear in mind the impact of removing function words. Because of the preprocessing and the reparallelization method, both PUD and PSUD end up having 995 sentences each, while the removal of function words leaves 835 sentences in PUD and only 13 sentences in PSUD. As explained in Appendix~\ref{materials_appendix}, the latter is mainly because SUD trees often have function words acting as syntactic roots, so that removing them breaks the tree into disconnected fragments that can no longer be analyzed.
This means that the removal of function words in PSUD implies serious statistical problems. On the one hand, a higher variance in the estimation of the average $\Omega$ due to a reduction of the sample size of about two orders of magnitude compared to the other conditions. On the other hand, a potential increase in the bias arising from the complex interaction among annotation criteria, the preprocessing of syntactic structures and the reparallelization method. For example, any sentence headed by an auxiliary is excluded for the reason explained above, so some types of syntactic structures (e.g., passive voice) cannot be represented at all in the reduced sample. For all these reasons, we believe that the analysis without function words on PSUD cannot be considered reliable, even though we include it for completeness.

The statistical problems arising when removing function words in PSUD can be understood further by considering the Hasse diagrams that result from removing function words. 
That of PUD (Fig. \ref{Hasse_diagram_PUD_content_words_figure}) contrasts with the one that is obtained for PSUD, which has no edges, and for that reason it is not shown. The absence of edges is due to the fact that none of the significant differences in $\Omega$ between pairs of languages survives after applying Holm's correction, consistent with the dramatic loss of sentences in PSUD when function words are removed.

\begin{figure}
\centering
\includegraphics[width=\linewidth]{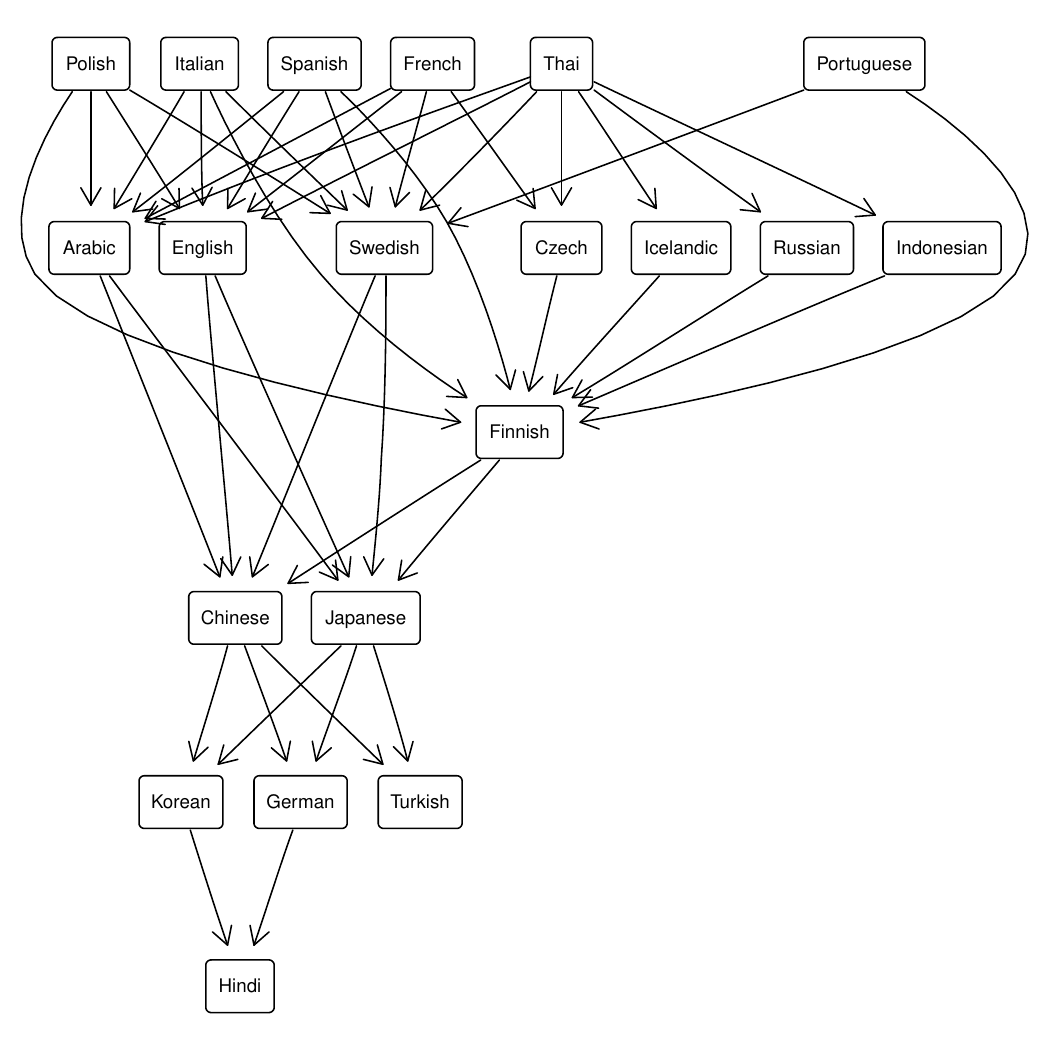}
\caption{\label{Hasse_diagram_PUD_content_words_figure} The Hasse diagram (based on $\Omega$) that is obtained when function words are removed from PUD. }
\end{figure}

\clearpage

%% file: figures/patterns_in_optimality_table_D.tex
UD & 92 & 91  & 1  & Pama-Nyungan (1) & Warlpiri$^{0.2}$ \\
SUD & 92 & 92  & --- & --- & --- \\
Prague & 30 & 30  & --- & --- & --- \\
Stanford & 30 & 30  & --- & --- & --- \\

%% file: figures/correlation_with_sentence_length_table_Gamma_greater.tex
UD & 92 & 62  & 30  & Afro-Asiatic (3) & Assyrian$^{0.4}$ Akkadian$^{1}$ Amharic$^{0.8}$ \\
& & & & Austronesian (1) & Tagalog$^{1}$ \\
& & & & Dravidian (1) & Telugu$^{1}$ \\
& & & & Indo-European (12) & Afrikaans$^{0.2}$ Albanian$^{0.8}$ Breton$^{1}$ Welsh$^{1}$ Faroese$^{1}$ Gothic$^{1}$ Ancient Greek$^{1}$ Swiss German$^{1}$ Upper Sorbian$^{0.2}$ Icelandic$^{1}$ Sanskrit$^{0.4}$ Slovak$^{0.8}$ \\
& & & & Mande (1) & Bambara$^{1}$ \\
& & & & Mongolic (1) & Buryat$^{1}$ \\
& & & & Niger-Congo (1) & Yoruba$^{1}$ \\
& & & & Pama-Nyungan (1) & Warlpiri$^{1}$ \\
& & & & Sign Language (1) & Swedish Sign Language$^{0.1}$ \\
& & & & Tai-Kadai (1) & Thai$^{0.3}$ \\
& & & & Uralic (7) & Komi-Permyak$^{0.5}$ Komi-Zyrian$^{1}$ Karelian$^{0.1}$ Moksha$^{0.8}$ Erzya$^{0.1}$ Livvi$^{1}$ Skolt Sami$^{0.1}$ \\
SUD & 92 & 54  & 38  & Afro-Asiatic (2) & Assyrian$^{1}$ Akkadian$^{1}$ \\
& & & & Austronesian (1) & Tagalog$^{1}$ \\
& & & & Dravidian (2) & Tamil$^{0.1}$ Telugu$^{1}$ \\
& & & & Indo-European (18) & Albanian$^{1}$ Belarusian$^{0.2}$ Breton$^{1}$ Czech$^{1}$ Welsh$^{0.2}$ Faroese$^{0.5}$ Galician$^{0.2}$ Gothic$^{1}$ Ancient Greek$^{1}$ Swiss German$^{0.3}$ Upper Sorbian$^{0.7}$ Icelandic$^{0.6}$ Italian$^{1}$ Kurmanji$^{1}$ Persian$^{1}$ Polish$^{0.1}$ Sanskrit$^{1}$ Slovak$^{0.2}$ \\
& & & & Mande (1) & Bambara$^{1}$ \\
& & & & Mongolic (1) & Buryat$^{1}$ \\
& & & & Niger-Congo (2) & Wolof$^{0.1}$ Yoruba$^{1}$ \\
& & & & Pama-Nyungan (1) & Warlpiri$^{0.3}$ \\
& & & & Sign Language (1) & Swedish Sign Language$^{1}$ \\
& & & & Sino-Tibetan (1) & Classical Chinese$^{0.3}$ \\
& & & & Tai-Kadai (1) & Thai$^{1}$ \\
& & & & Uralic (7) & Komi-Permyak$^{1}$ Komi-Zyrian$^{1}$ Karelian$^{0.5}$ Moksha$^{0.7}$ Erzya$^{1}$ Livvi$^{1}$ Skolt Sami$^{1}$ \\
Prague & 30 & 28  & 2  & Dravidian (1) & Telugu$^{0.2}$ \\
& & & & Uralic (1) & Estonian$^{0.2}$ \\
Stanford & 30 & 30  & --- & --- & --- \\

%% file: figures/anti_dependency_distance_table_D.tex
3  & UD & 87 & 9  & Austronesian (1) & Tagalog$^{1.5}$ \\
 & & & & Dravidian (1) & Telugu$^{1.6}$ \\
 & & & & Indo-European (6) & German$^{3.1}$ English$^{3.1}$ Old French$^{3.1}$ Norwegian$^{2.5}$ Sanskrit$^{2.4}$ Slovak$^{3.1}$ \\
 & & & & Japanese (1) & Japanese$^{3.1}$ \\
   & SUD & 87 & 2  & Austronesian (1) & Tagalog$^{1.5}$ \\
 & & & & Indo-European (1) & Persian$^{2}$ \\
   & Prague & 30 & 2  & Dravidian (1) & Telugu$^{3.5}$ \\
 & & & & Indo-European (1) & Persian$^{3.5}$ \\
   & Stanford & 30 & 5  & Dravidian (1) & Telugu$^{3.5}$ \\
 & & & & Indo-European (3) & Czech$^{3.5}$ German$^{3.5}$ Slovak$^{3.5}$ \\
 & & & & Japanese (1) & Japanese$^{3.5}$ \\
4  & UD & 91 & 1  & Pama-Nyungan (1) & Warlpiri$^{1.9}$ \\
   & SUD & 91 & 0  & --- & --- \\
   & Prague & 30 & 0  & --- & --- \\
   & Stanford & 30 & 0  & --- & --- \\

%% file: figures/anti_dependency_distance_table_D_z.tex
3  & UD & 87 & 10  & Austronesian (1) & Tagalog$^{1.5}$ \\
 & & & & Dravidian (1) & Telugu$^{1.7}$ \\
 & & & & Indo-European (7) & Breton$^{1.4}$ German$^{3.1}$ English$^{3.1}$ Old French$^{3.1}$ Norwegian$^{2.6}$ Sanskrit$^{2.5}$ Slovak$^{3.1}$ \\
 & & & & Japanese (1) & Japanese$^{3.1}$ \\
   & SUD & 87 & 2  & Austronesian (1) & Tagalog$^{1.5}$ \\
 & & & & Indo-European (1) & Persian$^{2.4}$ \\
   & Prague & 30 & 2  & Dravidian (1) & Telugu$^{3.5}$ \\
 & & & & Indo-European (1) & Persian$^{3.5}$ \\
   & Stanford & 30 & 5  & Dravidian (1) & Telugu$^{3.5}$ \\
 & & & & Indo-European (3) & Czech$^{3.5}$ German$^{3.5}$ Slovak$^{3.5}$ \\
 & & & & Japanese (1) & Japanese$^{3.5}$ \\
4  & UD & 91 & 1  & Pama-Nyungan (1) & Warlpiri$^{1.9}$ \\
   & SUD & 91 & 0  & --- & --- \\
   & Prague & 30 & 0  & --- & --- \\
   & Stanford & 30 & 0  & --- & --- \\